\documentclass{article}

\usepackage{microtype}
\usepackage{graphicx}
\usepackage{subfigure}
\usepackage{subcaption}
\usepackage{booktabs} 

\usepackage{hyperref}

\usepackage[accepted]{icml2024}

\usepackage{amsmath}
\usepackage{amssymb}
\usepackage{mathtools}
\usepackage{amsthm}
\usepackage{float}
\usepackage{multirow}
\usepackage[capitalize,noabbrev]{cleveref}

\theoremstyle{plain}
\newtheorem{theorem}{Theorem}[section]
\newtheorem{proposition}[theorem]{Proposition}
\newtheorem{lemma}[theorem]{Lemma}

\theoremstyle{definition}
\newtheorem{definition}[theorem]{Definition}

\theoremstyle{remark}
\newtheorem{remark}[theorem]{Remark}

\usepackage[textsize=tiny]{todonotes}

\newenvironment{myitemize}{\begin{list}{$\bullet$}
{\setlength{\topsep}{1mm}
\setlength{\itemsep}{0.25mm}
\setlength{\parsep}{0.25mm}
\setlength{\itemindent}{0mm}
\setlength{\partopsep}{0mm}
\setlength{\labelwidth}{15mm}
\setlength{\leftmargin}{4mm}}}{\end{list}}

\definecolor{lightred}{RGB}{227,74,51}
\definecolor{lightblue}{RGB}{67,162,202}

\icmltitlerunning{Boosting Reinforcement Learning with Strongly Delayed Feedback Through Auxiliary Short Delays}

\begin{document}

\twocolumn[
\icmltitle{Boosting Reinforcement Learning with Strongly Delayed Feedback \\ Through Auxiliary Short Delays}



\icmlsetsymbol{equal}{*}

\begin{icmlauthorlist}
\icmlauthor{Qingyuan Wu}{liverpool}
\icmlauthor{Simon Sinong Zhan}{northwest}
\icmlauthor{Yixuan Wang}{northwest}
\icmlauthor{Yuhui Wang}{kaust}\\
\icmlauthor{Chung-Wei Lin}{ntu_tw}
\icmlauthor{Chen Lv}{ntu_sig}
\icmlauthor{Qi Zhu}{northwest}
\icmlauthor{Jürgen Schmidhuber}{kaust,idisa}
\icmlauthor{Chao Huang}{liverpool,southampton}
\end{icmlauthorlist}

\icmlaffiliation{liverpool}{The University of Liverpool}
\icmlaffiliation{northwest}{Northwestern University}
\icmlaffiliation{kaust}{AI Initiative, King Abdullah University of Science and Technology}
\icmlaffiliation{idisa}{The Swiss AI Lab IDSIA/USI/SUPSI}
\icmlaffiliation{ntu_tw}{National Taiwan University}
\icmlaffiliation{ntu_sig}{Nanyang Technological University}
\icmlaffiliation{southampton}{The University of Southampton}

\icmlcorrespondingauthor{Chao Huang}{chao.huang@soton.ac.uk}

\icmlkeywords{Machine Learning, ICML}

\vskip 0.3in
]



\printAffiliationsAndNotice{}  

\begin{abstract}
Reinforcement learning (RL) is challenging in the common case of delays between events and their sensory perceptions. State-of-the-art (SOTA) state augmentation techniques either suffer from state space explosion or performance degeneration in stochastic environments. To address these challenges, we present a novel \textit{Auxiliary-Delayed Reinforcement Learning (AD-RL)} method that leverages auxiliary tasks involving short delays to accelerate RL with long delays, without compromising performance in stochastic environments. Specifically, AD-RL learns a value function for short delays and uses bootstrapping and policy improvement techniques to adjust it for long delays. We theoretically show that this can greatly reduce the sample complexity. On deterministic and stochastic benchmarks, our method significantly outperforms the SOTAs in both sample efficiency and policy performance.
Code is available at \href{https://github.com/QingyuanWuNothing/AD-RL}{{https://github.com/QingyuanWuNothing/AD-RL}}.
\end{abstract}

\section{Introduction}
\label{sec:introduction}

Reinforcement learning (RL) has already proved its mettle in complex tasks such as Backgammon~\citep{tesauro1994td}, Go~\citep{silver2018general}, MOBA Game~\citep{berner2019dota}, building control~\cite{xu2021learning,xu2022accelerate}, and various cyber-physical systems~\citep{wang2023joint,wang2023enforcing,zhan2024state}. 
Most of the above RL settings assume that the agent's interaction with the environment is instantaneous, which means that the agent can always execute commands without delay and gather feedback from the environment right away.
However, the persistent presence of delays in real-world applications significantly hampers agents' efficiency, performance, and safety if not handled properly (e.g., introducing estimation error~\citep{quadrotor_delay} and losing reproducibility~\citep{arm_delay} in practical robotic tasks).
Delay also needs to be considered in many stochastic settings such as financial markets~\citep{hasbrouck2013low} and weather forecasting~\citep{fathi2022big}.
Thus, addressing delays in RL algorithms is crucial for their deployment in real-world timing-sensitive tasks.

Delays in RL can be primarily divided into three categories: observation delay, action delay, and reward delay~\citep{firoiu2018human}, depending on where the delay occurs. Among them, observation delay receives considerable attention due to the application-wise generality and the technique-wise challenge: it has been proved to be a superset of action delay~\citep{delay_mdp, revisiting_augment}, and unlike well-studied reward delay~\citep{han2022off, kim2020automatic}, it disrupts the Markovian property of systems (i.e., the underlying dynamics depend on an unobserved state and the sequence of actions). 
In this work, we focus on \emph{non-anonymous and constant observation delay} under finite Markov Decision Process (MDP) settings, where the delay is known to the agent and always a constant number of time steps (details in Section~\ref{research_problem}), as in most existing works~\citep{memoryless, chen2021delay}.

\begin{figure*}[t]
    \vskip -0.15in
    \centering
    \scalebox{1.0}{
    \centerline{
        \subfigure[Deterministic MDP.]{\includegraphics[width=0.25\linewidth]{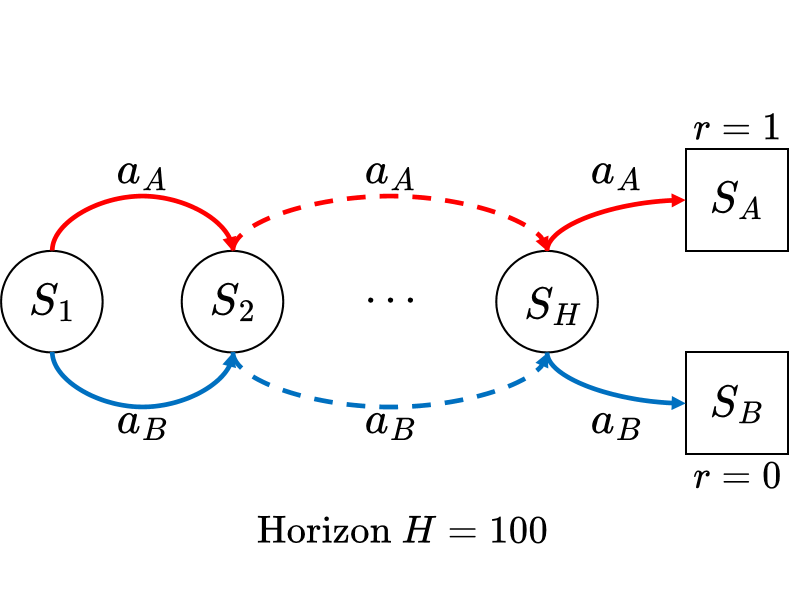}\label{fig:de_motivating_example}}
        \subfigure[Shorter $\Delta^\tau$ is better.]{\includegraphics[width=0.25\linewidth]{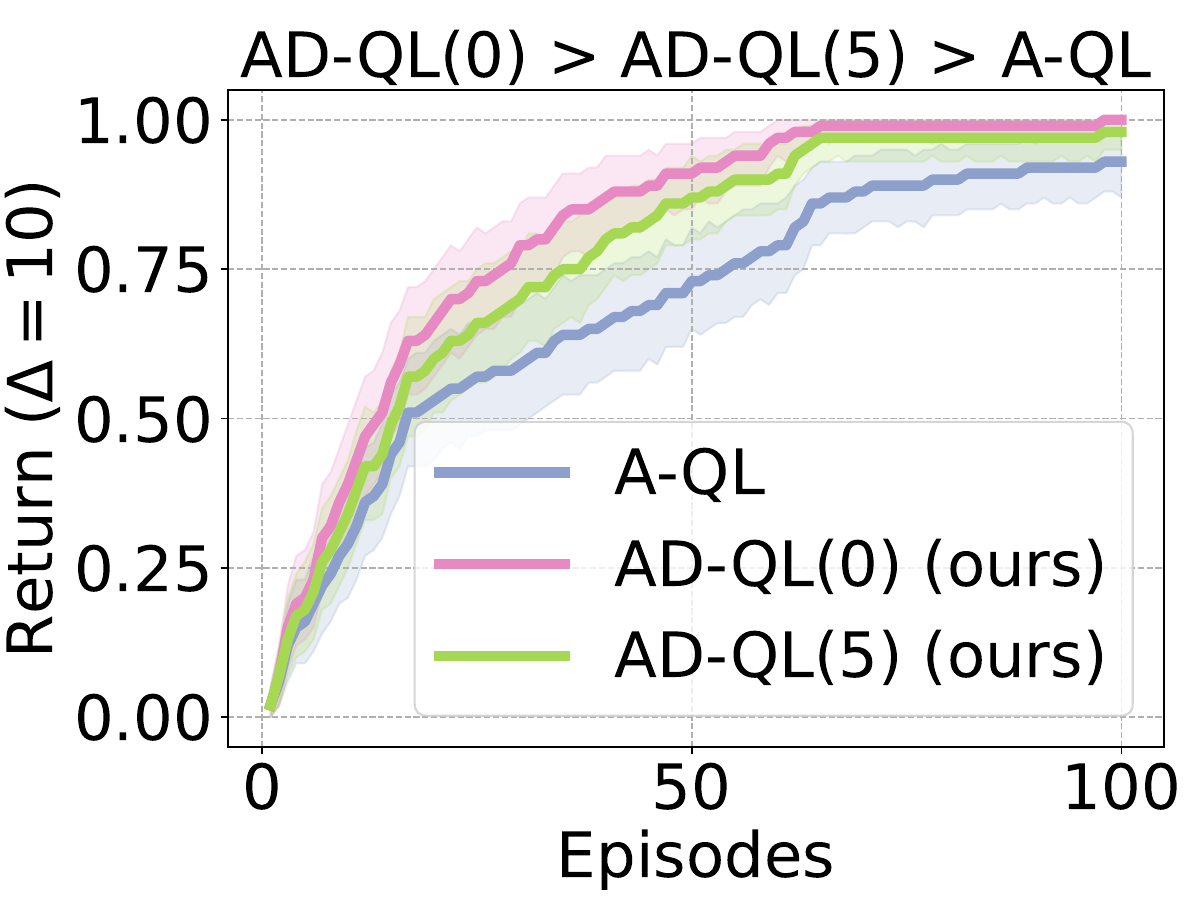}\label{fig:de_aux_delay_impact}}
        \subfigure[Stochastic MDP ($p=0.1$).]{\includegraphics[width=0.25\linewidth]{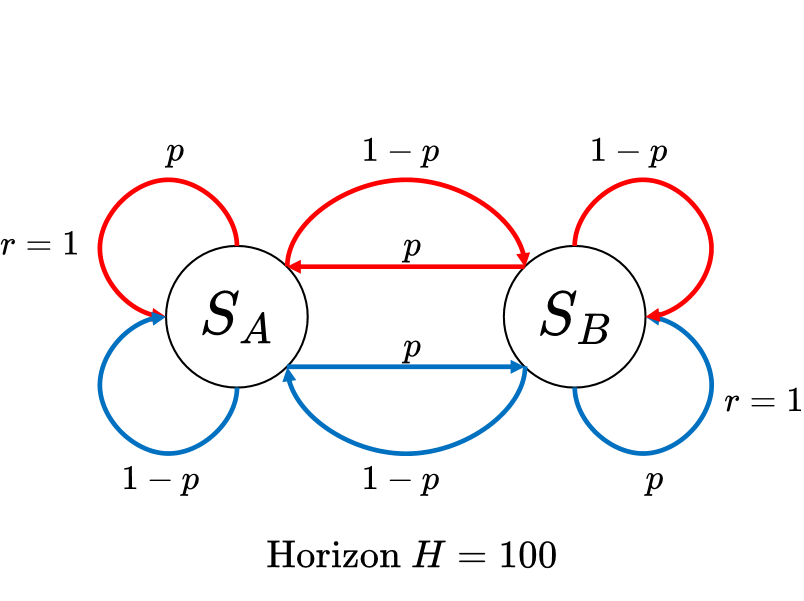}\label{fig:sto_motivating_example}}
        \subfigure[Shorter $\Delta^\tau$ is not better.]{\includegraphics[width=0.25\linewidth]{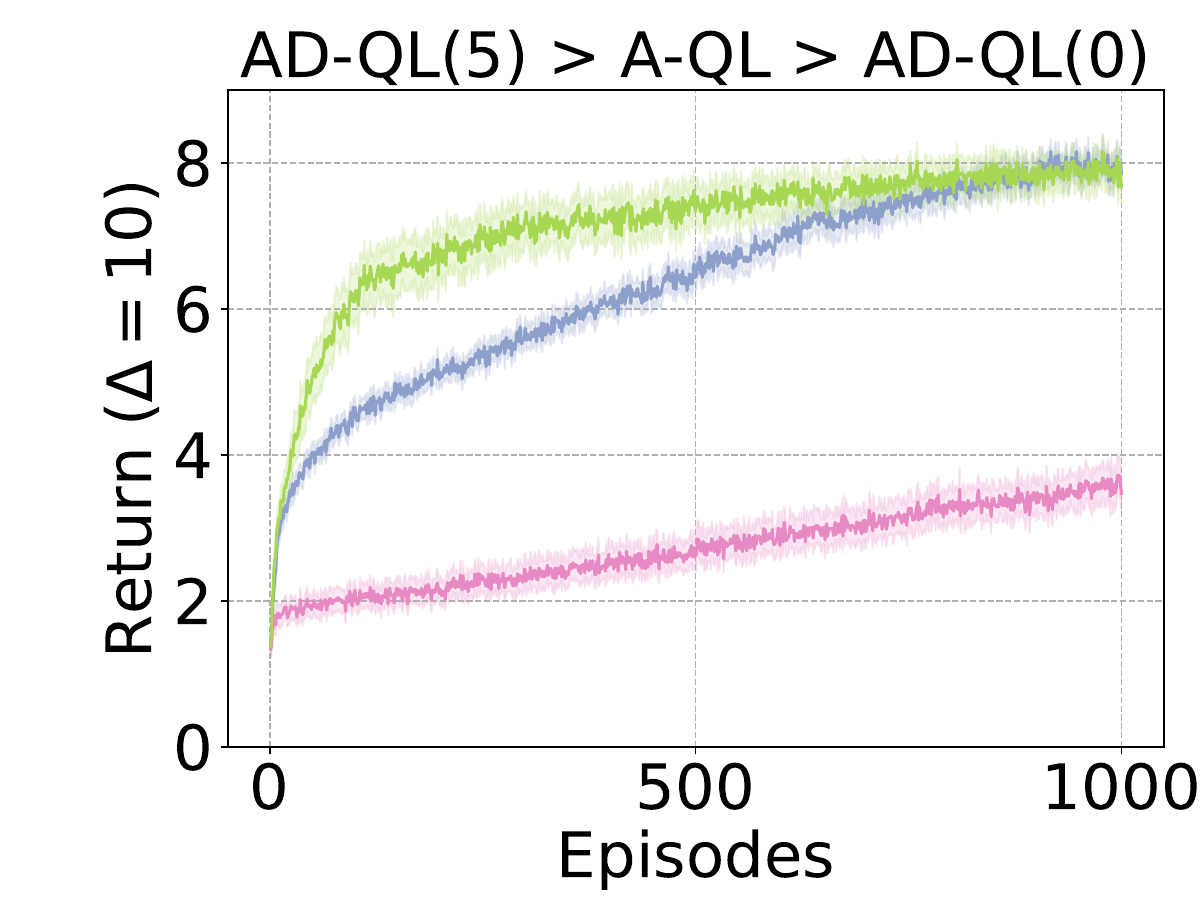}\label{fig:sto_aux_delay_impact}}
    }
    }
    \vskip -0.1in
    \label{fig:motivating_example}
    \caption{
    Our AD-RL method introduces an adjoint task with short delays, enhancing the original augmentation-based method (A-QL) in deterministic MDP (Fig.~\ref{fig:de_motivating_example}) with delay $\Delta=10$, shown in Fig.~\ref{fig:de_aux_delay_impact}. 
    Whereas, in stochastic MDP (Fig.~\ref{fig:sto_motivating_example}) with delay $\Delta=10$, a short auxiliary delays may lead to performance improvement (AD-QL(5)) or drop (AD-QL(0)) as shown in Fig.~\ref{fig:sto_aux_delay_impact}.  
    BPQL always uses a fixed $0$ auxiliary delays, equivalent to AD-QL(0) in these examples. Notably, the optimal auxiliary delays is irregular and task-specific, which we discussed in subsequent experiments in Section~\ref{boad_experiment}.
    }
    \label{fig:motivating}
    \vskip -0.15in
\end{figure*}

Promising augmentation-based approaches~\cite{altman_delay, delay_mdp} transform the delayed RL problem into an MDP by augmenting the latest observed state with a sequence of actions related to the delay, also known as the information state~\cite{bertsekas2012dynamic}.
After retrieving the Markovian property, the augmentation-based methods adopt classical RL methods to solve the delayed RL problem properly, such as augmented Q-learning (A-QL)~\cite{revisiting_augment}.
However, existing augmentation-based methods are plagued by the curse of dimensionality, shown by our toy examples in Fig.~\ref{fig:motivating}. Under a deterministic MDP setting (Fig. \ref{fig:de_motivating_example}), the original augmented state space grows exponentially with the delays, causing learning inefficiency. 
The variant of augmentation-based methods BPQL~\citep{kim2023belief} approximates the value function based on the delay-free MDP to tackle the inefficiency, which unexpectedly results in excessive information loss. 
Consequently, it cannot properly handle stochastic tasks (Fig. \ref{fig:sto_motivating_example}).

To address the aforementioned challenges, we propose a novel technique named \emph{Auxiliary-Delayed RL (AD-RL)}. 
Our AD-RL is inspired by a key observation that an elaborate auxiliary task with short delays carries much more accurate information than the delay-free case about the original task with long delays, and is still easy to learn.
By introducing the notion of delayed belief to bridge an auxiliary task with short delays and the original task with long delays, we can learn the auxiliary-delayed value function and map it to the original one.
The changeable auxiliary delays in our AD-RL has the ability to flexibly address the trade-off between the learning efficiency and approximation accuracy error in various MDPs.
In toy examples (Fig.~\ref{fig:de_motivating_example} and Fig.~\ref{fig:sto_motivating_example}) with 10 delays, we compare the performance of A-QL and AD-RL with 0 and 5 auxiliary delays respectively (AD-QL(0) and AD-QL(5)).
Our AD-RL not only remarkably enhances the learning efficiency (Fig. \ref{fig:de_aux_delay_impact}) but also possesses the flexibility to capture more information under the stochastic setting (Fig. \ref{fig:sto_aux_delay_impact}).
Notably, BPQL is a special variant of our AD-RL with fixed 0 auxiliary delays, resulting in poor performance under the stochastic setting.
In Section~\ref{boad_approch}, we develop AD-DQN and AD-SAC, extending from Deep Q-Network and Soft Actor-Critic with our AD-RL framework respectively.
Besides, we provide an in-depth theoretical analysis of learning efficiency, performance gap, and convergence in Section~\ref{boad_theory}. 
In Section \ref{boad_experiment}, we show superior efficacy of our method over the SOTA approaches on the different benchmarks. 
Our contributions can be summarized as:
\begin{myitemize}
    \vspace{-6pt}\item We address the sample inefficiency of the original augmentation-based approaches (denoted as A-RL) and excessive approximation error of the belief-based approaches by introducing AD-RL, which is more efficient with a short auxiliary-delayed task and achieves a theoretical similar performance with A-RL.
    \item Adapting the AD-RL framework, we devise AD-DQN and AD-SAC to handle discrete and continuous control tasks, respectively. 
    \item We analyze the superior sampling efficiency of AD-RL, the performance gap bound between AD-RL and A-RL, and provide the convergence guarantee of AD-RL.
    \item We show notable improvements of AD-RL over existing SOTA methods in policy performance and sampling efficiency for deterministic and stochastic benchmarks.

\end{myitemize}

\section{Preliminaries}
\label{preliminaries}
\subsection{Delay-free RL}
The delay-free RL problem is usually modelled as a Markov Decision Process (MDP), defined as a tuple $\langle \mathcal{S}, \mathcal{A}, \mathcal{P}, \mathcal{R}, \gamma, \rho \rangle$. 
An MDP consists of a state space $\mathcal{S}$, an action space $\mathcal{A}$, a probabilistic transition function $\mathcal{P}: \mathcal{S} \times \mathcal{A} \times \mathcal{S} \rightarrow [0, 1]$, a reward function $\mathcal{R}: \mathcal{S} \times \mathcal{A} \rightarrow \mathbb{R}$, a discount factor $\gamma \in (0, 1)$ and an initial state distribution $\rho$. 
At each time step $t$, based on the input state $s_t \in \mathcal{S}$ and the policy $\pi: \mathcal{S} \times \mathcal{A} \rightarrow [0, 1]$, the agent has an action $a_t \sim \pi(\cdot|s_t)$ where $a_t \in \mathcal{A}$, and then the MDP evolves to a new state $s_{t+1} \in \mathcal{S}$ based on the probabilistic transition function $\mathcal{P}$ and the agent receive a reward signal $r_t$ from reward function $\mathcal{R}(s_t, a_t)$. We use $d_{s_0}^\pi$ to denote the visited state distribution starting from $s_0$ based on policy $\pi$.
The objective of the agent in an MDP is to find a policy that maximizes return over the horizon $H$.
Given a state $s$, the value function of policy $\pi$ is defined as
$$
    V^\pi(s) = \mathop{\mathbb{E}}_{
    s_{t+1}\sim\mathcal{P}(\cdot| s_t, a_t)\atop
    a_t\sim \pi(\cdot|s_t)}
    \left[\sum_{t=0}^H\gamma^t\mathcal{R}(s_t, a_t) \middle |s_0=s \right].
$$
Similarly, given a state-action pair $(s, a)$, the Q-function of policy $\pi$ can be defined as 
$$
    Q^\pi(s, a) = \mathop{\mathbb{E}}_{
    s_{t+1}\sim\mathcal{P}(\cdot| s_t, a_t)\atop 
    a_t\sim \pi(\cdot|s_t)}
    \left[\sum_{t=0}^H\gamma^t\mathcal{R}(s_t, a_t) \middle |{s_0=s\atop a_0=a} \right].
$$

\subsection{Deep Q-Network and Soft Actor-Critic}
A widely used off-policy RL method is the Deep Q-Network (DQN)~\citep{deep_q_network} with the Q-function $Q_{\theta}: \mathcal{S} \times \mathcal{A} \rightarrow \mathbb{R}$ parameterized by $\theta$. It conducts temporal-difference (TD) learning based on the Bellman optimality equation. Given the transition data $(s_t, a_t, r_t, s_{t+1})$, DQN updates the Q-function via minimizing TD error.
    \begin{displaymath}
    \nabla_\theta \left[\frac{1}{2}(Q_{\theta}(s_t, a_t) - \mathbb{Y})^{2}\right]        
    \end{displaymath}
where $\mathbb{Y} = r_t + \gamma \max_{a_{t+1}}Q_{\theta}(s_{t+1}, a_{t+1})$ is the TD target.

Based on the maximum entropy principle, Soft Actor-Critic (SAC)~\citep{soft_actor_critic} provides a more stable actor-critic method by introducing a soft value function.
Given transition data $(s_t, a_t, r_t, s_{t+1})$, SAC conducts TD update for the critic using the soft TD target $\mathbb{Y}^{soft}$.
$$
    \begin{aligned}
    &\mathbb{Y}^{soft} = r_t \\
    &+ \gamma \mathop{\mathbb{E}}_{a_{t+1}\sim\pi_\psi(\cdot|s_{t+1})}
    \left[ {Q_{\theta}}(s_{t+1}, a_{t+1}) - \log\pi_\psi(a_{t+1}|s_{t+1}) \right]\\
    \end{aligned}
$$
where $\pi_\psi$ is the policy function parameterized by $\psi$. 
For the policy $\pi_\psi$, it can be optimized by the gradient update:
$$
\nabla_\psi \mathop{\mathbb{E}}_{\hat{a}\sim\pi_\psi(\cdot|s_t)} \left[ \log \pi_\psi(\hat{a}|s_t) - Q_{\theta}(s_t, \hat{a}) \right]
$$
\section{Problem Setting}
\label{research_problem}
We assume that delay-free MDP is endowed with a constant delay variable $\Delta \in \mathbb{N}$.
In this setting, the state of environment $s_t$ is only observed by the agent at a later timestep ${t+\Delta}$.
In other words, the real state of the environment is $s_t$, but the agent's observation is $s_{t-\Delta}$.
To retrieve the Markov property in this Delayed MDP (DMDP)~\cite{altman_delay, delay_mdp}, we need to augment the state space $\mathcal{X} = \mathcal{S} \times \mathcal{A}^\Delta$, where $\mathcal{A}^\Delta$ stands for actions in delay time steps. An augmented state $x_t = (s_{t-\Delta}, a_{t-\Delta}, \ldots, a_{t-1}) \in \mathcal{X}$ is composed with the latest observed state $s_{t-\Delta}$ and actions taken in  last $\Delta$ time steps $(a_{t-\Delta}, \ldots, a_{t-1})$.
Thus, with consideration of the delay into the dynamics, we can formulate a new MDP dynamic called \textit{Constant Delayed MDP}~(CDMDP), $\langle \mathcal{X}, \mathcal{A}, \mathcal{P}_\Delta, \mathcal{R}_\Delta, \gamma, \rho_\Delta \rangle$, where $\mathcal{X}$ is defined above, $\mathcal{A}$ stands for the action-space. $\mathcal{P}_\Delta$ is the delayed probabilistic transition function defined below.\begin{displaymath}
\begin{aligned}
    &\mathcal{P}_\Delta(x_{t+1}|x_t, a_t) = \\
    &\mathcal{P}(s_{t-\Delta+1}|s_{t-\Delta}, a_{t-\Delta})\delta_{a_t}(a'_t)\Pi_{i=1}^{\Delta-1}\delta_{a_{t-i}}(a'_{t-i})
\end{aligned}
\end{displaymath}  
where $\delta$ is the \textit{Dirac distribution}.
We also have a new delayed reward function $\mathcal{R}_\Delta$ defined as follows.
\begin{equation*}
\label{delayed_reward_function}
\mathcal{R}_\Delta(x_t, a_t) = \mathop{\mathbb{E}}_{s_t\sim b(\cdot|x_t)}[\mathcal{R}(s_t, a_t)]
\end{equation*}
Correspondingly, the initial state distribution is represented as $\rho_\Delta =\rho\Pi_{i=1}^{\Delta}\delta_{a_{-i}}$, where $b(s_t|x_t)$ is called belief defined as follows.
\begin{equation}
\begin{aligned}
    \label{eq:belief_function}
    &b(s_t|x_t) =\\
    &\int_{\mathcal{S}^\Delta}\Pi_{i=0}^{\Delta-1}\mathcal{P}(s_{t-\Delta+i+1}|s_{t-\Delta+i}, a_{t-\Delta+i})\mathrm{d}{s_{t-\Delta+i+1}} \\
\end{aligned}
\end{equation}
The idea is to infuse delayed state information $s_{t-\Delta}$ to $s_t$ into the augmented state $x_t$~\citep{gangwani2020learning}.

In this work, we assume the MDPs, policies and Q-functions satisfy the following Lipschitz Continuity~(LC) property, where \textit{Euclidean distance} is adopted in a deterministic space (e.g., $d_\mathcal{S}$ for state space $\mathcal{S}$, $d_\mathcal{A}$ for action space $\mathcal{A}$ and $d_\mathcal{R}$ for reward space $\mathcal{R}$), and $L1$-\textit{Wasserstein distance}~\citep{villani2009optimal}, denoted as $W_1$, is used in a probabilistic space (e.g., transition space $\mathcal{P}$ and policy space $\Pi$) respectively.

\begin{definition}[\textbf{Lipschitz Continuous MDP}~\cite{rl_lipschitz_continuous}]
\label{smoot_assuption}
An MDP $\langle \mathcal{S}, \mathcal{A}, \mathcal{P}, \mathcal{R}, \gamma, \rho \rangle$ is $(L_\mathcal{P}, L_\mathcal{R})$-LC, if $\forall (s_1, a_1), (s_2, a_2) \in \mathcal{S}\times\mathcal{A}$, we have
$$
\small{
\begin{aligned}
    W_1(\mathcal{P}(\cdot|s_1, a_1)||\mathcal{P}(\cdot|s_2, a_2)) \leq L_\mathcal{P} (d_{\mathcal{S}}(s_1, s_2) + d_{\mathcal{A}}(a_1, a_2))\\
    d_{\mathcal{R}}(\mathcal{R}(s_1, a_1) - \mathcal{R}(s_2, a_2)) \leq L_\mathcal{R} (d_{\mathcal{S}}(s_1, s_2) + d_{\mathcal{A}}(a_1, a_2))
\end{aligned}}
$$  
\end{definition}

\begin{definition}[\textbf{Lipschitz Continuous Policy}~\cite{rl_lipschitz_continuous}]
A stationary markovian policy $\pi$ is $L_\pi$-LC, if $\forall s_1, s_2 \in \mathcal{S}$, we have
$$
    W_1(\pi(\cdot|s_1)||\pi(\cdot|s_2)) \leq L_\pi d_{\mathcal{S}}(s_1, s_2)
$$    
\end{definition}

\begin{definition}[\textbf{Lipschitz Continuous Q-function}~\cite{rl_lipschitz_continuous}]
\label{assume_lc_q}
Given $(L_\mathcal{P}, L_\mathcal{R})$-LC MDP and $L_\pi$-LC policy $\pi$, such that $\gamma L_P(1+L_\pi) < 1$ where $\gamma$ is the discount factor of MDP, then Q-function $Q^\pi$ is $L_Q$-LC with $L_Q = \frac{L_R}{1-\gamma L_P(1+L_\pi)}$.
\end{definition}

Note that is a mild assumption commonly used in RL literature~\cite{rl_lipschitz_continuous, dida}.

\section{Our Approach: Auxiliary-Delayed RL}
\label{boad_approch}

In this section, we introduce our AD-RL framework to address the sample inefficiency of the original augmentation-based approach and illustrate the underlying relation between learning the original delayed task and the auxiliary one in Section~\ref{method_adrl}. Then we extend the AD-RL framework to the practical algorithms in Section~\ref{method_advi} and Section~\ref{method_adspi}.
 
\begin{figure}[h]
    \vskip -0.1in
    \begin{center}
        \centerline{\includegraphics[width=0.9\columnwidth]{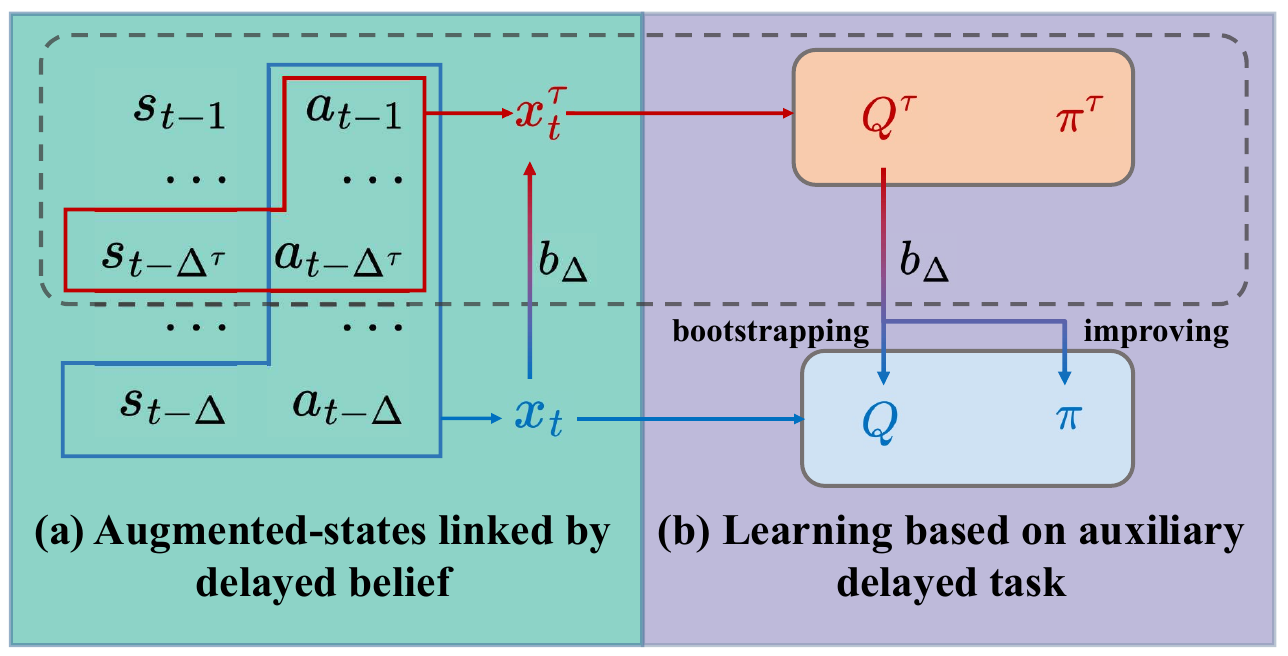}}
        \caption{The overview of AD-RL framework. Compared with the conventional augmentation-based method, AD-RL additionally introduces the auxiliary-delayed task shown in the dashline box.}
        \label{boad_diagram}
    \end{center}
    \vskip -0.3in
\end{figure}

\subsection{Auxiliary-Delayed Reinforcement Learning}
\label{method_adrl}

Instead of learning on the original augmented state space \textcolor{lightblue}{$\mathcal{X}$} with delays \textcolor{lightblue}{$\Delta$}, we introduce a corresponding auxiliary-delayed task with the shorter delays \textcolor{lightred}{$\Delta^{\tau}$}(\textcolor{lightred}{$\Delta^{\tau}$}$<$\textcolor{lightblue}{$\Delta$}) and, accordingly, a much smaller augmented state space \textcolor{lightred}{$\mathcal{X}^\tau$}.
Sharing a similar idea as belief function $b$ in Eq.~\eqref{eq:belief_function}, \textcolor{lightblue}{$\mathcal{X}$} and \textcolor{lightred}{$\mathcal{X}^\tau$} can be bridged by a delayed belief function $b_\Delta$ as:
\begin{equation}
\begin{aligned}
    &b_\Delta(\textcolor{lightred}{x^{\tau}_t}|\textcolor{lightblue}{x_t})=\\
    &\int_{\mathcal{S}^\Delta}\Pi_{i=0}^{\Delta-\Delta^{\tau}-1}\mathcal{P}(s_{t-\Delta+i+1}|s_{t-\Delta+i}, a_{t-\Delta+i})\mathrm{d}{s_{t-\Delta+i+1}}\\
\end{aligned}    
\end{equation}
where \textcolor{black}{$\textcolor{lightblue}{x_t}=(s_{t-\Delta}, a_{t-\Delta}, \ldots, a_{t-1})\in \textcolor{lightblue}{\mathcal{X}}$} and \textcolor{black}{$\textcolor{lightred}{x^{\tau}_t} = (s_{t-{\Delta^{\tau}}}, a_{t-\Delta^{\tau}}, \ldots, a_{t-1})\in \textcolor{lightred}{\mathcal{X}^\tau}$}.
As shown in the Fig.~\ref{boad_diagram} (a), both \textcolor{lightblue}{$x_t$} and \textcolor{lightred}{$x^{\tau}_t$} share the sub-sequence $(a_{t-{\Delta^{\tau}}}, \ldots, a_{t-1})$ of $\mathcal{A}^\Delta$. Besides, in the original MDP setting, transitioning from the state $s_{t-{\Delta}}$ to the state $s_{t-{\Delta^{\tau}}}$ can be accomplished by applying the action sub-sequence $(a_{t-{\Delta}}, \ldots, a_{t-{\Delta^{\tau}}+1})$.
\begin{remark}[Implicit Delayed Belief]
    Practically, we do not need to learn the delayed belief $b_\Delta$ explicitly.
    As in the CDMDP, every state will be observed by the agent eventually. 
    In other words, given an entire trajectory collected by the agent, we can create the synthetic augmented state for any required delay.
\end{remark}

With $b_\Delta$ we can transform learning the original \textcolor{lightblue}{$\Delta$}-delayed task into learning the auxiliary \textcolor{lightred}{$\Delta^\tau$}-delayed task which is much easier to learn for a much smaller augmented state space.
As shown in Fig.~\ref{boad_diagram} (b), we can use the easier-to-learn auxiliary Q-function \textcolor{lightred}{$Q^\tau$} to help bootstrapping the Q-function \textcolor{lightblue}{$Q$} or improving the policy \textcolor{lightblue}{$\pi$}. 
The specific algorithms will be proposed in the next sections.
In this way, we can significantly improve the learning efficiency of the \textcolor{lightblue}{$\Delta$}-delayed task, and a more rigorous proof will be presented in Section~\ref{boad_theory}.

As a highly flexible delayed RL framework (Algorithm~\ref{boad_algo}), our AD-RL can be naturally embedded in most of the existing RL methods to serve different task specifications. 
In this paper, we specifically develop two practical algorithms AD-DQN and AD-SAC based on DQN~\citep{deep_q_network} and SAC~\citep{soft_actor_critic} to tackle discrete and continuous control tasks respectively. 
\begin{remark}
BPQL~\cite{kim2023belief} can be seen as a special case of our AD-RL via setting the auxiliary delays to fixed zero (i.e., \textcolor{lightred}{$\Delta^\tau$} $=0$). However, the excessive loss of information might lead to poor performance in stochastic MDP in Fig.~\ref{fig:sto_aux_delay_impact}. We provide more experimental results about this in Section \ref{boad_experiment}. 
\end{remark}

\begin{algorithm}[t]
   \caption{Auxiliary-Delayed RL Framework}
   \label{boad_algo}
   \begin{algorithmic}
       \STATE {\bfseries Input:} \textcolor{lightblue}{$Q$}, \textcolor{lightblue}{$\pi$} for \textcolor{lightblue}{$\Delta$} delays; 
       \textcolor{lightred}{$Q^\tau$}, \textcolor{lightred}{$\pi^\tau$} for \textcolor{lightred}{$\Delta^\tau$} delays
       \FOR{each update step}
           \STATE \textcolor{black!30}{\# Learning \textcolor{lightred!30}{$\Delta^\tau$}-delayed task}
           \STATE Updating \textcolor{lightred}{$Q^\tau$}, \textcolor{lightred}{$\pi^\tau$} by a given delayed RL method
           \STATE \textcolor{black!30}{\# Learning \textcolor{lightblue!30}{$\Delta$}-delayed task based on \textcolor{red!30}{$Q^\tau$}}
           \STATE Bootstrapping \textcolor{lightblue}{$Q$} based on \textcolor{lightred}{$Q^\tau$} via Eq.~\eqref{aux_vi} \textcolor{black!30}{$\#$ discrete}
           \STATE Improving \textcolor{lightblue}{$\pi$} based on \textcolor{lightred}{$Q^\tau$} via Eq.~\eqref{aux_pi}\textcolor{black!30}{$\#$ continuous}
        \ENDFOR
       \STATE {\bfseries Output:} \textcolor{lightblue}{$Q$}, \textcolor{lightblue}{$\pi$}
    \end{algorithmic}
\end{algorithm}

\subsection{Discrete Control: from AD-VI to AD-DQN}
\label{method_advi}
Before developing the practical algorithm: AD-DQN, we first have to derive AD-VI, the AD-RL version of value iteration (VI)~\cite{rlai}.
In the tabular setting, AD-VI maintains two Q-functions \textcolor{lightblue}{$Q$} and \textcolor{lightred}{$Q^{\tau}$} for the original delayed task(\textcolor{lightblue}{$\Delta$}) and the auxiliary-delayed task(\textcolor{lightred}{$\Delta^{\tau}$}), respectively.
Different from \textcolor{lightred}{$Q^{\tau}$} updated by the original Bellman operator, we update \textcolor{lightblue}{$Q$} by applying the auxiliary-delayed Bellman operator $\mathcal{T}$ as follow:
\begin{equation}
    \label{aux_vi}
    \begin{aligned}
    &\mathcal{T}\textcolor{lightblue}{Q}(\textcolor{lightblue}{x_t}, a_t) \triangleq \textcolor{lightblue}{\mathcal{R}_\Delta}(\textcolor{lightblue}{x_t}, a_t) \\
        &+ \gamma \mathop{\mathbb{E}}_{
        \textcolor{lightred}{x^{\tau}_{t+1}}\sim b_\Delta(\cdot|\textcolor{lightblue}{x_{t+1}})\atop
        \textcolor{lightblue}{x_{t+1}}\sim \mathcal{P}_\Delta(\cdot|\textcolor{lightblue}{x_{t}}, a_{t})
        }
        \left[
        \textcolor{lightred}{Q^{\tau}}(\textcolor{lightred}{x^{\tau}_{t+1}}, {\mathop{\arg\max}_{a_{t+1}}}\textcolor{lightblue}{Q}(\textcolor{lightblue}{x_{t+1}}, a_{t+1}))
        \right]\\
    \end{aligned}
\end{equation}
Then AD-DQN can be extended from AD-VI naturally via approximating the Q-functions by the parameterized functions (e.g., neural networks). 
The implementation details of AD-DQN are presented in Appendix \ref{boad_implementation_detail}.

\subsection{Continuous Control: from AD-SPI to AD-SAC}
\label{method_adspi}
Similarly, before AD-SAC, we begin with deriving AD-SPI, soft policy iteration (SPI)~\cite{soft_actor_critic} in the context of our AD-RL. AD-SPI also alternates between two steps: policy evaluation and policy improvement.
In policy evaluation, we evaluate the policy \textcolor{lightblue}{$\pi$} via iteratively applying the auxiliary-delayed soft Bellman operator $\mathcal{T}^{\pi}$ as follow:
\begin{equation}
    \label{aux_pe}
    \begin{aligned}
        &\mathcal{T}^\pi \textcolor{lightblue}{Q(x_t, \textcolor{black}{a_t})} \triangleq \textcolor{lightblue}{\mathcal{R}_\Delta}(\textcolor{lightblue}{x_t}, a_t) \\
        &+ \gamma \mathop{\mathbb{E}}_{\substack{\textcolor{lightblue}{a_{t+1}\sim\pi(\cdot|x_{t+1})}\\
        \textcolor{lightred}{x^{\tau}_{t+1}}\sim b_\Delta(\cdot|\textcolor{lightblue}{x_{t+1}})\\
        \textcolor{lightblue}{x_{t+1}}\sim \mathcal{P}_\Delta(\cdot|\textcolor{lightblue}{x_{t}}, a_{t})
        }}
        \left[
        \textcolor{lightred}{Q^{\tau}(x^{\tau}_{t+1}}, \textcolor{lightblue}{a_{t+1}}) - \log\textcolor{lightblue}{\pi(a_{t+1}|x_{t+1})}
        \right]\\
    \end{aligned}
\end{equation}
where \textcolor{lightred}{$Q^{\tau}$} is updated by the original soft Bellman operator~\cite{soft_actor_critic}. In policy improvement, we will update the policy \textcolor{lightblue}{$\pi$} based on \textcolor{lightred}{$Q^{\tau}$} instead of \textcolor{lightblue}{$Q$} as follow:
\begin{equation}
    \label{aux_pi}
    \begin{aligned}
    &\textcolor{lightblue}{\pi_{new}} =\\ 
    & {\arg\min}_{\textcolor{lightblue}{\pi'\in \Pi}}\text{KL}\left(\textcolor{lightblue}{\pi'(\cdot|x_t)}
    \middle|\middle|
    \frac{\exp(\mathop{\mathbb{E}}_{
        \textcolor{lightred}{x^{\tau}_{t}}\sim b_\Delta(\cdot|\textcolor{lightblue}{x_{t}})}[\textcolor{lightred}{Q^{\tau}(x^{\tau}_{t},\cdot)}])}{\textcolor{lightred}{Z(x^{\tau}_{t},\cdot)}}
    \right)\\
    \end{aligned}
\end{equation}
where \textcolor{lightblue}{$\Pi$} is the set of policies for tractable learning, $\text{KL}$ is the Kullback-Leibler divergence and \textcolor{lightred}{$Z$} is the term for normalizing the distribution.
Functions approximations of the Q-functions (\textcolor{lightblue}{$Q$} and \textcolor{lightred}{$Q^\tau$}) and policies (\textcolor{lightblue}{$\pi$} and \textcolor{lightred}{$\pi^\tau$}) bring us AD-SAC, the practical algorithm for continuous control task.
In AD-SAC, the policy \textcolor{lightblue}{$\pi_\psi$} parameterized by \textcolor{lightblue}{$\psi$} is updated by gradient with
\begin{equation}
    \triangledown_\psi \mathop{\mathbb{E}}_{
    \textcolor{lightblue}{\hat{a}\sim\pi_\psi(\cdot|x_t)} \atop
    \textcolor{lightred}{x^{\tau}_t}\sim b_\Delta(\cdot|\textcolor{lightblue}{x_t})
    }
    \left[
    \log {\textcolor{lightblue}{\pi_\psi(\hat{a}|x_t)}} - {\textcolor{lightred}{Q^{\tau}(x^{\tau}_t}}, {\textcolor{lightblue}{\hat{a}}})
    \right]
\end{equation}
In addition, we further improve AD-SAC with multi-step value estimation~\citep{rlai, dcac} for accelerating learning.
The implementation details of AD-SAC are presented in Appendix \ref{boad_implementation_detail}.

\section{Theoretical Analysis}
\label{boad_theory}
In this section, we first discuss why our AD-RL has better \textbf{sample efficiency} in Section \ref{efficiency_analysis}, then analyse the \textbf{performance gap} between optimal auxiliary-delayed value function and optimal delayed value function in Section \ref{bounding_analysis}, and finally derive the \textbf{convergence guarantee} of our AD-RL in Section \ref{convergence_analysis}.

\subsection{Sample Efficiency Analysis}
\label{efficiency_analysis}

Though it is hard to directly derive a formal conclusion on the sample efficiency of AD-RL, as the learning process is correlated to two different learning tasks at the same time, some existing works~\citep{lr_q_learning, speedy_q_learning} related to sample complexity provide insight into why our method has better sample efficiency.
The sample complexity of the optimized Q-learning is $\mathcal{O}\left(\frac{\log(|\mathcal{S}||\mathcal{A}|)}{\epsilon^{2.5} (1-\gamma)^5}\right)$~\citep{speedy_q_learning}, which shows the amount of samples are required for Q-learning to guarantee an $\epsilon$-optimal Q-function with high confidence. We can conclude that the sample complexity of augmented Q-learning in the augmented state space with delay $\Delta$ is $\mathcal{O}\left(\frac{\log(|\mathcal{S}||\mathcal{A}|^{\Delta+1})}{\epsilon^{2.5} (1-\gamma)^5}\right)$. 
Then our AD-RL makes bootstrapping in the auxiliary \textcolor{lightred}{$\Delta^\tau$}-augmented state-space instead of the original \textcolor{lightblue}{$\Delta$}-augmented state-space, the sample efficiency is improved by $\mathcal{O}(|\mathcal{A}|^{\textcolor{lightblue}{\Delta} - \textcolor{lightred}{\Delta_\tau}})$. 
\begin{remark}[Sample Inefficiency Issue]
AD-RL provides an effective framework to alleviate the sample inefficiency issue.
However, as shown in Fig.~\ref{boad_diagram} (d), in the stochastic environment with longer delays $\Delta$, we need to set relatively longer auxiliary delays $\Delta^\tau$ to achieve better performance while somewhat compromising the sample efficiency.
\end{remark}

\subsection{Performance Gap}
\label{bounding_analysis}

While acknowledging that bootstrapping in a smaller augmented state space combined with delayed belief might lead to sub-optimal performance, we demonstrate in this section that this degradation can be effectively bounded.
We start by deriving the Lemma \ref{lemma_general_delayed_performance_diff}, unifying performance gap between policies (\textcolor{lightblue}{$\pi$} and \textcolor{lightred}{$\pi^{\tau}$}) on the same auxiliary-delayed state space \textcolor{lightred}{$\mathcal{X}^\tau$}. 
Then, we derive the bounds on the performance gap through the difference of policies in Theorem \ref{delayed_performance_difference_bound}. 
Next, we extend this bound to get the bound on Q-functions of different state spaces in Theorem \ref{delayed_q_value_diff_bound}.
Finally, we show that the bound on optimal Q-functions will become nominal under the deterministic MDP setting.

Following the similar proof sketch with~\cite{optimal_approximate_rl, dida}, we give the general delayed policies performance difference lemma as below.
\begin{lemma}[General Delayed Performance Difference, see Appendix \ref{appendix_general_delayed_performance_diff} for proof]
    \label{lemma_general_delayed_performance_diff}
    For policies \textcolor{lightred}{$\pi^{\tau}$} and \textcolor{lightblue}{$\pi$}, with delays \textcolor{lightred}{$\Delta^{\tau}$} $<$ \textcolor{lightblue}{$\Delta$}. Given any \textcolor{lightblue}{$x_t$} $\in$ \textcolor{lightblue}{$\mathcal{X}$}, the performance difference is denoted as $I$(\textcolor{lightblue}{$x_t$})
    $$
    \begin{aligned}
        I(\textcolor{lightblue}{x_t}) &= \mathop{\mathbb{E}}_{\textcolor{lightred}{x^{\tau}_t}\sim b_\Delta(\cdot|\textcolor{lightblue}{x_t})}\left[\textcolor{lightred}{V^{\tau}(x^{\tau}_t)}\right] - \textcolor{lightblue}{V(x_t)}\\
        &= \frac{1}{1-\gamma} \mathop{\mathbb{E}}_{\substack{
        \textcolor{lightred}{\hat{x}^{\tau}} \sim b_\Delta(\cdot|\textcolor{lightblue}{\hat{x}})\\
        \textcolor{lightblue}{a\sim\pi(\cdot|\hat{x})}\\
        \textcolor{lightblue}{\hat{x} \sim d_{x_t}^{\pi}}
        }}
        \left[\textcolor{lightred}{V^{\tau}(\hat{x}^{\tau})} - \textcolor{lightred}{Q^{\tau}}(\textcolor{lightred}{\hat{x}^{\tau}}, \textcolor{lightblue}{a})\right]\\
    \end{aligned}
    $$
\end{lemma}
Lemma \ref{lemma_general_delayed_performance_diff} tells us that the performance difference $I$ between policies (\textcolor{lightblue}{$\pi$} and \textcolor{lightred}{$\pi^{\tau}$}) can be measured by the corresponding value functions (\textcolor{lightblue}{$V$} and \textcolor{lightred}{$V^{\tau}$}) sitting on different augmented state spaces (\textcolor{lightblue}{$\mathcal{X}$} and \textcolor{lightred}{{$\mathcal{X}^\tau$}}). Since the connection of \textcolor{lightblue}{$\mathcal{X}$} and \textcolor{lightred}{$\mathcal{X}^\tau$} can be specified by the delayed belief $b_\Delta$, we can unify the expressions on \textcolor{lightred}{$\mathcal{X}^\tau$} to measure the performance gap, using the auxiliary value functions (\textcolor{lightred}{$V^{\tau}$} and \textcolor{lightred}{$Q^{\tau}$}). 
Assuming \textcolor{lightred}{$Q^{\tau}$} is $L_Q$-LC (Definition \ref{assume_lc_q}), we can show that this performance difference between policies (\textcolor{lightblue}{$\pi$} and \textcolor{lightred}{$\pi^{\tau}$}) can be bounded by the $L_1$-Wassestein distance between them as followed Theorem \ref{delayed_performance_difference_bound}.

\begin{theorem}[Delayed Performance Difference Bound, proof in Appendix \ref{appendix_delayed_performance_difference_bound}]
    \label{delayed_performance_difference_bound}
    For policies \textcolor{lightred}{$\pi^{\tau}$} and \textcolor{lightblue}{$\pi$}, with \textcolor{lightred}{$\Delta^{\tau}$} $<$ \textcolor{lightblue}{$\Delta$}. Given any \textcolor{lightblue}{$x_t$} $\in$ \textcolor{lightblue}{$\mathcal{X}$}, if \textcolor{lightred}{$Q^{\tau}$} is $L_Q$-LC, the performance difference between policies can be bounded as follow
    $$
    \begin{aligned}
        &\mathop{\mathbb{E}}_{
        \textcolor{lightred}{x^{\tau}_t}\sim b_\Delta(\cdot|\textcolor{lightblue}{x_t})\atop
        \textcolor{lightblue}{a_t\sim\pi(\cdot|x_t)}
        }\left[
        \textcolor{lightred}{V^{\tau}(x^{\tau}_t)}
        -
        \textcolor{lightred}{Q^{\tau}}(\textcolor{lightred}{x^{\tau}_t}, \textcolor{lightblue}{a_t})
        \right]\\
        &\leq 
        L_Q \mathop{\mathbb{E}}_{\textcolor{lightred}{x^{\tau}_t}\sim b_\Delta(\cdot|\textcolor{lightblue}{x_t})}
        \left[{\mathcal{W}_1
        (\textcolor{lightred}{\pi^{\tau}(\cdot|x^{\tau}_t)}} || 
        \textcolor{lightblue}{\pi(\cdot|x_t)})
        \right]\\
    \end{aligned}
    $$
\end{theorem}
Combining Lemma \ref{lemma_general_delayed_performance_diff} and  Theorem \ref{delayed_performance_difference_bound}, we can extend the bound on state-values (\textcolor{lightblue}{$V$} and \textcolor{lightred}{$V^{\tau}$}) to the bound on Q-values (\textcolor{lightblue}{$Q$} and \textcolor{lightred}{$Q^{\tau}$}) and optimal Q-values (\textcolor{lightblue}{$Q_{(*)}$} and \textcolor{lightred}{$Q^{\tau}_{(*)}$}) by the $L_1$-Wassestein distance of the corresponding policies (\textcolor{lightblue}{$\pi$} and \textcolor{lightred}{$\pi^{\tau}$}) and optimal policies (\textcolor{lightblue}{$\pi_{(*)}$} and \textcolor{lightred}{$\pi^{\tau}_{(*)}$}), respectively.

\begin{theorem}[Delayed Q-value Difference Bound, proof in Appendix \ref{appendix_delayed_q_value_difference_bound}]
    \label{delayed_q_value_diff_bound}
    For policies \textcolor{lightblue}{$\pi$} and \textcolor{lightred}{$\pi^{\tau}$}, with \textcolor{lightred}{$\Delta^{\tau}$} $<$ \textcolor{lightblue}{$\Delta$}. Given any \textcolor{lightblue}{$x_t$} $\in$ \textcolor{lightblue}{$\mathcal{X}$}, if \textcolor{lightred}{$Q^{\tau}$} is $L_Q$-LC, the corresponding Q-value difference can be bounded as follow
    $$
    \begin{aligned}
        &\mathop{\mathbb{E}}_{
        \textcolor{lightblue}{a_t\sim\pi(\cdot|x_t)}\atop
        \textcolor{lightred}{x^{\tau}_t}\sim b_\Delta(\cdot|\textcolor{lightblue}{x_t})
        }\left[\textcolor{lightred}{Q^{\tau}}(\textcolor{lightred}{x^{\tau}_t}, \textcolor{lightblue}{a_t}) - \textcolor{lightblue}{Q(x_t, a_t)}\right]& \\
         &\leq \frac{\gamma L_Q}{1-\gamma} \mathop{\mathbb{E}}_{\substack{
         \textcolor{lightred}{\hat{x}^{\tau}}\sim b_\Delta(\cdot|\textcolor{lightblue}{\hat{x}})\\
         \textcolor{lightblue}{\hat{x}\sim d_{x_{t+1}}^{\pi}}\\
         \textcolor{lightblue}{x_{t+1}\sim \mathcal{P}_{\Delta}(\cdot|x_t, a_t)}\\
         \textcolor{lightblue}{a_t\sim\pi(\cdot|x_t)}}}
         \left[\mathcal{W}_1(
         \textcolor{lightred}{\pi^{\tau}(\cdot|\hat{x}^{\tau})} || 
         \textcolor{lightblue}{\pi(\cdot|\hat{x})}
         )\right]& \\
    \end{aligned}
    $$
    Specially, for optimal policies \textcolor{lightblue}{$\pi_{(*)}$} and \textcolor{lightred}{$\pi^{\tau}_{(*)}$}, if \textcolor{lightred}{$Q^{\tau}_{(*)}$} is $L_Q$-LC, the corresponding optimal Q-value difference can be bounded as follow
    $$
        \begin{aligned}
            &\left|\left|\mathop{\mathbb{E}}_{\textcolor{lightred}{x^{\tau}_t}\sim b_\Delta(\cdot|\textcolor{lightblue}{x_t})}\left[\textcolor{lightred}{Q^{\tau}_{(*)}(\textcolor{lightred}{x^{\tau}_t}}, a_t)\right] - \textcolor{lightblue}{Q_{(*)}(x_t}, a_t)\right|\right|_{\infty}& \\
            &\leq
            \frac{\gamma^2 L_Q}{(1-\gamma)^2} \mathop{\mathbb{E}}_{\substack{
                \textcolor{lightred}{\hat{x}^{\tau}}\sim b_\Delta(\cdot|\textcolor{lightblue}{\hat{x}})\\
                \textcolor{lightblue}{\hat{x}\sim d_{x_{t+1}}^{\pi}}\\
                \textcolor{lightblue}{x_{t+1}\sim \mathcal{P}_{\Delta}(\cdot|x_t, a_t)}\\
                \textcolor{lightblue}{a_t\sim\pi_{(*)}(\cdot|x_t)}}}
                \left[\mathcal{W}_1\left(
                \textcolor{lightred}{\pi^{\tau}_{(*)}(\cdot|\hat{x}^{\tau})} \middle|\middle| 
                \textcolor{lightblue}{\pi_{(*)}(\cdot|\hat{x})}
                \right)\right]& \\
        \end{aligned}
    $$
\end{theorem}

The results provided in Theorem \ref{delayed_q_value_diff_bound}, however, is hard to derive the unifying insight further.
For instance, it is difficult to calculate
\textcolor{black}{$\mathcal{W}_1(\textcolor{lightred}{\pi^{\tau}_{(*)}(\cdot|\hat{x}^{\tau})} || \textcolor{lightblue}{\pi_{(*)}(\cdot|\hat{x})})$}, as the optimal policies (\textcolor{lightblue}{$\pi_{(*)}$} and \textcolor{lightred}{$\pi^\tau_{(*)}$}) indeed depend on the property of the underlying MDP. In the case of deterministic MDP where the optimal policies (\textcolor{lightblue}{$\pi_{(*)}$} and \textcolor{lightred}{$\pi^\tau_{(*)}$}) are the same, we can conclude that the optimal delayed Q-value difference becomes nominal in the following remark. 
\begin{remark}[Deterministic MDP Case]
    For deterministic MDP, $b_\Delta$ is also deterministic and becomes injection function meaning that given the \textcolor{lightblue}{$x$}, the \textcolor{lightred}{$x^{\tau}$} is determined.
    And then due to \textcolor{lightblue}{$\pi_{(*)}(\cdot|x)$} = $\mathop{\mathbb{E}}_{\textcolor{lightred}{x^{\tau}}\sim b_\Delta(\cdot|\textcolor{lightblue}{x})}\left[\textcolor{lightred}{\pi^{\tau}_{(*)}(\cdot|x^{\tau})}\right]$,
    we have $$
    \mathop{\mathbb{E}}_{\textcolor{lightred}{x^{\tau}_t}\sim b_\Delta(\cdot|\textcolor{lightblue}{x_t})}\left[\textcolor{lightred}{Q^{\tau}_{(*)}(\textcolor{lightred}{x^{\tau}_t}}, a_t)\right]
    =
    \textcolor{lightblue}{Q_{(*)}(x_t}, a_t)
    $$
\end{remark}
We also discuss stochastic MDPs as summarized in the following remark.
\begin{remark}[Stochastic MDP Case]
    In the case of stochastic MDP, the performance gap might become larger as the difference between \textcolor{lightblue}{$\Delta$} and \textcolor{lightred}{$\Delta^\tau$} increases.
    Using a moderate auxiliary delays \textcolor{lightred}{$\Delta^\tau$} could trade-off the sample efficiency (closer to $0$) and performance consistency (closer to \textcolor{lightblue}{$\Delta$}). We also provide experimental results to investigate this in Section \ref{boad_experiment}.
    Additionally, in Appendix \ref{appendix_sto}, we give a stochastic MDP case to exemplify the above conclusion.
\end{remark}

\subsection{Convergence Analysis}
\label{convergence_analysis}
We show in this section that our AD-RL does not sacrifice the convergence.
Before presenting the final result, we assume action space $\mathcal{A}$ is finite, which means $|\mathcal{A}| < \infty$. Then, for any \textcolor{lightblue}{$x$} $\in$ \textcolor{lightblue}{$\mathcal{X}$}, the $L_1$-Wassestein distance between policies \textcolor{lightblue}{$\pi$} and \textcolor{lightred}{$\pi^{\tau}$} becomes the bounded $l_1$ distance, and then the following holds
$$
    \mathop{\mathbb{E}}_{\textcolor{lightred}{x^{\tau}}\sim b_\Delta(\cdot|\textcolor{lightblue}{x})}
    \left[
        \mathcal{W}_1(
         \textcolor{lightred}{\pi^{\tau}(\cdot|x^{\tau})} || 
         \textcolor{lightblue}{\pi(\cdot|x)}
        )
    \right]
    < \infty
$$
Furthermore, the entropy of the policy \textcolor{lightblue}{$\pi$} is also bounded as $\mathcal{H}(\textcolor{lightblue}{\pi(\cdot|x)}) < \infty$.
Then, we show the convergence guarantee of AD-VI and AD-SPI in Section \ref{advi_convergence} and Section \ref{adspi_convergence} respectively.

\subsubsection{Convergence of AD-VI}
\label{advi_convergence}
We assume that auxiliary Q-function \textcolor{lightred}{$Q^\tau$} converges to the fixed point \textcolor{lightred}{$Q^\tau_{(*)}$} and the Q-function \textcolor{lightblue}{$Q$} is updated based on \textcolor{lightred}{$Q^\tau_{(*)}$}, since we only care about the final converged point of \textcolor{lightblue}{$Q$}.
Then, we can update an initial Q-function \textcolor{lightblue}{$Q_{(0)}$} by repeatedly applying the Bellman operator $\mathcal{T}$ given by Eq.~\eqref{aux_vi} to get the fixed point \textcolor{lightblue}{$Q_{(\approx)}$}$=\mathop{\mathbb{E}}_{\textcolor{lightred}{x^{\tau}}\sim b_\Delta(\cdot|\textcolor{lightblue}{x})}\left[\textcolor{lightred}{Q^{\tau}_{(*)}}(\textcolor{lightred}{x^{\tau}}, a)\right]$ (Theorem \ref{vi_convergence}).

\begin{theorem}[AD-VI Convergence Guarantee, proof in Appendix \ref{appendix_vi_convergence}]
\label{vi_convergence}
Consider the bellman operator $\mathcal{T}$ in Eq.~\eqref{aux_vi} and the initial Q-function \textcolor{lightblue}{$Q_{(0)}$}: $\textcolor{lightblue}{\mathcal{X}}\times\mathcal{A}\rightarrow\mathbb{R}$ with $|\mathcal{A}| < \infty$, and define a sequence \textcolor{lightblue}{$\{Q_{(k)}\}_{k=0}^\infty$} where \textcolor{lightblue}{$Q_{(k+1)}$}$= \mathcal{T}$\textcolor{lightblue}{$Q_{(k)}$}. As $k\rightarrow\infty$, \textcolor{lightblue}{$Q_{(k)}$} will converge to the fixed point \textcolor{lightblue}{$Q_{(\approx)}$} with \textcolor{lightred}{$Q^{\tau}$} converges to \textcolor{lightred}{$Q^{\tau}_{(*)}$}. And for any $(\textcolor{lightblue}{x_t}, a_t)$ $\in$ $\textcolor{lightblue}{\mathcal{X}} \times \mathcal{A}$, we have
$$
    \textcolor{lightblue}{Q_{(\approx)}}(\textcolor{lightblue}{x_t}, a_t)
    = 
    \mathop{\mathbb{E}}_{\textcolor{lightred}{x^{\tau}_t}\sim b_\Delta(\cdot|\textcolor{lightblue}{x_t})}\left[\textcolor{lightred}{Q^{\tau}_{(*)}}(\textcolor{lightred}{x^{\tau}_t}, a_t)\right]
$$
\end{theorem}
Theorem \ref{vi_convergence} guarantees the convergence of the Bellman operator $\mathcal{T}$ and ensures the stability and effectiveness of the learning process in the context of the corresponding practical method, AD-DQN.

\begin{table*}[t]
\caption{Results of MuJoCo tasks with 25 delays for 1M global time-steps. Each method was evaluated with 10 trials and is shown with the standard deviation denoted by $\pm$. The best performance is in blue.}
\centering
\scalebox{0.75}{
\begin{tabular}{c|cccccccccc}
\hline
Delays=25        
    & Ant-v4              
    & HalfCheetah-v4      
    & Hopper-v4           
    & Humanoid-v4         
    & HumanoidStandup-v4    
    & Pusher-v4          
    & Reacher-v4        
    & Swimmer-v4         
    & Walker2d-v4         
\\ \hline
A-SAC
    & $0.07_{\pm 0.07}$
    & $0.04_{\pm 0.01}$
    & $0.13_{\pm 0.04}$
    & $0.05_{\pm 0.01}$
    & $0.97_{\pm 0.09}$
    & $0.49_{\pm 0.32}$
    & $0.96_{\pm 0.02}$
    & $0.72_{\pm 0.02}$
    & $0.12_{\pm 0.02}$
\\
DC/AC
    & $0.19_{\pm 0.02}$
    & $0.16_{\pm 0.07}$
    & $0.19_{\pm 0.04}$
    & $0.04_{\pm 0.01}$
    & $1.03_{\pm 0.03}$
    & $1.12_{\pm 0.02}$
    & \textcolor{blue}{$1.00_{\pm 0.00}$}
    & $0.78_{\pm 0.12}$
    & $0.26_{\pm 0.08}$
\\
DIDA
    & $0.29_{\pm 0.07}$
    & $0.12_{\pm 0.03}$
    & $0.27_{\pm 0.08}$
    & $0.07_{\pm 0.00}$
    & $0.97_{\pm 0.02}$
    & $1.04_{\pm 0.01}$
    & $0.98_{\pm 0.01}$
    & $0.93_{\pm 0.09}$
    & $0.10_{\pm 0.02}$
\\
BPQL
    & $0.57_{\pm 0.11}$
    & \textcolor{blue}{$0.87_{\pm 0.04}$}
    & \textcolor{blue}{$1.21_{\pm 0.18}$}
    & $0.12_{\pm 0.01}$
    & $1.09_{\pm 0.05}$
    & $1.07_{\pm 0.06}$
    & $0.87_{\pm 0.05}$
    & $1.36_{\pm 0.56}$
    & $0.59_{\pm 0.30}$
\\
AD-SAC (ours) 
    & \textcolor{blue}{$0.66_{\pm 0.04}$}
    & $0.71_{\pm 0.12}$
    & $0.86_{\pm 0.25}$
    & \textcolor{blue}{$0.25_{\pm 0.16}$}
    & \textcolor{blue}{$1.15_{\pm 0.08}$}
    & \textcolor{blue}{$1.29_{\pm 0.03}$}
    & $0.98_{\pm 0.02}$
    & \textcolor{blue}{$2.52_{\pm 0.40}$}
    & \textcolor{blue}{$0.72_{\pm 0.11}$}
\\ \hline
\end{tabular}
}
\label{mujoco_results}
\end{table*}

\begin{figure*}[t]
\begin{center}
\centerline{
    \subfigure[(deterministic) 10 delays]{\includegraphics[width=0.32\linewidth]{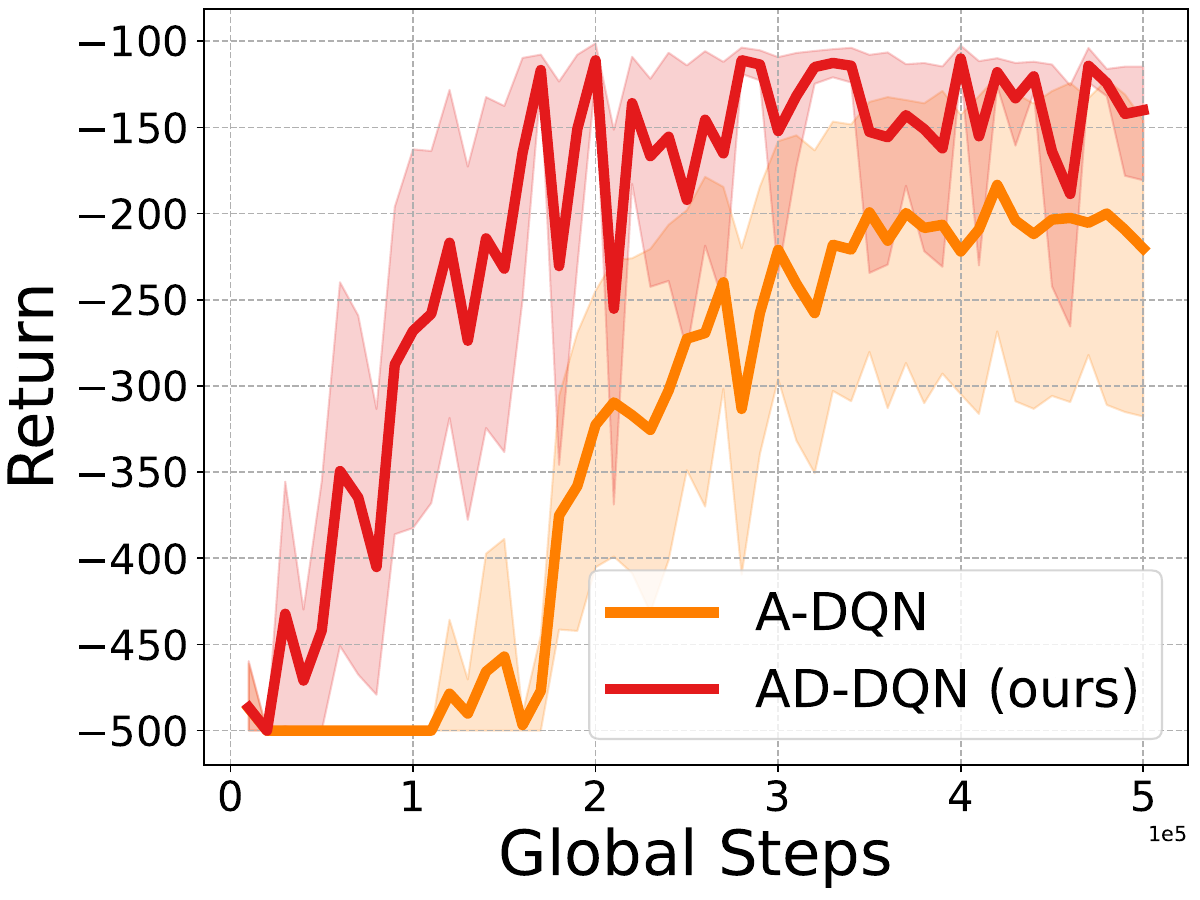} \label{fig:de_acrobot_results}}
    \subfigure[(deterministic) varying delays]{\includegraphics[width=0.32\linewidth]{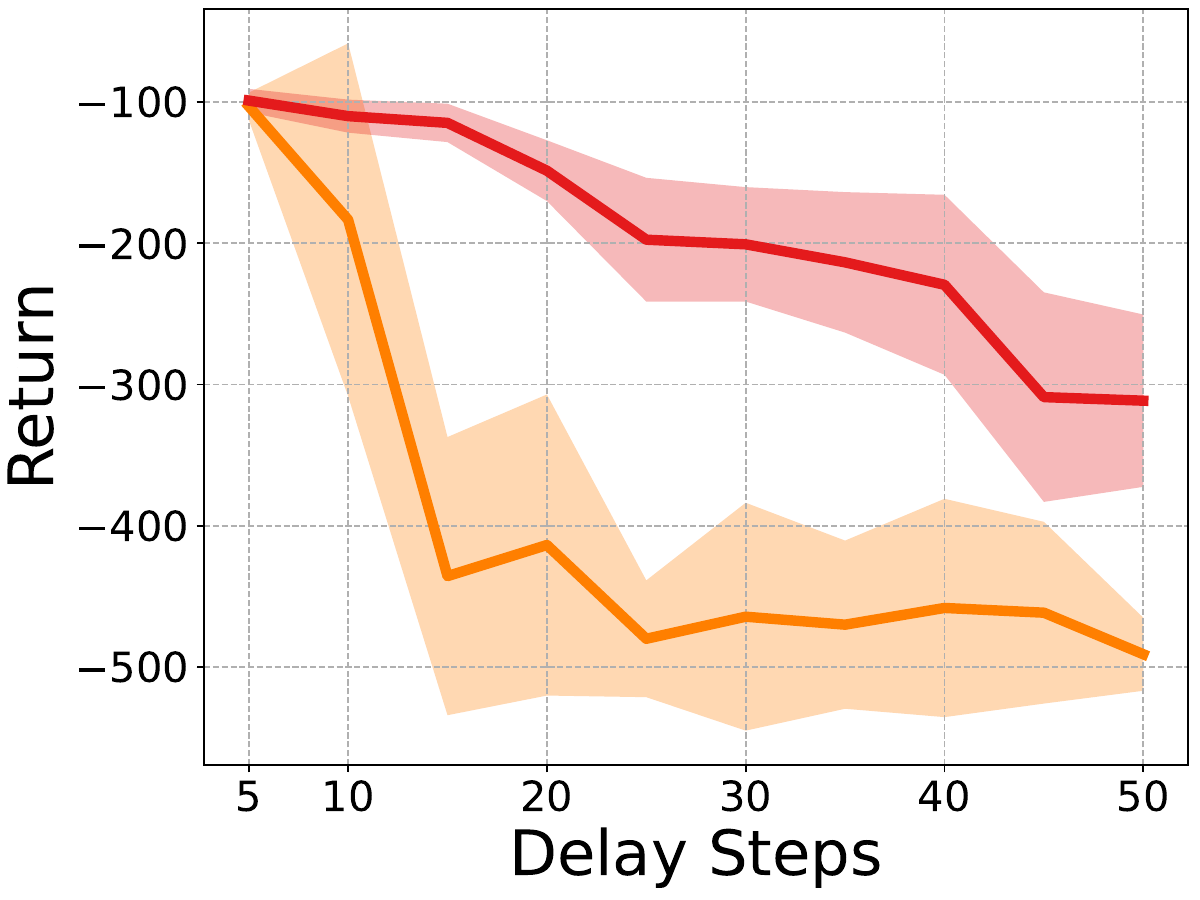} \label{fig:de_acrobot_delays}}
    \subfigure[(stochastic) normalized return of \textcolor{lightred}{$\Delta^\tau_{best}$}]{\includegraphics[width=0.33\linewidth]{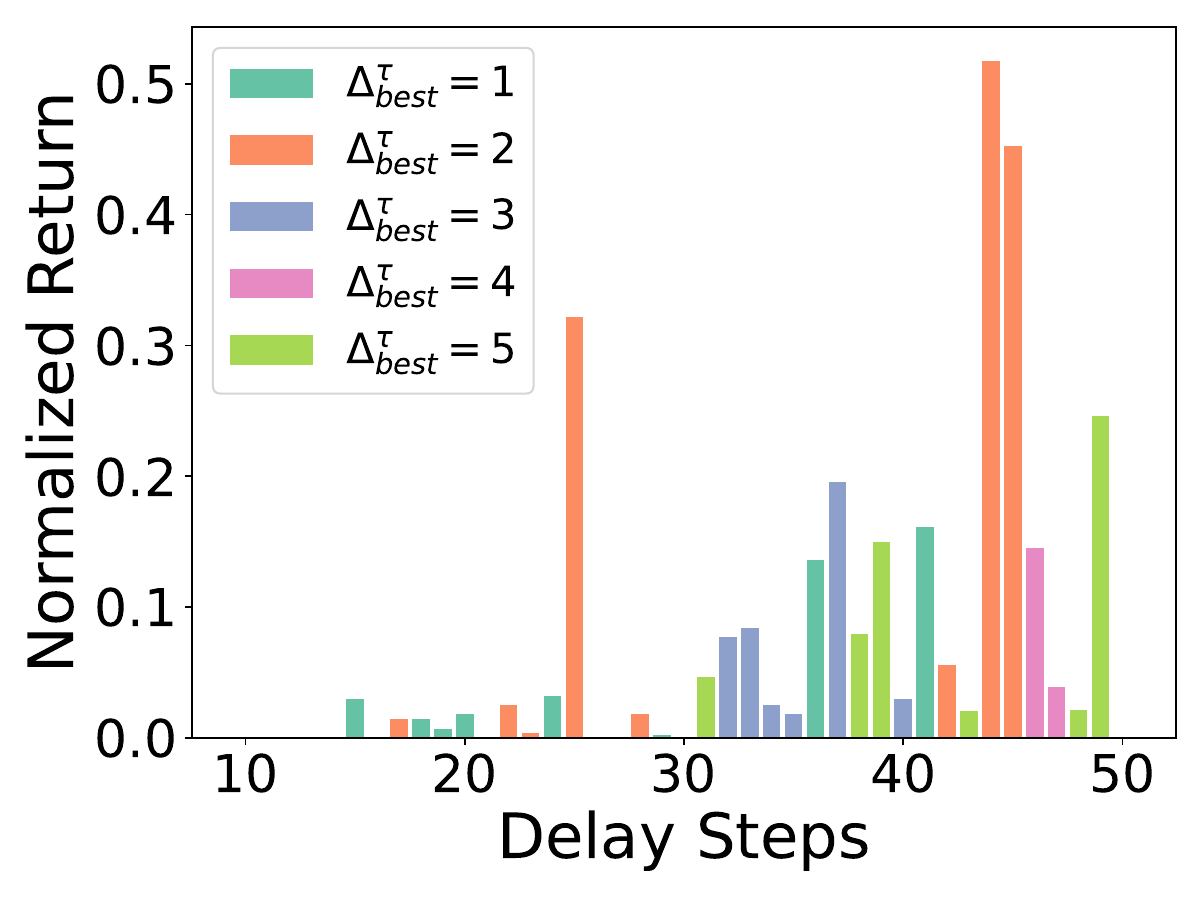} \label{fig:sto_acrobot_ret}}
}
\caption{Results of deterministic Acrobot for (a) learning curves with 10 delays and (b) final performance with varying delays (5-50). The shaded area is the standard deviation. Results of stochastic Acrobot for (c) the normalized return of \textcolor{lightred}{$\Delta^\tau_{best}$} with different delays (10-50). Different colors stand for different best returns achieved by different optimal auxiliary delays \textcolor{lightred}{$\Delta^\tau_{best}$}. }
\end{center}
\vskip -0.3in
\end{figure*}

\subsubsection{Convergence of AD-SPI}
\label{adspi_convergence}
Next, we derive the convergence guarantee of AD-SPI which consists of policy evaluation (Eq.~\eqref{aux_pe}) and policy improvement (Eq.~\eqref{aux_pi}).
Similar to AD-VI, we assume that \textcolor{lightred}{$Q^{\tau}$} has converged to the soft Q-value \textcolor{lightred}{$Q^{\tau}_{soft}$}~\cite{soft_q_learning, soft_actor_critic} in the context of AD-SPI.
For the policy evaluation, Q-function \textcolor{lightblue}{$Q$} can converge to a fixed point via iteratively applying the soft Bellman operator \textcolor{lightblue}{$\mathcal{T}^\pi$} defined in Eq.~\eqref{aux_pe}.
Lemma \ref{lemma_soft_policy_evaluation} shows this convergence guarantee.

\begin{lemma}[Policy Evaluation Convergence Guarantee, proof in Appendix \ref{appendix_soft_policy_evaluation}]
\label{lemma_soft_policy_evaluation}
Consider the soft bellman operator $\mathcal{T}^\pi$ in Eq.~\eqref{aux_pe} and the initial Q-value function \textcolor{lightblue}{$Q_{(0)}$}: $\textcolor{lightblue}{\mathcal{X}}\times\mathcal{A}\rightarrow\mathbb{R}$ with $|\mathcal{A}| < \infty$, and define a sequence \textcolor{lightblue}{$\{Q_{(k)}\}_{k=0}^\infty$} where \textcolor{lightblue}{$Q_{(k+1)}$}$= \mathcal{T}^\pi$\textcolor{lightblue}{$Q_{(k)}$}. Then for any $(\textcolor{lightblue}{x_t}, a_t)$ $\in$ $\textcolor{lightblue}{\mathcal{X}} \times \mathcal{A}$, as $k\rightarrow\infty$, \textcolor{lightblue}{$Q_{(k)}(x_t, \textcolor{black}{a_t})$} will converge to the fixed point $$
    \mathop{\mathbb{E}}_{
        \textcolor{lightred}{x^{\tau}_t}\sim b_\Delta(\cdot|\textcolor{lightblue}{x_t})
        }\left[
            \textcolor{lightred}{Q^{\tau}_{soft}}(\textcolor{lightred}{x^{\tau}_t}, a_t)
        \right]
        - \log\textcolor{lightblue}{\pi}(a_t|\textcolor{lightblue}{x_t})
    $$
\end{lemma}

For the soft policy improvement, we improve the old policy \textcolor{lightblue}{$\pi_{old}$} to the new one \textcolor{lightblue}{$\pi_{new}$} via applying the update rule in Eq.~\eqref{aux_pi}. And formalizing in Lemma \ref{lemma_soft_policy_improvement}, we can show that the improved policy \textcolor{lightblue}{$\pi_{new}$} is better than \textcolor{lightblue}{$\pi_{old}$}.

\begin{lemma}[Policy Improvement Guarantee, proof in Appendix \ref{appendix_soft_policy_improvement}]
    \label{lemma_soft_policy_improvement}
    Consider the policy update rule in Eq.~\eqref{aux_pi}, and let \textcolor{lightblue}{$\pi_{old}, \pi_{new}$} be the old policy and new policy improved from old one respectively. Then for any $(\textcolor{lightblue}{x_t}, a_t)$ $\in$ $\textcolor{lightblue}{\mathcal{X}} \times \mathcal{A}$ with $|\mathcal{A}| < \infty$, we have
    $
        \textcolor{lightblue}{Q_{old}(x_t}, a_t)
        \leq
        \textcolor{lightblue}{Q_{new}(x_t}, a_t)
    $.
\end{lemma}
Alternating between the policy evaluation and policy improvement, the policy learned by the AD-SPI will converge to a policy \textcolor{lightblue}{$\pi_{(\approx)}$} having the highest value among the policies in \textcolor{lightblue}{$\Pi$}. This result formalized in Theorem \ref{spi_convergence} establishes the theoretical fundamental for us to develop the AD-SAC.
\begin{theorem}[AD-SPI Convergence Guarantee]
    \label{spi_convergence}
    Applying policy evaluation in Eq.~\eqref{aux_pe} and policy improvement in Eq.~\eqref{aux_pi} repeatedly to any given policy \textcolor{lightblue}{$\pi\in\Pi$}, it converges to \textcolor{lightblue}{$\pi_{(\approx)}$} such that \textcolor{lightblue}{$Q^{\pi}(x_t$}$, a_t) \leq$ \textcolor{lightblue}{$Q^{\pi}_{(\approx)}(x_t$}$, a_t)$ for any \textcolor{lightblue}{$\pi\in\Pi$}, $(\textcolor{lightblue}{x_t}, a_t)$ $\in$ $\textcolor{lightblue}{\mathcal{X}} \times \mathcal{A}$ with $|\mathcal{A}| < \infty$.
\end{theorem}

\section{Experimental Results}
\label{boad_experiment}

\subsection{Benchmarks and Baselines}
\label{benchmark_baseline}
\textbf{Benchmarks.} We choose the Acrobot~\citep{sutton1995generalization} and the MuJoCo control suite~\citep{mujoco} for discrete and continuous control tasks, respectively.
Especially, to investigate the effectiveness of AD-RL in the stochastic MDPs, we adopt the stochastic Acrobot where the input from the agent is disturbed by the noise from a uniform distribution with the probability of $0.1$.
In this way, the output for the same inputted action from the agent becomes stochastic.

\textbf{Baselines.} We select a wide range of techniques as baselines to examine the performance of our AD-RL under different environments.
In Acrobot, we mainly compare AD-DQN against the augmented DQN (A-DQN)~\citep{deep_q_network}.
For the continuous control task, we compared our AD-SAC against existing SOTAs: augmented SAC (A-SAC)~\citep{soft_actor_critic}, DC/AC~\citep{dcac}, DIDA~\citep{dida} and BPQL~\citep{kim2023belief}.
In the deterministic Acrobot and Mujoco, the auxiliary delays \textcolor{lightred}{$\Delta^\tau$} of our AD-RL is set to $0$.
For computationally fairness, we keep all methods training with the same amount of gradient descent. 
In other words, the training times of \textcolor{lightblue}{$\pi$} in our AD-RL is \textbf{half of that in all baselines}, as AD-RL trains \textcolor{lightred}{$\pi^\tau$} additional to the original task at the same time.
The result of each method was obtained using 10 random seeds and its hyper-parameters are in Appendix \ref{boad_implementation_detail}.

\begin{table*}[t]
\centering
\caption{Results of MuJoCo tasks with stochastic delays for 1M global time-steps. Each method was evaluated with 10 trials and is shown with the standard deviation denoted by $\pm$. The best performance is in blue.
\label{appendix::stochastic_delay}}
\begin{tabular}{c|c|c|c|c|c}
\hline
Task           & A-SAC             & DC/AC             & DIDA              & BPQL              & AD-SAC                                   \\ \hline
Ant-v4         & $0.18_{\pm 0.01}$ & $0.27_{\pm 0.02}$ & $0.55_{\pm 0.08}$ & $0.58_{\pm 0.12}$ & \textcolor{blue}{$0.69_{\pm 0.17}$}  \\
HalfCheetah-v4 & $0.36_{\pm 0.12}$ & $0.36_{\pm 0.18}$ & $0.75_{\pm 0.02}$ & $0.76_{\pm 0.16}$ & \textcolor{blue}{$1.03_{\pm 0.06}$} \\
Hopper-v4      & $0.85_{\pm 0.22}$ & $0.94_{\pm 0.29}$ & $0.31_{\pm 0.08}$ & $0.68_{\pm 0.34}$ & \textcolor{blue}{$1.05_{\pm 0.22}$} \\
Humanoid-v4    & $0.15_{\pm 0.06}$ & $0.67_{\pm 0.18}$ & $0.07_{\pm 0.01}$ & $0.40_{\pm 0.42}$ & \textcolor{blue}{$0.97_{\pm 0.07}$} \\
Standup-v4     & $1.03_{\pm 0.05}$ & $1.20_{\pm 0.08}$ & $1.00_{\pm 0.00}$ & $1.10_{\pm 0.07}$ & \textcolor{blue}{$1.26_{\pm 0.07}$} \\
Pusher-v4      & $1.11_{\pm 0.02}$ & $1.17_{\pm 0.02}$ & $1.02_{\pm 0.01}$ & $1.07_{\pm 0.05}$ & \textcolor{blue}{$1.22_{\pm 0.01}$} \\
Reacher-v4     & $0.98_{\pm 0.01}$ & $1.02_{\pm 0.01}$ & $1.02_{\pm 0.00}$ & $0.85_{\pm 0.11}$ & \textcolor{blue}{$1.05_{\pm 0.01}$} \\
Swimmer-v4     & $0.82_{\pm 0.10}$ & $1.47_{\pm 0.58}$ & $1.03_{\pm 0.02}$ & $1.53_{\pm 0.52}$ & \textcolor{blue}{$2.36_{\pm 0.64}$} \\
Walker2d-v4    & $0.68_{\pm 0.28}$ & $0.89_{\pm 0.08}$ & $0.54_{\pm 0.09}$ & $0.63_{\pm 0.39}$ & \textcolor{blue}{$1.19_{\pm 0.14}$} \\ \hline
\end{tabular}
\end{table*}
\subsection{Empirical Results Analysis}

\textbf{Deterministic Acrobot.}
We first compare A-DQN and our AD-DQN in the deterministic Acrobot under fixed 10 delays. From Fig.~\ref{fig:de_acrobot_results}, we can tell that within the 500k time steps, AD-DQN learns much faster than the A-DQN. In addition, as delays change from 5 to 50, from Fig.~\ref{fig:de_acrobot_delays}, A-DQN is not able to learn any useful policy after 20 delays. However, our AD-DQN shows robust performance even under 50 delays.

\textbf{MuJoCo.} The results for MuJoCo are presented in Table~\ref{mujoco_results}. Our AD-SAC and BPQL provide leading performance in most MuJoCo tasks. 
We report the delay-free normalized scores $Ret_{nor} = (Ret_{alg} - Ret_{rand}) / (Ret_{df} - Ret_{rand})$, where $Ret_{df}$ is the return of a delay-free agent trained by the soft actor-critic and $Ret_{rand}$ is the return of a random policy.
The experiment results support our argument that learning the auxiliary-delayed task facilitates agents to learn the original delayed task. In deterministic scenarios, the delayed belief estimation degenerates to a deterministic function for any auxiliary delays \textcolor{lightred}{$\Delta^\tau$}, and a smaller \textcolor{lightred}{$\Delta^\tau$} benefits the sampling efficiency. BPQL (a special variant of our approach when \textcolor{lightred}{$\Delta^\tau$}$ = 0$) and AD-SAC(0) thus provide comparable results. 
The results for MuJoCo tasks with 5 and 50 delays, presented in Appendix \ref{appendix_mujoco_results}, also validate this conclusion.

\textbf{MuJoCo with Stochastic Delays.}
We evaluate the performance and robustness of AD-RL in MuJoCo which has the delay $\Delta=5$ with a probability of 0.9, and the delay $\Delta\in [1, 5]$ with a probability of 0.1.
We adopt a similar stochastic delay setting as BPQL~\cite{kim2023belief} wherein the evaluation environment has the delay $\Delta=5$ with a probability of 0.9, and the delay $\Delta \in [1,5]$ with a probability of 0.1.
As shown in Table \ref{appendix::stochastic_delay}, AD-RL outperforms all the other baselines including BPQL in all tasks with stochastic delays. Compared with the result of AD-RL in the constant delay setting, we can see that AD-RL is more robust than others.

\textbf{Stochastic Acrobot.} Then we investigate the influence of auxiliary delays for AD-RL on the stochastic Acrobot.
We conduct a series of experiments for various combinations of delays (ranging from $10$ to $50$) and auxiliary delays (ranging from $0$ to $5$). The optimal auxiliary delays, denoted as \textcolor{lightred}{$\Delta^\tau_{best}$}, which yield the best performance under different delays, are recorded.
Furthermore, the relative return is defined as $Ret_{rela} = (Ret_{\Delta^\tau_{best}}-Ret_{0}) / (Ret_{0}-Ret_{rand})$ where $Ret_{\Delta^\tau_{best}}$ is the return of setting \textcolor{lightred}{$\Delta^\tau = \Delta^\tau_{best}$}, $Ret_{0}$ is the return of setting \textcolor{lightred}{$\Delta^\tau = 0$} and $Ret_{rand}$ is the return of the random policy.
It measures the comparatively improved performance in return by \textcolor{lightred}{$\Delta^\tau = \Delta^\tau_{best}$} instead of setting \textcolor{lightred}{$\Delta^\tau = 0$}.
As illustrated in Fig.~\ref{fig:sto_acrobot_ret}, we can observe that \textcolor{lightred}{$\Delta^\tau = 0$} (as in BPQL) may not always be the optimal choice and the selection of best auxiliary delays appears irregular in the stochastic MDP.
It is evident to confirm our aforementioned conclusion: the shorter auxiliary delays \textcolor{lightred}{$\Delta^\tau$} could improve the learning efficiency but also potentially result in performance degradation caused by a more significant information loss in stochastic environments. So there exists a trade-off between learning efficiency (\textcolor{lightred}{$\Delta^\tau$} closes to $0$) and performance consistency (\textcolor{lightred}{$\Delta^\tau$} closes to \textcolor{lightblue}{$\Delta$}).

\begin{table}[h]
\centering
\caption{The best auxiliary delays and corresponding normalized score of AD-RL in Stochastic MuJoCo with 50 delays.
\label{table::stochastic_delay}}
\begin{tabular}{c|c}
\hline
Delays=50          & Normalized score ($\Delta^\tau_{best}$) \\ \hline
Ant-v4             & $0.16\pm_{0.04}(4)$                     \\ \hline
HalfCheetah-v4     & $0.05\pm_{0.28}(2)$                     \\ \hline
Hopper-v4          & $0.05\pm_{0.17}(5)$                     \\ \hline
Humanoid-v4        & $0.00\pm_{0.00}(0)$                     \\ \hline
HumanoidStandup-v4 & $0.03\pm_{0.06}(1)$                     \\ \hline
Pusher-v4          & $0.02\pm_{0.01}(3)$                     \\ \hline
Reacher-v4         & $0.03\pm_{0.02}(1)$                     \\ \hline
Swimmer-v4         & $0.00\pm_{0.00}(0)$                     \\ \hline
Walker2d-v4        & $0.00\pm_{0.00}(0)$                     \\ \hline
\end{tabular}
\end{table}
\textbf{Stochastic MuJoCo.} 
We adopt a similar stochastic MuJoCo setting as BPQL~\cite{kim2023belief} wherein an agent-unaware action noise is added into the environment with a probability of $0.1$, making the MuJoCo stochastic (different outputs for the same input action).
In stochastic MuJoCo with $50$ delays, we train AD-RL with different auxiliary delays $\Delta^\tau$ ranging from 0 to 5, and we report the best auxiliary delays $\Delta^\tau_{best}$ and corresponding normalized scores $(Ret_{best} - Ret_{0}) / (Ret_{0} - Ret_{rand})$.
As shown in Table \ref{table::stochastic_delay}, the best choice is not always $0$, and it is sensitive to the specific task setting and shows strong irregularity.

\section{Related Work}
\label{related_works}

Early works model a continuous system with observation delay by Delay Differential Equations (DDE)~\citep{myshkis1955lineare, cooke1963differential} which have been extensively studied in the control community in terms of reachability~\citep{fridman2003reachable, xue2021reach}, stability~\citep{feng2019taming}, and safety~\citep{xue2020over}. These techniques rely on explicit dynamics models and cannot scale well. 

In recent RL approaches, the environment with observation delay is modelled under the MDP framework \citep{altman_delay, delay_mdp}. 
Unlike the well-studied reward delay problem which raises the credit assignment issue~\citep{sutton1984temporal, arjona2019rudder, wang2024highway}, the observation delay problem disrupts the Markovian property required for traditional RL. Existing techniques differ from each other in how to tame the delay, including memoryless-based, model-based, and augmentation-based methods. Inspired by the partially observed MDP (POMDP), memory-less approaches, e.g., dQ and dSARSA, were developed in~\citep{memoryless}, which learn the policy based on the observed state. 
However, these methods ignore the non-Markovian property of the problem and could lead to performance degeneration. 

Model-based methods retrieve the Markovian property by predicting the unobserved state and then selecting an action based on it. The performance thus highly relies on state generation techniques.  \citeauthor{learning_planning_delayed_feedback} propose a deterministic generative model which is learned via model-based simulation. Similarly, ~\citeauthor{non_stationary_model_based} suggested successively applying an approximate forward model to estimate state \citep{non_stationary_model_based}. Different stochastic generative models, e.g., the ensemble of Gaussian distributions~\cite{chen2021delay} and Transformers~\cite{learning_belief}, were also explored.
However, the non-negligible prediction error results in sub-optimal returns~\cite {learning_belief}.

Augmentation-based methods seek to equivalently retrieve the Markov property by augmenting the delay-related information into the state-space~\citep{altman_delay}.
For instance, inspired by multi-step learning, \citeauthor{dcac} develops a partial trajectory resampling technique to accelerate the learning process \citep{dcac}.
Additionally, based on imitation learning and dataset aggregation technique, \citeauthor{dida} trains an undelayed expert policy and subsequently generalizes the expert's behavior into an augmented policy \citep{dida}.
Despite possessing optimality and the Markov property, augmentation-based methods are plagued by the curse of dimensionality in facing tasks with long delays, resulting in learning inefficiency.
BPQL~\citep{kim2023belief} evaluates the augmented policy by a non-augmented Q-function. However, BPQL suffers from performance loss in stochastic environments.
Certain difficult problems can be solved by 'divide and conquer' methods such as the neural sequence chunker~\citep{jurgenneural} which learns to hierarchically decompose sequences into predictable subsequences, and the subgoal generator~\citep{schmidhuber1991learning}, which decomposes task into smaller tasks with shorter solutions by learning appropriate subgoals.

\section{Conclusion}
In this work, we focus on RL in environments with delays between events and their observations. Existing methods exhibit learning inefficiency in the presence of long delays and performance degradation in stochastic environments. 
To address these issues, we propose AD-RL, which leverages auxiliary tasks with short delays to enhance learning. Under the AD-RL framework, we develop AD-DQN and AD-SAC for discrete and continuous control tasks respectively.
We further provide a theoretical analysis in terms of sample efficiency, performance gap, and convergence.
In both deterministic and stochastic benchmarks, we empirically show that AD-RL achieves new state-of-the-art performance, dramatically outperforming existing methods.

\section{Limitations and Chanllenges}
\paragraph{Hyperparameter \textcolor{lightred}{$\Delta^\tau_{best}$}} It is worth noting that the performance of AD-RL is subject to the selection of the auxiliary delays \textcolor{lightred}{$\Delta^\tau_{best}$}. Such selection is deeply related to the specific tasks and even the delay $\Delta$, and is highly challenging: Fig.~\ref{fig:sto_acrobot_ret} demonstrates that the relation between \textcolor{lightred}{$\Delta^\tau_{best}$} and $\Delta$ is not linear or parabolic. We will consider a systemic framework to select the best auxiliary delay by exploring the underlying relationship between the auxiliary delay, the original delay and the task specification, etc.

\paragraph{Implicit Belief Learning} Learning delayed belief, especially in stochastic environments, proves to be challenging. AD-RL implicitly represents the belief function by sampling in two augmented state spaces (\textcolor{lightblue}{$\Delta$} and \textcolor{lightred}{$\Delta^\tau$}) to overcome this issue, leading to additional memory cost compared to conventional augmentation-based approaches. In the future, we will explore learning explicit belief using different neural representations.

\paragraph{Sample Efficiency} 
Sample efficiency remains a critical challenge, especially in environments characterized by long delays or stochasticity.
In both cases, AD-RL needs to set relatively longer auxiliary delays $\Delta^\tau$ to carry more information of the augmented state in the auxiliary state. While a longer auxiliary delay achieves better performance, it brings back the sample inefficiency issue, which will be investigated in the next step.

\clearpage
\section*{Acknowledgement}
We sincerely acknowledge the support by the grant EP/Y002644/1 under the EPSRC ECR International Collaboration Grants program, funded by the International Science Partnerships Fund (ISPF) and the UK Research and Innovation, and the grant 2324936 by the US National Science Foundation. 
This work is also supported by Taiwan NSTC under Grant Number NSTC-112-2221-E-002-168-MY3, and the European Research Council (ERC, Advanced Grant Number 742870).

\section*{Impact Statement}
This paper aims to advance the field of Machine Learning. While acknowledging there are many potential societal consequences of our work, we believe that none need to be specially highlighted here.

\bibliography{main}

\begin{thebibliography}{53}
\providecommand{\natexlab}[1]{#1}
\providecommand{\url}[1]{\texttt{#1}}
\expandafter\ifx\csname urlstyle\endcsname\relax
  \providecommand{\doi}[1]{doi: #1}\else
  \providecommand{\doi}{doi: \begingroup \urlstyle{rm}\Url}\fi

\bibitem[Altman \& Nain(1992)Altman and Nain]{altman_delay}
Altman, E. and Nain, P.
\newblock Closed-loop control with delayed information.
\newblock \emph{ACM sigmetrics performance evaluation review}, 20\penalty0 (1):\penalty0 193--204, 1992.

\bibitem[Arjona-Medina et~al.(2019)Arjona-Medina, Gillhofer, Widrich, Unterthiner, Brandstetter, and Hochreiter]{arjona2019rudder}
Arjona-Medina, J.~A., Gillhofer, M., Widrich, M., Unterthiner, T., Brandstetter, J., and Hochreiter, S.
\newblock Rudder: Return decomposition for delayed rewards.
\newblock \emph{Advances in Neural Information Processing Systems}, 32, 2019.

\bibitem[Azar et~al.(2011)Azar, Munos, Ghavamzadeh, and Kappen]{speedy_q_learning}
Azar, M.~G., Munos, R., Ghavamzadeh, M., and Kappen, H.
\newblock Speedy q-learning.
\newblock In \emph{Advances in neural information processing systems}, 2011.

\bibitem[Berner et~al.(2019)Berner, Brockman, Chan, Cheung, Debiak, Dennison, Farhi, Fischer, Hashme, Hesse, et~al.]{berner2019dota}
Berner, C., Brockman, G., Chan, B., Cheung, V., Debiak, P., Dennison, C., Farhi, D., Fischer, Q., Hashme, S., Hesse, C., et~al.
\newblock Dota 2 with large scale deep reinforcement learning.
\newblock \emph{arXiv preprint arXiv:1912.06680}, 2019.

\bibitem[Bertsekas(2012)]{bertsekas2012dynamic}
Bertsekas, D.
\newblock \emph{Dynamic programming and optimal control: Volume I}, volume~4.
\newblock Athena scientific, 2012.

\bibitem[Bouteiller et~al.(2020)Bouteiller, Ramstedt, Beltrame, Pal, and Binas]{dcac}
Bouteiller, Y., Ramstedt, S., Beltrame, G., Pal, C., and Binas, J.
\newblock Reinforcement learning with random delays.
\newblock In \emph{International conference on learning representations}, 2020.

\bibitem[Chen et~al.(2021)Chen, Xu, Li, and Zhao]{chen2021delay}
Chen, B., Xu, M., Li, L., and Zhao, D.
\newblock Delay-aware model-based reinforcement learning for continuous control.
\newblock \emph{Neurocomputing}, 450:\penalty0 119--128, 2021.

\bibitem[Cooke(1963)]{cooke1963differential}
Cooke, K.~L.
\newblock Differential—difference equations.
\newblock In \emph{International symposium on nonlinear differential equations and nonlinear mechanics}, pp.\  155--171. Elsevier, 1963.

\bibitem[Derman et~al.(2021)Derman, Dalal, and Mannor]{non_stationary_model_based}
Derman, E., Dalal, G., and Mannor, S.
\newblock Acting in delayed environments with non-stationary markov policies.
\newblock \emph{arXiv preprint arXiv:2101.11992}, 2021.

\bibitem[Even-Dar et~al.(2003)Even-Dar, Mansour, and Bartlett]{lr_q_learning}
Even-Dar, E., Mansour, Y., and Bartlett, P.
\newblock Learning rates for q-learning.
\newblock \emph{Journal of machine learning Research}, 5\penalty0 (1), 2003.

\bibitem[Fathi et~al.(2022)Fathi, Haghi~Kashani, Jameii, and Mahdipour]{fathi2022big}
Fathi, M., Haghi~Kashani, M., Jameii, S.~M., and Mahdipour, E.
\newblock Big data analytics in weather forecasting: A systematic review.
\newblock \emph{Archives of Computational Methods in Engineering}, 29\penalty0 (2):\penalty0 1247--1275, 2022.

\bibitem[Feng et~al.(2019)Feng, Chen, Zhan, Fr{\"a}nzle, and Xue]{feng2019taming}
Feng, S., Chen, M., Zhan, N., Fr{\"a}nzle, M., and Xue, B.
\newblock Taming delays in dynamical systems: Unbounded verification of delay differential equations.
\newblock In \emph{International Conference on Computer Aided Verification}, pp.\  650--669. Springer, 2019.

\bibitem[Firoiu et~al.(2018)Firoiu, Ju, and Tenenbaum]{firoiu2018human}
Firoiu, V., Ju, T., and Tenenbaum, J.
\newblock At human speed: Deep reinforcement learning with action delay.
\newblock \emph{arXiv preprint arXiv:1810.07286}, 2018.

\bibitem[Fridman \& Shaked(2003)Fridman and Shaked]{fridman2003reachable}
Fridman, E. and Shaked, U.
\newblock On reachable sets for linear systems with delay and bounded peak inputs.
\newblock \emph{Automatica}, 39\penalty0 (11):\penalty0 2005--2010, 2003.

\bibitem[Gangwani et~al.(2020)Gangwani, Lehman, Liu, and Peng]{gangwani2020learning}
Gangwani, T., Lehman, J., Liu, Q., and Peng, J.
\newblock Learning belief representations for imitation learning in pomdps.
\newblock In \emph{uncertainty in artificial intelligence}, pp.\  1061--1071. PMLR, 2020.

\bibitem[Haarnoja et~al.(2017)Haarnoja, Tang, Abbeel, and Levine]{soft_q_learning}
Haarnoja, T., Tang, H., Abbeel, P., and Levine, S.
\newblock Reinforcement learning with deep energy-based policies.
\newblock In \emph{International conference on machine learning}, pp.\  1352--1361. PMLR, 2017.

\bibitem[Haarnoja et~al.(2018{\natexlab{a}})Haarnoja, Zhou, Abbeel, and Levine]{soft_actor_critic}
Haarnoja, T., Zhou, A., Abbeel, P., and Levine, S.
\newblock Soft actor-critic: Off-policy maximum entropy deep reinforcement learning with a stochastic actor.
\newblock In \emph{International conference on machine learning}, pp.\  1861--1870. PMLR, 2018{\natexlab{a}}.

\bibitem[Haarnoja et~al.(2018{\natexlab{b}})Haarnoja, Zhou, Hartikainen, Tucker, Ha, Tan, Kumar, Zhu, Gupta, Abbeel, et~al.]{soft_actor_critic_application}
Haarnoja, T., Zhou, A., Hartikainen, K., Tucker, G., Ha, S., Tan, J., Kumar, V., Zhu, H., Gupta, A., Abbeel, P., et~al.
\newblock Soft actor-critic algorithms and applications.
\newblock \emph{arXiv preprint arXiv:1812.05905}, 2018{\natexlab{b}}.

\bibitem[Han et~al.(2022)Han, Ren, Wu, Zhou, and Peng]{han2022off}
Han, B., Ren, Z., Wu, Z., Zhou, Y., and Peng, J.
\newblock Off-policy reinforcement learning with delayed rewards.
\newblock In \emph{International Conference on Machine Learning}, pp.\  8280--8303. PMLR, 2022.

\bibitem[Hasbrouck \& Saar(2013)Hasbrouck and Saar]{hasbrouck2013low}
Hasbrouck, J. and Saar, G.
\newblock Low-latency trading.
\newblock \emph{Journal of Financial Markets}, 16\penalty0 (4):\penalty0 646--679, 2013.

\bibitem[Hwangbo et~al.(2017)Hwangbo, Sa, Siegwart, and Hutter]{quadrotor_delay}
Hwangbo, J., Sa, I., Siegwart, R., and Hutter, M.
\newblock Control of a quadrotor with reinforcement learning.
\newblock \emph{IEEE Robotics and Automation Letters}, 2\penalty0 (4):\penalty0 2096--2103, 2017.

\bibitem[Kakade \& Langford(2002)Kakade and Langford]{optimal_approximate_rl}
Kakade, S. and Langford, J.
\newblock Approximately optimal approximate reinforcement learning.
\newblock In \emph{Proceedings of the Nineteenth International Conference on Machine Learning}, pp.\  267--274, 2002.

\bibitem[Katsikopoulos \& Engelbrecht(2003)Katsikopoulos and Engelbrecht]{delay_mdp}
Katsikopoulos, K.~V. and Engelbrecht, S.~E.
\newblock Markov decision processes with delays and asynchronous cost collection.
\newblock \emph{IEEE transactions on automatic control}, 48\penalty0 (4):\penalty0 568--574, 2003.

\bibitem[Kim et~al.(2023)Kim, Kim, Kang, Baek, and Han]{kim2023belief}
Kim, J., Kim, H., Kang, J., Baek, J., and Han, S.
\newblock Belief projection-based reinforcement learning for environments with delayed feedback.
\newblock In \emph{Thirty-seventh Conference on Neural Information Processing Systems}, 2023.

\bibitem[Kim \& Lee(2020)Kim and Lee]{kim2020automatic}
Kim, J.-G. and Lee, B.
\newblock Automatic p2p energy trading model based on reinforcement learning using long short-term delayed reward.
\newblock \emph{Energies}, 13\penalty0 (20):\penalty0 5359, 2020.

\bibitem[Kingma \& Ba(2014)Kingma and Ba]{kingma2014adam}
Kingma, D.~P. and Ba, J.
\newblock Adam: A method for stochastic optimization.
\newblock \emph{arXiv preprint arXiv:1412.6980}, 2014.

\bibitem[Kumar et~al.(2019)Kumar, Fu, Soh, Tucker, and Levine]{kumar2019stabilizing}
Kumar, A., Fu, J., Soh, M., Tucker, G., and Levine, S.
\newblock Stabilizing off-policy q-learning via bootstrapping error reduction.
\newblock \emph{Advances in Neural Information Processing Systems}, 32, 2019.

\bibitem[Liotet et~al.(2021)Liotet, Venneri, and Restelli]{learning_belief}
Liotet, P., Venneri, E., and Restelli, M.
\newblock Learning a belief representation for delayed reinforcement learning.
\newblock In \emph{2021 International Joint Conference on Neural Networks (IJCNN)}, pp.\  1--8. IEEE, 2021.

\bibitem[Liotet et~al.(2022)Liotet, Maran, Bisi, and Restelli]{dida}
Liotet, P., Maran, D., Bisi, L., and Restelli, M.
\newblock Delayed reinforcement learning by imitation.
\newblock In \emph{International Conference on Machine Learning}, pp.\  13528--13556. PMLR, 2022.

\bibitem[Mahmood et~al.(2018)Mahmood, Korenkevych, Komer, and Bergstra]{arm_delay}
Mahmood, A.~R., Korenkevych, D., Komer, B.~J., and Bergstra, J.
\newblock Setting up a reinforcement learning task with a real-world robot.
\newblock In \emph{2018 IEEE/RSJ International Conference on Intelligent Robots and Systems (IROS)}, pp.\  4635--4640. IEEE, 2018.

\bibitem[Mnih et~al.(2015)Mnih, Kavukcuoglu, Silver, Rusu, Veness, Bellemare, Graves, Riedmiller, Fidjeland, Ostrovski, et~al.]{deep_q_network}
Mnih, V., Kavukcuoglu, K., Silver, D., Rusu, A.~A., Veness, J., Bellemare, M.~G., Graves, A., Riedmiller, M., Fidjeland, A.~K., Ostrovski, G., et~al.
\newblock Human-level control through deep reinforcement learning.
\newblock \emph{nature}, 518\penalty0 (7540):\penalty0 529--533, 2015.

\bibitem[Myshkis(1955)]{myshkis1955lineare}
Myshkis, A.~D.
\newblock Lineare differentialgleichungen mit nacheilendem argument.
\newblock 1955.

\bibitem[Nath et~al.(2021)Nath, Baranwal, and Khadilkar]{revisiting_augment}
Nath, S., Baranwal, M., and Khadilkar, H.
\newblock Revisiting state augmentation methods for reinforcement learning with stochastic delays.
\newblock In \emph{Proceedings of the 30th ACM International Conference on Information \& Knowledge Management}, pp.\  1346--1355, 2021.

\bibitem[Rachelson \& Lagoudakis(2010)Rachelson and Lagoudakis]{rl_lipschitz_continuous}
Rachelson, E. and Lagoudakis, M.~G.
\newblock On the locality of action domination in sequential decision making.
\newblock 2010.

\bibitem[Schmidhuber(1991{\natexlab{a}})]{jurgenneural}
Schmidhuber, J.
\newblock neural sequence chunkers.
\newblock 1991{\natexlab{a}}.

\bibitem[Schmidhuber(1991{\natexlab{b}})]{schmidhuber1991learning}
Schmidhuber, J.
\newblock Learning to generate sub-goals for action sequences.
\newblock In \emph{Artificial neural networks}, pp.\  967--972, 1991{\natexlab{b}}.

\bibitem[Schuitema et~al.(2010)Schuitema, Bu{\c{s}}oniu, Babu{\v{s}}ka, and Jonker]{memoryless}
Schuitema, E., Bu{\c{s}}oniu, L., Babu{\v{s}}ka, R., and Jonker, P.
\newblock Control delay in reinforcement learning for real-time dynamic systems: A memoryless approach.
\newblock In \emph{2010 IEEE/RSJ International Conference on Intelligent Robots and Systems}, pp.\  3226--3231. IEEE, 2010.

\bibitem[Silver et~al.(2018)Silver, Hubert, Schrittwieser, Antonoglou, Lai, Guez, Lanctot, Sifre, Kumaran, Graepel, et~al.]{silver2018general}
Silver, D., Hubert, T., Schrittwieser, J., Antonoglou, I., Lai, M., Guez, A., Lanctot, M., Sifre, L., Kumaran, D., Graepel, T., et~al.
\newblock A general reinforcement learning algorithm that masters chess, shogi, and go through self-play.
\newblock \emph{Science}, 362\penalty0 (6419):\penalty0 1140--1144, 2018.

\bibitem[Sutton(1984)]{sutton1984temporal}
Sutton, R.~S.
\newblock \emph{Temporal credit assignment in reinforcement learning}.
\newblock University of Massachusetts Amherst, 1984.

\bibitem[Sutton(1995)]{sutton1995generalization}
Sutton, R.~S.
\newblock Generalization in reinforcement learning: Successful examples using sparse coarse coding.
\newblock \emph{Advances in neural information processing systems}, 8, 1995.

\bibitem[Sutton \& Barto(2018)Sutton and Barto]{rlai}
Sutton, R.~S. and Barto, A.~G.
\newblock \emph{Reinforcement learning: An introduction}.
\newblock MIT press, 2018.

\bibitem[Tesauro(1994)]{tesauro1994td}
Tesauro, G.
\newblock Td-gammon, a self-teaching backgammon program, achieves master-level play.
\newblock \emph{Neural computation}, 6\penalty0 (2):\penalty0 215--219, 1994.

\bibitem[Todorov et~al.(2012)Todorov, Erez, and Tassa]{mujoco}
Todorov, E., Erez, T., and Tassa, Y.
\newblock Mujoco: A physics engine for model-based control.
\newblock In \emph{2012 IEEE/RSJ international conference on intelligent robots and systems}, pp.\  5026--5033. IEEE, 2012.

\bibitem[Villani et~al.(2009)]{villani2009optimal}
Villani, C. et~al.
\newblock \emph{Optimal transport: old and new}, volume 338.
\newblock Springer, 2009.

\bibitem[Walsh et~al.(2009)Walsh, Nouri, Li, and Littman]{learning_planning_delayed_feedback}
Walsh, T.~J., Nouri, A., Li, L., and Littman, M.~L.
\newblock Learning and planning in environments with delayed feedback.
\newblock \emph{Autonomous Agents and Multi-Agent Systems}, 18:\penalty0 83--105, 2009.

\bibitem[Wang et~al.(2023{\natexlab{a}})Wang, Zhan, Wang, Huang, Wang, Yang, and Zhu]{wang2023joint}
Wang, Y., Zhan, S., Wang, Z., Huang, C., Wang, Z., Yang, Z., and Zhu, Q.
\newblock Joint differentiable optimization and verification for certified reinforcement learning.
\newblock In \emph{Proceedings of the ACM/IEEE 14th International Conference on Cyber-Physical Systems (with CPS-IoT Week 2023)}, pp.\  132--141, 2023{\natexlab{a}}.

\bibitem[Wang et~al.(2023{\natexlab{b}})Wang, Zhan, Jiao, Wang, Jin, Yang, Wang, Huang, and Zhu]{wang2023enforcing}
Wang, Y., Zhan, S.~S., Jiao, R., Wang, Z., Jin, W., Yang, Z., Wang, Z., Huang, C., and Zhu, Q.
\newblock Enforcing hard constraints with soft barriers: Safe reinforcement learning in unknown stochastic environments.
\newblock In \emph{International Conference on Machine Learning}, pp.\  36593--36604. PMLR, 2023{\natexlab{b}}.

\bibitem[Wang et~al.(2024)Wang, Strupl, Faccio, Wu, Liu, Grudzie{\'n}, Tan, and Schmidhuber]{wang2024highway}
Wang, Y., Strupl, M., Faccio, F., Wu, Q., Liu, H., Grudzie{\'n}, M., Tan, X., and Schmidhuber, J.
\newblock Highway reinforcement learning.
\newblock \emph{arXiv preprint arXiv:2405.18289}, 2024.

\bibitem[Xu et~al.(2021)Xu, Fu, Wang, O'Neill, and Zhu]{xu2021learning}
Xu, S., Fu, Y., Wang, Y., O'Neill, Z., and Zhu, Q.
\newblock Learning-based framework for sensor fault-tolerant building hvac control with model-assisted learning.
\newblock In \emph{Proceedings of the 8th ACM international conference on systems for energy-efficient buildings, cities, and transportation}, pp.\  1--10, 2021.

\bibitem[Xu et~al.(2022)Xu, Fu, Wang, Yang, O'Neill, Wang, and Zhu]{xu2022accelerate}
Xu, S., Fu, Y., Wang, Y., Yang, Z., O'Neill, Z., Wang, Z., and Zhu, Q.
\newblock Accelerate online reinforcement learning for building hvac control with heterogeneous expert guidances.
\newblock In \emph{Proceedings of the 9th ACM International Conference on Systems for Energy-Efficient Buildings, Cities, and Transportation}, BuildSys '22, pp.\  89–98, 2022.

\bibitem[Xue et~al.(2020)Xue, Wang, Feng, and Zhan]{xue2020over}
Xue, B., Wang, Q., Feng, S., and Zhan, N.
\newblock Over-and underapproximating reach sets for perturbed delay differential equations.
\newblock \emph{IEEE Transactions on Automatic Control}, 66\penalty0 (1):\penalty0 283--290, 2020.

\bibitem[Xue et~al.(2021)Xue, Bai, Zhan, Liu, and Jiao]{xue2021reach}
Xue, B., Bai, Y., Zhan, N., Liu, W., and Jiao, L.
\newblock Reach-avoid analysis for delay differential equations.
\newblock In \emph{2021 60th IEEE Conference on Decision and Control (CDC)}, pp.\  1301--1307. IEEE, 2021.

\bibitem[Zhan et~al.(2024)Zhan, Wang, Wu, Jiao, Huang, and Zhu]{zhan2024state}
Zhan, S.~S., Wang, Y., Wu, Q., Jiao, R., Huang, C., and Zhu, Q.
\newblock State-wise safe reinforcement learning with pixel observations.
\newblock \emph{6th Annual Learning for Dynamics and Control Conference}, 2024.

\end{thebibliography}
\bibliographystyle{icml2024}

\newpage
\appendix
\onecolumn

\section{Implementation Detail}
\label{boad_implementation_detail}
The hyper-parameters setting used in this work is provided in Table \ref{hyperparameters_setting} for Acrobot and MuJoCo benchmarks. We provide the pseudo-code of AD-DQN and AD-SAC in Algorithm \ref{appendix_boad_q_learning_alg} and Algorithm \ref{appendix_boad_sac_alg}, respectively, along with the description of how to practically implement them in detail.
The code for reproducing our results can be found in the supplementary material.

\subsection{Hyper-parameters Setting}
\begin{table}[h]
    \centering
    \begin{tabular}{|c|c|c|}
    \hline
         Hyper-parameter & Setting(Acrobot) & Setting(MuJoCo)\\
    \hline
         buffer size & 1,000 & 1,000,000\\
         batch size     & 128 & 256\\
         total timesteps& 500,000 & 1,000,000\\
         discount factor& 0.99 & 0.99\\
         learning rate  & 2.5e-4 & 3e-4(actor), 1e-3(critic)\\
         network layers  & 3 & 3\\
         network neurons  & [128, 64] & [256, 256]\\
         activation  & ReLU & ReLU\\
         optimizer  & Adam~\cite{kingma2014adam} & Adam\\
         initial $\epsilon$ for $\epsilon$-greedy & 1.0 & -\\
         final $\epsilon$ for $\epsilon$-greedy & 0.05 & -\\
         initial entropy $\alpha$ for SAC & - & 0.2 \\
         learning rate for entropy $\alpha$ & - & 1e-3 \\
         train frequency & 5 & 2(actor), 1(critic)\\
         target network update frequency & 500 & -\\
         target network soft update factor & - & 5e-3 \\
         n for N-steps & - & 3 \\
         auxiliary delays $\Delta^{\tau}$ for AD-RL & 0 & 0\\
         
    \hline
    \end{tabular}
    \caption{Hyper-parameters Setting on Acrobot and MuJoCo benchmarks.}
    \label{hyperparameters_setting}
\end{table}

\subsection{Discrete Control: AD-DQN}
In practice, we will maintain two Q-networks (\textcolor{lightblue}{$Q_{\psi}$}, \textcolor{lightred}{$Q^{\tau}_{\theta}$}) at the same time, and each of them corresponds to different delays (\textcolor{lightblue}{$\Delta$}, \textcolor{lightred}{$\Delta^{\tau}$}).
When the agent needs to select an action, it will first enquire about two Q-networks and get the two best actions in different views of delays, separately. Then evaluate these two actions based on the auxiliary Q-network \textcolor{lightred}{$Q^{\tau}_{\theta}$}.
Here, we experimentally found that the \textit{argmin} operator, taking a more conservative action is more stable than \textit{argmax} operator.
For the auxiliary Q-network \textcolor{lightred}{$Q^{\tau}_{\theta}$}, its update rule is the same as the original DQN.
And for the Q-network \textcolor{lightblue}{$Q^{\tau}_{\psi}$}, we update it in our way: bootstrapping on \textcolor{lightred}{$Q^{\tau}_{\theta}$}. To stabilize the training process, we also adopt the target networks to estimate the td-targets~\cite{deep_q_network}.

\begin{algorithm}[h]
   \caption{Auxiliary-Delayed Deep Q-Network (AD-DQN)}
   \label{appendix_boad_q_learning_alg}
   \begin{algorithmic}
       \STATE {\bfseries Input:} Q-network \textcolor{lightblue}{$Q_{\psi}$} and target network \textcolor{lightblue}{$\hat{Q}_{\hat{\psi}}$} for delays \textcolor{lightblue}{$\Delta$}; 
       Q-network \textcolor{lightred}{$Q^{\tau}_{\theta}$} and target network \textcolor{lightred}{$\hat{Q}^{\tau}_{\hat{\theta}}$} for auxiliary delays \textcolor{lightred}{$\Delta^{\tau}$}; 
       discount factor $\gamma$;
       empty replay buffer $\mathcal{D}$;
       \FOR{each episode}
           \STATE Initial state buffer $(s_1, \cdots, s_{\Delta^{\tau}}, \cdots, s_{\Delta})$ and action buffer $(a_1, \cdots, a_{\Delta^{\tau}}, \cdots, a_{\Delta})$
           \FOR{each environment step $t$}
                \STATE Observe augment states $\textcolor{lightblue}{x_t}=(s_1, a_1, \cdots, a_{\Delta})$ and $\textcolor{lightred}{x^{\tau}_t}=(s_{\Delta^{\tau}}, a_{\Delta^{\tau}}, \cdots, a_{\Delta})$
                \STATE Sample and take action $a_t = \mathop{\arg\min}\limits_{\hat{a} \in \{\arg\max_a\textcolor{lightblue}{Q_\psi}(\textcolor{lightblue}{x_t}, a), \arg\max_a\textcolor{lightred}{Q^{\tau}_{\theta}}(\textcolor{lightred}{x^{\tau}_t}, a)\}}\textcolor{lightred}{Q^{\tau}_{\theta}}(\textcolor{lightred}{x^{\tau}_t}, \hat{a})$
                \STATE Observe reward $r_t$ and next state $s_{t+1}$
                \STATE Update state buffers: $(s_1, \cdots, s_{\Delta})\leftarrow (s_2, \cdots, s_{\Delta}, s_t)$ and action buffers: $(a_1, \cdots, a_{\Delta})\leftarrow (a_2, \cdots, a_{\Delta}, a_t)$
                \STATE Update replay buffer: $\mathcal{D} \leftarrow \mathcal{D} \cup \{\textcolor{lightblue}{x_t}, \textcolor{lightred}{x^{\tau}_t}, a_t, r_t, \textcolor{lightblue}{x_{t+1}}, \textcolor{lightred}{x^{\tau}_{t+1}}\}$
                \FOR{each batch $\{\textcolor{lightblue}{x_t}, \textcolor{lightred}{x^{\tau}_t}, a_t, r_t, \textcolor{lightblue}{x_{t+1}}, \textcolor{lightred}{x^{\tau}_{t+1}}\} \sim \mathcal{D}$}
                \IF{$\text{Uniform}(0, 1) > 0.5$}
                \STATE Update \textcolor{lightred}{$Q^{\tau}_{\theta}$} via applying gradient descent\\
                $$
                    \triangledown_{\textcolor{lightred}{\theta}}
                    \left(\textcolor{lightred}{Q^{\tau}_{\theta}}(\textcolor{lightred}{x^{\tau}_t}, a_t)
                    - \left[r_t + \gamma \max_{a_{t+1}}\textcolor{lightred}{\hat{Q}^{\tau}_{\hat{\theta}}}(\textcolor{lightred}{x^{\tau}_{t+1}}, a_{t+1})\right]\right)
                $$
                \ELSE
                \STATE Update \textcolor{lightblue}{$Q_{\psi}$} via applying gradient descent\\
                $$
                \triangledown_{\textcolor{lightblue}{\psi}}
                    \left(\textcolor{lightblue}{Q_{\psi}}(\textcolor{lightblue}{x_t}, a_t)
                    -
                    \left[r_t + \gamma \textcolor{lightred}{\hat{Q}^{\tau}_{\hat{\theta}}}(\textcolor{lightred}{x^{\tau}_{t+1}}, {\arg\max}_{a_{t+1}}\textcolor{lightblue}{\hat{Q}_{\hat{\psi}}}(\textcolor{lightblue}{x_{t+1}}, a_{t+1}))\right]\right)
                $$
                \ENDIF
                \ENDFOR
                \STATE Update target networks weights \textcolor{lightblue}{$\hat{Q}_{\hat{\psi}}$} and \textcolor{lightred}{$\hat{Q}^{\tau}_{\hat{\theta}}$} via copying from \textcolor{lightblue}{${Q}_{{\psi}}$} and \textcolor{lightred}{${Q}^{\tau}_{{\theta}}$}, respectively\\
            \ENDFOR
       \ENDFOR
       \STATE {\bfseries Output:} Q-network \textcolor{lightblue}{$Q_\psi$}
    \end{algorithmic}
\end{algorithm}

\subsection{Continuous Control: AD-SAC}
Applying our method to soft actor-critic, we adopt some advanced modifications~\cite{soft_actor_critic_application}, including removing the unnecessary value network, adjusting entropy automatically.
We maintain two policies \textcolor{lightblue}{$\pi_\psi$} and \textcolor{lightred}{$\pi^{\tau}_\phi$} for delays \textcolor{lightblue}{$\Delta$} and \textcolor{lightred}{$\Delta^{\tau}$} respectively.
To stabilize off-policy training and reduce bootstrapping error~\cite{kumar2019stabilizing} in AD-SAC, we maintain two Q-functions (\textcolor{lightred}{$Q^{\tau}_{\theta_1}, Q^{\tau}_{\theta_2}$}) for evaluating \textcolor{lightblue}{$\pi_\psi$} and \textcolor{lightred}{$\pi^{\tau}_\phi$} respectively.
Similar to the AD-DQN, selecting the conservative action from policies \textcolor{lightblue}{$\pi_\psi$} and \textcolor{lightred}{$\pi^{\tau}_\phi$} can stabilize the learning process.

\begin{algorithm}[h]
   \caption{Auxiliary Delay Soft Actor-Critic (AD-SAC)}
   \label{appendix_boad_sac_alg}
   \begin{algorithmic}
        \STATE {\bfseries Input:} actor \textcolor{lightblue}{$\pi_\psi$} for delay \textcolor{lightblue}{$\Delta$};
        actor \textcolor{lightred}{$\pi^{\tau}_\phi$}, critics \textcolor{lightred}{$Q^{\tau}_{\theta_1}, Q^{\tau}_{\theta_2}$} and target critics \textcolor{lightred}{$\hat{Q^{\tau}}_{\hat{\theta_1}}, \hat{Q^{\tau}}_{\hat{\theta_2}}$} for auxiliary delays \textcolor{lightred}{$\Delta^{\tau}$}; n-step $n$;
        replay buffer $\mathcal{D}$;
        \FOR{each episode}
            \STATE Initial state buffer $(s_1, \cdots, s_{\Delta^{\tau}}, \cdots, s_{\Delta})$ and action buffer $(a_1, \cdots, a_{\Delta^{\tau}}, \cdots, a_{\Delta})$
            \FOR{each environment step $t$}
                \STATE Observe augment states $\textcolor{lightblue}{x_t}=(s_1, a_1, \cdots, a_{\Delta})$ and $\textcolor{lightred}{x^{\tau}_t}=(s_{\Delta^{\tau}}, a_{\Delta^{\tau}}, \cdots, a_{\Delta})$
                \STATE Sample and take action $a_t = \mathop{\arg\min}\limits_{a \in \{\textcolor{lightblue}{\pi_\psi(x_t)}, \textcolor{lightred}{\pi^{\tau}_\phi(x^{\tau}_t)}\}}\textcolor{lightred}{Q^{\tau}_{\theta_i}}(\textcolor{lightred}{x^{\tau}_t}, a)$
                \STATE Observe reward $r_t$ and next state $s_{t+1}$
                \STATE Update state buffer $(s_1, \cdots, s_{\Delta})\leftarrow (s_2, \cdots, s_{\Delta}, s_t)$ and action buffer $(a_1, \cdots, a_{\Delta})\leftarrow (a_2, \cdots, a_{\Delta}, a_t)$\\
                \STATE Update replay buffer: $\mathcal{D} \leftarrow \mathcal{D} \cup \{\textcolor{lightblue}{x_t}, \textcolor{lightred}{x^{\tau}_t}, a_t, r_t, \textcolor{lightblue}{x_{t+1}}, \textcolor{lightred}{x^{\tau}_{t+1}}\}$
            \ENDFOR
        \ENDFOR
        \FOR{each batch $\{\textcolor{lightblue}{x_t}, \textcolor{lightred}{x^{\tau}_t}, a_t, r_t:r_{t+n-1}, \textcolor{lightblue}{x_{t+n}}, \textcolor{lightred}{x^{\tau}_{t+n}}\} \sim \mathcal{D}$}
            \STATE Compute TD Target 
            $$
            \mathbb{Y} =
            \mathop{\mathbb{E}}_{\textcolor{lightred}{\hat{a}\sim\pi^{\tau}_\phi(\cdot|x^{\tau}_{t+n})}\atop
            \textcolor{lightblue}{\hat{a}\sim\pi_\psi(\cdot|x_{t+n})}
            }\left[
                \sum_{i=0}^{n-1}\left[\gamma^i r_{t+i}\right] + \gamma^n \min \left(
                \textcolor{lightred}{Q^{\tau}_{\theta_1}(x^{\tau}_{t+n}, \hat{a})}-\log \textcolor{lightred}{\pi^{\tau}_\phi(\hat{a}|x^{\tau}_{t+n})},
                \textcolor{lightred}{Q^{\tau}_{\theta_2}(x_{t+n}, \textcolor{lightblue}{\hat{a}})}-\log \textcolor{lightblue}{\pi_\phi(\hat{a}|x_{t+1})},
                \right)
            \right]
            $$
            \STATE Update \textcolor{lightred}{$Q^{\tau}_{\theta_i}(i=1,2)$} via applying gradient descent\\
            $$    
            \begin{aligned}
                \triangledown_{\textcolor{lightred}{\theta_i}} \left[\textcolor{lightred}{Q^{\tau}_{\theta_i}(x^{\tau}_t, \textcolor{black}{a_t})}
                -
                \mathbb{Y}
                \right] \\
            \end{aligned}
            $$
            \IF{$\text{Uniform}(0, 1) > 0.5$}
            \STATE Update \textcolor{lightred}{$\pi^{\tau}_\phi$} via applying gradient descent\\
            $$\triangledown_{\textcolor{lightred}{\phi}} \mathop{\mathbb{E}}_{\textcolor{lightred}{\hat{a}\sim\pi^{\tau}_\phi(\cdot|x^{\tau}_t)}}\left[\log \textcolor{lightred}{\pi^{\tau}_\phi(\hat{a}|x^{\tau}_t)}-\min_{i=1,2}\textcolor{lightred}{Q^{\tau}_{\theta_i}(x^{\tau}_t, \hat{a})}\right]$$
            \ELSE
            \STATE Update \textcolor{lightblue}{$\pi_\psi$} via applying gradient descent\\
            $$\triangledown_{\textcolor{lightblue}{\psi}} \mathop{\mathbb{E}}_{\textcolor{lightblue}{\hat{a}\sim\pi_\psi(\cdot|x_t)}}\left[\log \textcolor{lightblue}{\pi_\psi(\hat{a}|x_t)}-\min_{i=1,2}\textcolor{lightred}{Q^{\tau}_{\theta_i}}(\textcolor{lightred}{x^{\tau}_t}, \textcolor{lightblue}{\hat{a}})\right]$$            
            \ENDIF
            \STATE Soft update target critics weights \textcolor{lightred}{$Q^{\tau}_{\theta_1}, Q^{\tau}_{\theta_2}$} via copying from \textcolor{lightred}{$\hat{Q^{\tau}}_{\hat{\theta_1}}, \hat{Q^{\tau}}_{\hat{\theta_2}}$}, respectively\\
        \ENDFOR
       \STATE {\bfseries Output:} actor \textcolor{lightblue}{$\pi_\psi$}
    \end{algorithmic}
\end{algorithm}

\clearpage
\section{Theoretical Analysis}

\subsection{Performance Difference}
\label{appendix_vf_bound}

\begin{proposition}[Lipschitz Continuous Q-value function Bound~\cite{dida}]
    \label{relation_q_pi}
    Consider a $L_Q$-LC Q-function $Q^\pi$ of the $L_\pi$-LC policy $\pi$ in the $(L_P, L_R)$-LC MDP, it satisfies that $\forall x\in\mathcal{X}$,
    $$
        \left|
        \mathop{\mathbb{E}}_{
            a_1 \sim \mu \atop
            a_2 \sim \upsilon
        }\left[
        Q^\pi(x, a_1) - Q^\pi(x, a_2)
        \right]
        \right|
        \leq
        L_Q W_1\left(\mu||\upsilon\right)
    $$
    where $\mu, \upsilon$ are two arbitrary distributions over $\mathcal{X}$.
\end{proposition}

\begin{lemma}[General Delayed Performance Difference]
    \label{appendix_general_delayed_performance_diff}
    For policies \textcolor{lightred}{$\pi^{\tau}$} and \textcolor{lightblue}{$\pi$}, with delays \textcolor{lightred}{$\Delta^{\tau}$} $<$ \textcolor{lightblue}{$\Delta$}. Given any \textcolor{lightblue}{$x_t$} $\in$ \textcolor{lightblue}{$\mathcal{X}$}, the performance difference is denoted as $I$(\textcolor{lightblue}{$x_t$})
    $$
    \begin{aligned}
        I(\textcolor{lightblue}{x_t}) &= \mathop{\mathbb{E}}_{\textcolor{lightred}{x^{\tau}_t}\sim b_\Delta(\cdot|\textcolor{lightblue}{x_t})}\left[\textcolor{lightred}{V^{\tau}(x^{\tau}_t)}\right] - \textcolor{lightblue}{V(x_t)}\\
        &= \frac{1}{1-\gamma} \mathop{\mathbb{E}}_{\substack{
        \textcolor{lightred}{\hat{x}^{\tau}} \sim b_\Delta(\cdot|\textcolor{lightblue}{\hat{x}})\\
        \textcolor{lightblue}{a\sim\pi(\cdot|\hat{x})}\\
        \textcolor{lightblue}{\hat{x} \sim d_{x_t}^{\pi}}
        }}
        \left[\textcolor{lightred}{V^{\tau}(\hat{x}^{\tau})} - \textcolor{lightred}{Q^{\tau}}(\textcolor{lightred}{\hat{x}^{\tau}}, \textcolor{lightblue}{a})\right]\\
    \end{aligned}
    $$

\end{lemma}
\begin{proof}
    $$
    \begin{aligned}
        & \underbrace{
            \mathop{\mathbb{E}}_{
                \textcolor{lightred}{x^{\tau}_t}\sim b_\Delta(\cdot|\textcolor{lightblue}{x_t})
            }\left[
                \textcolor{lightred}{V^{\tau}(x^{\tau}_t)}
            \right] 
            - \textcolor{lightblue}{V(x_t)}
        }_{I(\textcolor{lightblue}{x_t}) }\\
        =& \underbrace{\mathop{\mathbb{E}}_{\textcolor{lightred}{x^{\tau}_t}\sim b_\Delta(\cdot|\textcolor{lightblue}{x_t})}\left[\textcolor{lightred}{V^{\tau}(x^{\tau}_t)}\right] 
        - \mathop{\mathbb{E}}_{
            \textcolor{lightred}{x^{\tau}_t}\sim b_\Delta(\cdot|\textcolor{lightblue}{x_t})\atop
            \textcolor{lightblue}{a_t \sim \pi(\cdot|x_t)}
        }
        \left[
            \textcolor{lightred}{\mathcal{R}_\tau(x^{\tau}_t}, \textcolor{lightblue}{a_t}) + \gamma \mathop{\mathbb{E}}_{\textcolor{lightred}{x^{\tau}_{t+1}\sim\mathcal{P}_{\Delta^{\tau}}(\cdot|x^{\tau}_{t}}, \textcolor{lightblue}{a_t})}\left[\textcolor{lightred}{V^{\tau}(x^{\tau}_{t+1})}\right]
        \right]}_A
        \\
        & + \underbrace{\mathop{\mathbb{E}}_{
            \textcolor{lightred}{x^{\tau}_t}\sim b_\Delta(\cdot|\textcolor{lightblue}{x_t})\atop
            \textcolor{lightblue}{a_t \sim \pi(\cdot|x_t)}
        }
        \left[
            \textcolor{lightred}{\mathcal{R}_\tau}(\textcolor{lightred}{x^{\tau}_t}, \textcolor{lightblue}{a_t}) + \gamma \mathop{\mathbb{E}}_{\textcolor{lightred}{x^{\tau}_{t+1}\sim\mathcal{P}_{\Delta^{\tau}}(\cdot|x^{\tau}_{t}}, \textcolor{lightblue}{a_t})}\left[\textcolor{lightred}{V^{\tau}(x^{\tau}_{t+1})}\right]
        \right]
        - \textcolor{lightblue}{V(x_t)}}_B\\
    \end{aligned}
    $$
    For $A$, we have
    $$
        \begin{aligned}
            A = \mathop{\mathbb{E}}_{\textcolor{lightred}{x^{\tau}_t}\sim b_\Delta(\cdot|\textcolor{lightblue}{x_t})}\left[\textcolor{lightred}{V^{\tau}(x^{\tau}_t)}\right] 
            - \mathop{\mathbb{E}}_{
                \textcolor{lightred}{x^{\tau}_t}\sim b_\Delta(\cdot|\textcolor{lightblue}{x_t})\atop
                \textcolor{lightblue}{a_t \sim \pi(\cdot|x_t)}
            }
            \left[
                \textcolor{lightred}{Q^{\tau}(x^{\tau}_t}, \textcolor{lightblue}{a_t})
            \right]
        \end{aligned}
    $$
    And for $B$, note that $
    \textcolor{lightblue}{V(x_t)} 
    = \mathop{\mathbb{E}}_{
        \textcolor{lightred}{x^{\tau}_t} \sim b_\Delta(\cdot|\textcolor{lightblue}{x_t}) \atop
        \textcolor{lightblue}{a_t \sim \pi(\cdot|x_t)}
    }
    [
    \textcolor{lightred}{\mathcal{R}_\tau}(\textcolor{lightred}{x^{\tau}_t}, \textcolor{lightblue}{a_t})
    ]
    +
    \gamma \mathop{\mathbb{E}}_{
    \textcolor{lightblue}{x_{t+1}\sim\mathcal{P}_{\Delta}(\cdot|x_{t}}, \textcolor{lightblue}{a_t})\atop
    \textcolor{lightblue}{a_t \sim \pi(\cdot|x_t)}
    }
    [\textcolor{lightblue}{V(x_{t+1})}]
    $, then we have
    $$
        \begin{aligned}
            B =& \mathop{\mathbb{E}}_{
                \textcolor{lightred}{x^{\tau}_t}\sim b_\Delta(\cdot|\textcolor{lightblue}{x_t})\atop
                \textcolor{lightblue}{a_t \sim \pi(\cdot|x_t)}
            }
            \left[
                \textcolor{lightred}{\mathcal{R}_\tau}(\textcolor{lightred}{x^{\tau}_t}, \textcolor{lightblue}{a_t}) + \gamma \mathop{\mathbb{E}}_{\textcolor{lightred}{x^{\tau}_{t+1}\sim\mathcal{P}_{\Delta^{\tau}}(\cdot|x^{\tau}_{t}}, \textcolor{lightblue}{a_t})}\left[\textcolor{lightred}{V^{\tau}(x^{\tau}_{t+1})}\right]
            \right]\\
            &- 
            \mathop{\mathbb{E}}_{
                \textcolor{lightred}{x^{\tau}_t \sim b_\Delta(\cdot|x_t)}\atop
                \textcolor{lightblue}{a_t \sim \pi(\cdot|x_t)}
            }
            \left[
                \textcolor{lightred}{\mathcal{R}_\tau}(\textcolor{lightred}{x^{\tau}_t}, \textcolor{lightblue}{a_t})
            \right]
            -
            \gamma \mathop{\mathbb{E}}_{
            \textcolor{lightblue}{x_{t+1}\sim\mathcal{P}_{\Delta}(\cdot|x_{t}}, \textcolor{lightblue}{a_t})\atop
            \textcolor{lightblue}{a_t \sim \pi(\cdot|x_t)}
            }
            \left[\textcolor{lightblue}{V(x_{t+1})}\right]\\
            = & \mathop{\mathbb{E}}_{
                \textcolor{lightred}{x^{\tau}_t}\sim b_\Delta(\cdot|\textcolor{lightblue}{x_t})\atop
                \textcolor{lightblue}{a_t \sim \pi(\cdot|x_t)}
            }
            \left[
                \gamma \mathop{\mathbb{E}}_{\textcolor{lightred}{x^{\tau}_{t+1}\sim\mathcal{P}_{\Delta^{\tau}}(\cdot|x^{\tau}_{t}}, \textcolor{lightblue}{a_t})}\left[\textcolor{lightred}{V^{\tau}(x^{\tau}_{t+1})}\right]
            \right]
            - 
            \mathop{\mathbb{E}}_{
                \textcolor{lightblue}{a_t \sim \pi(\cdot|x_t)}
            }
            \left[
                \gamma \mathop{\mathbb{E}}_{\textcolor{lightblue}{x_{t+1}\sim\mathcal{P}_{\Delta}(\cdot|x_{t}}, \textcolor{lightblue}{a_t})}\left[\textcolor{lightblue}{V(x_{t+1})}\right]
            \right]\\
            = & \gamma \mathop{\mathbb{E}}_{
            \textcolor{lightblue}{x_{t+1}\sim\mathcal{P}_{\Delta}(\cdot|x_{t}}, \textcolor{lightblue}{a_t})\atop
            \textcolor{lightblue}{a_t \sim \pi(\cdot|x_t)}
            }\underbrace{\left[\mathop{\mathbb{E}}_{\textcolor{lightred}{x^{\tau}_{t+1}}\sim b_\Delta(\cdot|\textcolor{lightblue}{x_{t+1}})}\left[\textcolor{lightred}{V^{\tau}(x^{\tau}_{t+1})}\right] - \textcolor{lightblue}{V(x_{t+1})}\right]}_{I(\textcolor{lightblue}{x_{t+1}})}\\
        \end{aligned}
    $$

    The last step can be derived due to the fact that 
    $$
    \mathop{\mathbb{E}}_{\substack{
                \textcolor{lightred}{{x^{\tau}_{t+1}\sim\mathcal{P}_{\Delta^{\tau}}(\cdot|x^{\tau}_{t}}, \textcolor{lightblue}{a_t})}\\
                \textcolor{lightred}{x^{\tau}_t}\sim b_\Delta(\cdot|\textcolor{lightblue}{x_t})\\
                \textcolor{lightblue}{a_t \sim \pi(\cdot|x_t)}
            }}
            \left[
                \textcolor{lightred}{{x^{\tau}_{t+1}}}
            \right]
    =
    \mathop{\mathbb{E}}_{\substack{
            \textcolor{lightred}{x^{\tau}_{t+1}}\sim b_\Delta(\cdot|\textcolor{lightblue}{x_{t+1}})\\
            \textcolor{lightblue}{x_{t+1}\sim\mathcal{P}_{\Delta}(\cdot|x_{t}}, \textcolor{lightblue}{a_t})\\
            \textcolor{lightblue}{a_t \sim \pi(\cdot|x_t)}
            }}
            \left[
                \textcolor{lightred}{{x^{\tau}_{t+1}}}
            \right]
    $$

    Based on the above iterative equation, we have
    $$
        \begin{aligned}
            I(\textcolor{lightblue}{x_t}) &=
            \mathop{\mathbb{E}}_{\textcolor{lightred}{x^{\tau}_t}\sim b_\Delta(\cdot|\textcolor{lightblue}{x_t})}\left[\textcolor{lightred}{V^{\tau}(x^{\tau}_t)}\right] 
            - \mathop{\mathbb{E}}_{
                \textcolor{lightred}{x^{\tau}_t}\sim b_\Delta(\cdot|\textcolor{lightblue}{x_t})\atop
                \textcolor{lightblue}{a_t \sim \pi(\cdot|x_t)}
            }
            \left[
                \textcolor{lightred}{Q^{\tau}(x^{\tau}_t}, \textcolor{lightblue}{a_t})
            \right]
            + \gamma \mathop{\mathbb{E}}_{
            \textcolor{lightblue}{x_{t+1}\sim\mathcal{P}_{\Delta}(\cdot|x_{t}}, \textcolor{lightblue}{a_t})\atop
            \textcolor{lightblue}{a_t \sim \pi(\cdot|x_t)}
            }\left[{I(\textcolor{lightblue}{x_{t+1}})}\right]\\
            &= \sum_{i=0}^\infty\gamma^i
            \mathop{\mathbb{E}}_{
            \textcolor{lightblue}{x_{t+i}\sim\mathcal{P}_{\Delta}(\cdot|x_{t+i-1}}, \textcolor{lightblue}{a_{t+i-1}})\atop
            \textcolor{lightblue}{a_{t+i-1} \sim \pi(\cdot|x_{t+i-1})}}
            \left[\mathop{\mathbb{E}}_{\textcolor{lightred}{x^{\tau}_{t+i}}\sim b_\Delta(\cdot|\textcolor{lightblue}{x_{t+i}})}\left[\textcolor{lightred}{V^{\tau}(x^{\tau}_{t+i})}\right]
            - \mathop{\mathbb{E}}_{
                \textcolor{lightred}{x^{\tau}_{t+i}}\sim b_\Delta(\cdot|\textcolor{lightblue}{x_{t+i}})\atop
                \textcolor{lightblue}{a_{t+i} \sim \pi(\cdot|x_{t+i})}
            }
            \left[
                \textcolor{lightred}{Q^{\tau}(x^{\tau}_{t+i}}, \textcolor{lightblue}{a_{t+i}})
            \right]\right]\\
            &= \frac{1}{1-\gamma} \mathop{\mathbb{E}}_{\substack{
            \textcolor{lightred}{\hat{x}^{\tau}} \sim b_\Delta(\cdot|\textcolor{lightblue}{\hat{x}})\\
            \textcolor{lightblue}{a\sim\pi(\cdot|\hat{x})}\\
            \textcolor{lightblue}{\hat{x} \sim d_{x_t}^{\pi}}
            }}
            \left[\textcolor{lightred}{V^{\tau}(\hat{x}^{\tau})} - \textcolor{lightred}{Q^{\tau}}(\textcolor{lightred}{\hat{x}^{\tau}}, \textcolor{lightblue}{a})\right]\\
        \end{aligned}
    $$
\end{proof}

\begin{theorem}[Delayed Performance Difference Bound]
    \label{appendix_delayed_performance_difference_bound}
    For policies \textcolor{lightred}{$\pi^{\tau}$} and \textcolor{lightblue}{$\pi$}, with \textcolor{lightred}{$\Delta^{\tau}$} $<$ \textcolor{lightblue}{$\Delta$}. Given any \textcolor{lightblue}{$x_t$} $\in$ \textcolor{lightblue}{$\mathcal{X}$}, if \textcolor{lightred}{$Q^{\tau}$} is $L_Q$-LC, the performance difference between policies can be bounded as follow
    $$
    \begin{aligned}
        \mathop{\mathbb{E}}_{
        \textcolor{lightred}{x^{\tau}_t}\sim b_\Delta(\cdot|\textcolor{lightblue}{x_t})\atop
        \textcolor{lightblue}{a_t\sim\pi(\cdot|x_t)}
        }\left[
        \textcolor{lightred}{V^{\tau}(x^{\tau}_t)}
        -
        \textcolor{lightred}{Q^{\tau}}(\textcolor{lightred}{x^{\tau}_t}, \textcolor{lightblue}{a_t})
        \right]
        \leq 
        L_Q \mathop{\mathbb{E}}_{\textcolor{lightred}{x^{\tau}_t}\sim b_\Delta(\cdot|\textcolor{lightblue}{x_t})}
        \left[{\mathcal{W}
        (\textcolor{lightred}{\pi^{\tau}(\cdot|x^{\tau}_t)}} || 
        \textcolor{lightblue}{\pi(\cdot|x_t)})
        \right]
    \end{aligned}
    $$
\end{theorem}
\begin{proof}
We can rewrite the left hand side and apply the Proposition \ref{relation_q_pi} to get the result
$$
    \begin{aligned}
        &\mathop{\mathbb{E}}_{
        \textcolor{lightred}{x^{\tau}_t}\sim b_\Delta(\cdot|\textcolor{lightblue}{x_t})\atop
        \textcolor{lightblue}{a_t\sim\pi(\cdot|x_t)}
        }\left[
        \textcolor{lightred}{V^{\tau}(x^{\tau}_t)}
        -
        \textcolor{lightred}{Q^{\tau}}(\textcolor{lightred}{x^{\tau}_t}, \textcolor{lightblue}{a_t})
        \right]\\
        =&\mathop{\mathbb{E}}_{
        \textcolor{lightred}{x^{\tau}_t}\sim b_\Delta(\cdot|\textcolor{lightblue}{x_t})\atop
        \textcolor{lightred}{a_t\sim\pi^{\tau}(\cdot|x^{\tau}_t)}
        }\left[\textcolor{lightred}{Q^{\tau}(x^{\tau}_t, a_t)}\right]
        - 
        \mathop{\mathbb{E}}_{
        \textcolor{lightred}{x^{\tau}_t}\sim b_\Delta(\cdot|\textcolor{lightblue}{x_t})\atop
        \textcolor{lightblue}{a_t\sim\pi(\cdot|x_t)}
        }\left[\textcolor{lightred}{Q^{\tau}}(\textcolor{lightred}{x^{\tau}_t}, \textcolor{lightblue}{a_t})\right]& \\
        =&\mathop{\mathbb{E}}_{
        \textcolor{lightred}{x^{\tau}_t}\sim b_\Delta(\cdot|\textcolor{lightblue}{x_t})\atop
        \textcolor{lightred}{a_t\sim\pi^{\tau}(\cdot|x^{\tau}_t)}
        }\left[\textcolor{lightred}{Q^{\tau}(x^{\tau}_t, a_t)}\right]
        - 
        \mathop{\mathbb{E}}_{
        \textcolor{lightred}{x^{\tau}_t}\sim b_\Delta(\cdot|\textcolor{lightblue}{x_t})\atop
        \textcolor{lightblue}{a_t\sim\pi(\cdot|x_t)}
        }\left[\textcolor{lightred}{Q^{\tau}}(\textcolor{lightred}{x^{\tau}_t}, \textcolor{lightblue}{a_t})\right]& \\
        =&\mathop{\mathbb{E}}_{
        \textcolor{lightred}{x^{\tau}_t}\sim b_\Delta(\cdot|\textcolor{lightblue}{x_t})
        }\left[\mathop{\mathbb{E}}_{
        \textcolor{lightred}{a_t\sim\pi^{\tau}(\cdot|x^{\tau}_t)}
        }\left[\textcolor{lightred}{Q^{\tau}(x^{\tau}_t, a_t)}\right]
        - 
        \mathop{\mathbb{E}}_{
        \textcolor{lightblue}{a_t\sim\pi(\cdot|x_t)}
        }\left[\textcolor{lightred}{Q^{\tau}}(\textcolor{lightred}{x^{\tau}_t}, \textcolor{lightblue}{a_t})\right]\right]& \\
        \leq& L_Q \mathop{\mathbb{E}}_{\textcolor{lightred}{x^{\tau}_t}\sim b_\Delta(\cdot|\textcolor{lightblue}{x_t})}\left[{\mathcal{W}
        (\textcolor{lightred}{\pi^{\tau}(\cdot|x^{\tau}_t)}} || 
        \textcolor{lightblue}{\pi(\cdot|x_t)})
        \right]
    \end{aligned}
$$

\end{proof}

\begin{theorem}[Delayed Q-value Difference Bound]
    \label{appendix_delayed_q_value_difference_bound}
    For policies \textcolor{lightblue}{$\pi$} and \textcolor{lightred}{$\pi^{\tau}$}, with \textcolor{lightred}{$\Delta^{\tau}$} $<$ \textcolor{lightblue}{$\Delta$}. Given any \textcolor{lightblue}{$x_t$} $\in$ \textcolor{lightblue}{$\mathcal{X}$}, if \textcolor{lightred}{$Q^{\tau}$} is $L_Q$-LC, the corresponding Q-value difference can be bounded as follow
    $$
    \begin{aligned}
        \mathop{\mathbb{E}}_{
        \textcolor{lightblue}{a_t\sim\pi(\cdot|x_t)}\atop
        \textcolor{lightred}{x^{\tau}_t}\sim b_\Delta(\cdot|\textcolor{lightblue}{x_t})
        }\left[\textcolor{lightred}{Q^{\tau}}(\textcolor{lightred}{x^{\tau}_t}, \textcolor{lightblue}{a_t}) - \textcolor{lightblue}{Q(x_t, a_t)}\right]
         \leq \frac{\gamma L_Q}{1-\gamma} \mathop{\mathbb{E}}_{\substack{
         \textcolor{lightred}{\hat{x}^{\tau}}\sim b_\Delta(\cdot|\textcolor{lightblue}{\hat{x}})
         \textcolor{lightblue}{\hat{x}\sim d_{x_{t+1}}^{\pi}}\\
         \textcolor{lightblue}{x_{t+1}\sim \mathcal{P}_{\Delta}(\cdot|x_t, a_t)}\\
         \textcolor{lightblue}{a_t\sim\pi(\cdot|x_t)}}}
         \left[\mathcal{W}_1(
         \textcolor{lightred}{\pi^{\tau}(\cdot|\hat{x}^{\tau})} || 
         \textcolor{lightblue}{\pi(\cdot|\hat{x})}
         )\right]
    \end{aligned}
    $$
    Specially, for optimal policies \textcolor{lightblue}{$\pi_{(*)}$} and \textcolor{lightred}{$\pi^{\tau}_{(*)}$}, if \textcolor{lightred}{$Q^{\tau}_{(*)}$} is $L_Q$-LC, the corresponding optimal Q-value difference can be bounded as follow
    $$
        \begin{aligned}
            \left|\left|\mathop{\mathbb{E}}_{\textcolor{lightred}{x^{\tau}_t}\sim b_\Delta(\cdot|\textcolor{lightblue}{x_t})}\left[\textcolor{lightred}{Q^{\tau}_{(*)}(\textcolor{lightred}{x^{\tau}_t}}, a_t)\right] - \textcolor{lightblue}{Q_{(*)}(x_t}, a_t)\right|\right|_{\infty}
            \leq
            \frac{\gamma^2 L_Q}{(1-\gamma)^2} \mathop{\mathbb{E}}_{\substack{
                \textcolor{lightred}{\hat{x}^{\tau}}\sim b_\Delta(\cdot|\textcolor{lightblue}{\hat{x}})\\
                \textcolor{lightblue}{\hat{x}\sim d_{x_{t+1}}^{\pi}}\\
                \textcolor{lightblue}{x_{t+1}\sim \mathcal{P}_{\Delta}(\cdot|x_t, a_t)}\\
                \textcolor{lightblue}{a_t\sim\pi_{(*)}(\cdot|x_t)}}}
                \left[\mathcal{W}\left(
                \textcolor{lightred}{\pi^{\tau}_{(*)}(\cdot|\hat{x}^{\tau})} \middle|\middle| 
                \textcolor{lightblue}{\pi_{(*)}(\cdot|\hat{x})}
                \right)\right]
        \end{aligned}
    $$
\end{theorem}
\begin{proof}
    For policies \textcolor{lightblue}{$\pi$} and \textcolor{lightred}{$\pi^{\tau}$}, we can rewrite the left hand side as
    $$
    \begin{aligned}
        &\mathop{\mathbb{E}}_{
        \textcolor{lightblue}{a_t\sim\pi(\cdot|x_t)}\atop
        \textcolor{lightred}{x^{\tau}_t}\sim b_\Delta(\cdot|\textcolor{lightblue}{x_t})
        }\left[\textcolor{lightred}{Q^{\tau}}(\textcolor{lightred}{x^{\tau}_t}, \textcolor{lightblue}{a_t}) - \textcolor{lightblue}{Q(x_t, a_t)}\right] \\
        =& \mathop{\mathbb{E}}_{
        \textcolor{lightblue}{a_t\sim\pi(\cdot|x_t)}\atop
        \textcolor{lightred}{x^{\tau}_t}\sim b_\Delta(\cdot|\textcolor{lightblue}{x_t})
        }\left[
        \textcolor{lightred}{\mathcal{R}_\tau}(\textcolor{lightred}{x^{\tau}_t}, \textcolor{lightblue}{a_t}) 
        + \gamma \mathop{\mathbb{E}}_{\textcolor{lightred}{x^{\tau}_{t+1}\sim\mathcal{P}_{\Delta^{\tau}}(\cdot|x^{\tau}_t, \textcolor{lightblue}{a_t}}\textcolor{lightred}{)}}\left[\textcolor{lightred}{V^{\tau}}(\textcolor{lightred}{x^{\tau}_{t+1}})\right]
        \right]
        -
        \mathop{\mathbb{E}}_{
        \textcolor{lightblue}{a_t\sim\pi(\cdot|x_t)}
        }\left[
        \textcolor{lightblue}{\mathcal{R}_\Delta}(\textcolor{lightblue}{x_t}, \textcolor{lightblue}{a_t}) 
        + \gamma \mathop{\mathbb{E}}_{\textcolor{lightblue}{x_{t+1}\sim\mathcal{P}_{\Delta}(\cdot|x_t, \textcolor{lightblue}{a_t}}\textcolor{lightblue}{)}}\left[\textcolor{lightblue}{V}(\textcolor{lightblue}{x_{t+1}})\right]
        \right]\\
        =& \gamma \mathop{\mathbb{E}}_{
        \textcolor{lightblue}{a_t\sim\pi(\cdot|x_t)}
        }\left[
            \mathop{\mathbb{E}}_{
            \textcolor{lightred}{x^{\tau}_{t+1}\sim\mathcal{P}_{\Delta^{\tau}}(\cdot|x^{\tau}_t, \textcolor{lightblue}{a_t}}\textcolor{lightred}{)}\atop
            \textcolor{lightred}{x^{\tau}_t}\sim b_\Delta(\cdot|\textcolor{lightblue}{x_t})
            }\left[\textcolor{lightred}{V^{\tau}}(\textcolor{lightred}{x^{\tau}_{t+1}})\right]
            -
            \mathop{\mathbb{E}}_{\textcolor{lightblue}{x_{t+1}\sim\mathcal{P}_{\Delta}(\cdot|x_t, \textcolor{lightblue}{a_t}}\textcolor{lightblue}{)}}\left[\textcolor{lightblue}{V}(\textcolor{lightblue}{x_{t+1}})\right]
        \right]\\
        =& \gamma \mathop{\mathbb{E}}_{
        \textcolor{lightblue}{x_{t+1}\sim\mathcal{P}_{\Delta}(\cdot|x_t, \textcolor{lightblue}{a_t}}\textcolor{lightblue}{)}\atop
        \textcolor{lightblue}{a_t\sim\pi(\cdot|x_t)}
        }
            \underbrace{\left[\mathop{\mathbb{E}}_{
            \textcolor{lightred}{x^{\tau}_{t+1}}\sim b_\Delta(\cdot|\textcolor{lightblue}{x_{t+1}})
            }\left[\textcolor{lightred}{V^{\tau}}(\textcolor{lightred}{x^{\tau}_{t+1}})\right]
            -
            \textcolor{lightblue}{V}(\textcolor{lightblue}{x_{t+1}})\right]}_{I(\textcolor{lightblue}{x_{t+1}})}
        \\
        \leq & \frac{\gamma L_Q}{1-\gamma} \mathop{\mathbb{E}}_{\substack{
        \textcolor{lightred}{\hat{x}^{\tau}}\sim b_\Delta(\cdot|\textcolor{lightblue}{\hat{x}})\\
        \textcolor{lightblue}{\hat{x}\sim d_{x_{t+1}}^{\pi}}\\
        \textcolor{lightblue}{x_{t+1}\sim \mathcal{P}_{\Delta}(\cdot|x_t, a_t)}\\
        \textcolor{lightblue}{a_t\sim\pi(\cdot|x_t)}}}
        \left[\mathcal{W}(
        \textcolor{lightred}{\pi^{\tau}(\cdot|\hat{x}^{\tau})} || 
        \textcolor{lightblue}{\pi(\cdot|\hat{x})}
        )\right] \\
    \end{aligned}
    $$
    The last two steps are derived via applying Lemma \ref{appendix_general_delayed_performance_diff} and Theorem \ref{appendix_delayed_performance_difference_bound}
    Then based on the above result, for optimal policies \textcolor{lightblue}{$\pi_{(*)}$} and \textcolor{lightred}{$\pi^{\tau}_{(*)}$}, we have 
    $$
        \begin{aligned}
            &\left|\left|\mathop{\mathbb{E}}_{\textcolor{lightred}{x^{\tau}_t}\sim b_\Delta(\cdot|\textcolor{lightblue}{x_t})}\left[\textcolor{lightred}{Q^{\tau}_{(*)}(\textcolor{lightred}{x^{\tau}_t}}, \textcolor{lightblue}{a_t})\right] - \textcolor{lightblue}{Q_{(*)}(x_t, a_t)}\right|\right|_{\infty}\\
            =& \left|\left|\gamma \mathop{\mathbb{E}}_{
            \textcolor{lightred}{x^{\tau}_{t+1}}\sim b_{\Delta}(\cdot|\textcolor{lightblue}{x_{t+1}})\atop
            \textcolor{lightblue}{x_{t+1}\sim \mathcal{P}_{\Delta}(x_{t}, a_t)}
            }\left[\textcolor{lightred}{\max_{a_{t+1}}Q^{\tau}_{(*)}(x^{\tau}_{t+1}, a_{t+1})}\right]
            - \gamma \mathop{\mathbb{E}}_{\textcolor{lightblue}{x_{t+1}\sim \mathcal{P}_{\Delta}(x_{t}, a_t)}}\left[\textcolor{lightblue}{\max_{a_{t+1}}Q_{(*)}(x_{t+1}, a_{t+1})}\right]
            \right|\right|_{\infty}\\
            =& ||\gamma \mathop{\mathbb{E}}_{
            \textcolor{lightred}{x^{\tau}_{t+1}}\sim b_{\Delta}(\cdot|\textcolor{lightblue}{x_{t+1}})\atop
            \textcolor{lightblue}{x_{t+1}\sim \mathcal{P}_{\Delta}(x_{t}, a_t)}
            }\left[\textcolor{lightred}{\max_{a_{t+1}}Q^{\tau}_{(*)}(x^{\tau}_{t+1}, a_{t+1})}\right]
            - \gamma \mathop{\mathbb{E}}_{
            \substack{
            \textcolor{lightblue}{a_{t+1}\sim \pi_{(*)}(\cdot|x_{t+1})}\\
            \textcolor{lightred}{x^{\tau}_{t+1}}\sim b_{\Delta}(\cdot|\textcolor{lightblue}{x_{t+1}})\\
            \textcolor{lightblue}{x_{t+1}\sim \mathcal{P}_{\Delta}(x_{t}, a_t)}}
            }\left[\textcolor{lightred}{Q^{\tau}_{(*)}(x^{\tau}_{t+1}}, \textcolor{lightblue}{a_{t+1}})\right]\\
            &+ \gamma \mathop{\mathbb{E}}_{
            \substack{
            \textcolor{lightblue}{a_{t+1}\sim \pi_{(*)}(\cdot|x_{t+1})}\\
            \textcolor{lightred}{x^{\tau}_{t+1}}\sim b_{\Delta}(\cdot|\textcolor{lightblue}{x_{t+1}})\\
            \textcolor{lightblue}{x_{t+1}\sim \mathcal{P}_{\Delta}(x_{t}, a_t)}}
            }\left[\textcolor{lightred}{Q^{\tau}_{(*)}(x^{\tau}_{t+1}}, \textcolor{lightblue}{a_{t+1}})\right]
            - \gamma \mathop{\mathbb{E}}_{\textcolor{lightblue}{x_{t+1}\sim \mathcal{P}_{\Delta}(x_{t}, a_t)}}\left[\textcolor{lightblue}{\max_{a_{t+1}}Q_{(*)}(x_{t+1}, a_{t+1})}\right]
            ||_{\infty}\\
            \leq&\gamma \left|\left|\mathop{\mathbb{E}}_{\textcolor{lightred}{x^{\tau}_t}\sim b_\Delta(\cdot|\textcolor{lightblue}{x_t})}\left[\textcolor{lightred}{Q^{\tau}_{(*)}(\textcolor{lightred}{x^{\tau}_t}}, \textcolor{lightblue}{a_t})\right] - \textcolor{lightblue}{Q_{(*)}(x_t, a_t)}\right|\right|_{\infty}
            +
            \gamma
            \frac{\gamma L_Q}{1-\gamma} \mathop{\mathbb{E}}_{\substack{
            \textcolor{lightred}{\hat{x}^{\tau}}\sim b_\Delta(\cdot|\textcolor{lightblue}{\hat{x}})\\
            \textcolor{lightblue}{\hat{x}\sim d_{x_{t+1}}^{\pi}}\\
            \textcolor{lightblue}{x_{t+1}\sim \mathcal{P}_{\Delta}(\cdot|x_t, a_t)}\\
            \textcolor{lightblue}{a_t\sim\pi_{(*)}(\cdot|x_t)}}}
            \left[\mathcal{W}\left(
            \textcolor{lightred}{\pi^{\tau}_{(*)}(\cdot|\hat{x}^{\tau})} \middle|\middle| 
            \textcolor{lightblue}{\pi_{(*)}(\cdot|\hat{x})}
            \right)\right]
            \\
        \end{aligned}
    $$
\end{proof}

\subsection{Convergence Analysis}
\label{appendix_convergence}

Here, we recall that we assume the action-space $\mathcal{A}$ is finite where $|\mathcal{A}| < \infty$.
When the assumption is satisfied, for any \textcolor{lightblue}{$x$} $\in$ \textcolor{lightblue}{$\mathcal{X}$}, the $L1$-Wasserstein distance between two delayed policies \textcolor{lightblue}{$\pi$} and \textcolor{lightred}{$\pi^{\tau}$} becomes the $l_1$ distance and it is bounded.

\begin{equation}
    \label{bounded_policy_w_d}
    \mathop{\mathbb{E}}_{\textcolor{lightred}{x^{\tau}}\sim b_\Delta(\cdot|\textcolor{lightblue}{x})}\left[
        \mathcal{W}\left(
         \textcolor{lightred}{\pi^{\tau}(\cdot|x^{\tau})} \middle|\middle| 
         \textcolor{lightblue}{\pi(\cdot|x)}
        \right)
    \right]
    < \infty
\end{equation}
Furthermore, the entropy of policy \textcolor{lightblue}{$\pi$} is also bounded
\begin{equation}
    \label{bounded_entropy}
    \textcolor{lightblue}{\mathcal{H}\left(\pi(\cdot|x)\right)} < \infty
\end{equation}

\begin{theorem}[AD-VI Convergence Guarantee]
    \label{appendix_vi_convergence}
    Consider the bellman operator $\mathcal{T}$ in Eq.~\eqref{aux_vi} and the initial Q-function \textcolor{lightblue}{$Q_{(0)}$}: $\textcolor{lightblue}{\mathcal{X}}\times\mathcal{A}\rightarrow\mathbb{R}$ with $|\mathcal{A}| < \infty$, and define a sequence \textcolor{lightblue}{$\{Q_{(k)}\}_{k=0}^\infty$} where \textcolor{lightblue}{$Q_{(k+1)}$}$= \mathcal{T}$\textcolor{lightblue}{$Q_{(k)}$}. As $k\rightarrow\infty$, \textcolor{lightblue}{$Q_{(k)}$} will converge to the fixed point \textcolor{lightblue}{$Q_{(\approx)}$} with \textcolor{lightred}{$Q^{\tau}$} converges to \textcolor{lightred}{$Q^{\tau}_{(*)}$}. And for any $(\textcolor{lightblue}{x_t}, a_t)$ $\in$ $\textcolor{lightblue}{\mathcal{X}} \times \mathcal{A}$, we have
    $$
        \textcolor{lightblue}{Q_{(\approx)}}(\textcolor{lightblue}{x_t}, a_t)
        = 
        \mathop{\mathbb{E}}_{\textcolor{lightred}{x^{\tau}_t}\sim b_\Delta(\cdot|\textcolor{lightblue}{x_t})}\left[\textcolor{lightred}{Q^{\tau}_{(*)}}(\textcolor{lightred}{x^{\tau}_t}, a_t)\right]
    $$

\end{theorem}
\begin{proof}
    The update rule of the bellman operator $\mathcal{T}$ is as follows
    $$
        \textcolor{lightblue}{Q(x_t, a_t)} \leftarrow \textcolor{lightblue}{\mathcal{R}_\Delta}(\textcolor{lightblue}{x_t}, a_t) + \gamma \mathop{\mathbb{E}}_{
        \textcolor{lightred}{x^{\tau}_{t+1}}\sim b_\Delta(\cdot|\textcolor{lightblue}{x_{t+1}})\atop
        \textcolor{lightblue}{x_{t+1}}\sim \textcolor{lightblue}{\mathcal{P}_\Delta}(\cdot|\textcolor{lightblue}{x_{t}}, a_{t})
        }\left[\textcolor{lightred}{Q^{\tau}(x^{\tau}_{t+1}}, \textcolor{lightblue}{{\mathop{\arg\max}_{a_{t+1}}}Q(x_{t+1}, a_{t+1})})\right]
    $$

    and the right hand side can be written as
    $$\scriptsize
        \underbrace{\textcolor{lightblue}{\mathcal{R}_\Delta}(\textcolor{lightblue}{x_t}, a_t) 
        + \gamma\mathop{\mathbb{E}}_{
        \textcolor{lightred}{x^{\tau}_{t+1}}\sim b_\Delta(\cdot|\textcolor{lightblue}{x_{t+1}})\atop
        \textcolor{lightblue}{x_{t+1}}\sim \textcolor{lightblue}{\mathcal{P}_\Delta}(\cdot|\textcolor{lightblue}{x_{t}}, a_{t})
        }\left[\textcolor{lightred}{Q^{\tau}(x^{\tau}_{t+1}}, \textcolor{lightblue}{{\mathop{\arg\max}_{a_{t+1}}}Q(x_{t+1}, a_{t+1})})
        - \textcolor{lightblue}{\max_{a_{t+1}}Q(x_{t+1}}, \textcolor{lightblue}{a_{t+1}})
        \right]}_{r^{\tau}(\textcolor{lightblue}{x_t}, a_t)}
        \\
        + \gamma
        \mathop{\mathbb{E}}_{\textcolor{lightblue}{x_{t+1}}\sim \textcolor{lightblue}{\mathcal{P}_\Delta}(\cdot|\textcolor{lightblue}{x_{t}}, a_{t})}
        \textcolor{lightblue}{\max_{a_{t+1}}Q(x_{t+1}}, \textcolor{lightblue}{a_{t+1}})
    $$
    
    Next, we prove that $r^{\tau}(\textcolor{lightblue}{x_t}, a_t)$ is bounded
    $$
    \begin{aligned}
        r^{\tau}(\textcolor{lightblue}{x_t}, a_t) &= \textcolor{lightblue}{\mathcal{R}_\Delta}(\textcolor{lightblue}{x_t}, a_t) 
        + \gamma\mathop{\mathbb{E}}_{
        \textcolor{lightred}{x^{\tau}_{t+1}}\sim b_\Delta(\cdot|\textcolor{lightblue}{x_{t+1}})\atop
        \textcolor{lightblue}{x_{t+1}}\sim \textcolor{lightblue}{\mathcal{P}_\Delta}(\cdot|\textcolor{lightblue}{x_{t}}, a_{t})
        }\left[\textcolor{lightred}{Q^{\tau}(x^{\tau}_{t+1}}, \textcolor{lightblue}{{\mathop{\arg\max}_{a_{t+1}}}Q(x_{t+1}, a_{t+1})})
        - \textcolor{lightblue}{\max_{a_{t+1}}Q(x_{t+1}}, \textcolor{lightblue}{a_{t+1}})
        \right]\\
        &= \textcolor{lightblue}{\mathcal{R}_\Delta}(\textcolor{lightblue}{x_t}, a_t) 
        + \gamma\mathop{\mathbb{E}}_{
        \substack{\textcolor{lightred}{x^{\tau}_{t+1}}\sim b_\Delta(\cdot|\textcolor{lightblue}{x_{t+1}})\\
        \textcolor{lightblue}{a_{t+1} = {\mathop{\arg\max}_{a_{t+1}}}Q(x_{t+1}, a_{t+1})}\\
        \textcolor{lightblue}{x_{t+1}}\sim \textcolor{lightblue}{\mathcal{P}_\Delta}(\cdot|\textcolor{lightblue}{x_{t}}, a_{t})
        }}
        \left[\textcolor{lightred}{Q^{\tau}(x^{\tau}_{t+1}}, \textcolor{lightblue}{a_{t+1}})
        - \textcolor{lightblue}{\max_{a_{t+1}}Q(x_{t+1}}, \textcolor{lightblue}{a_{t+1}})
        \right]\\
        &\leq \textcolor{lightblue}{\mathcal{R}_\Delta}(\textcolor{lightblue}{x_t}, a_t) 
        + \gamma\frac{\gamma L_Q}{1-\gamma} \mathop{\mathbb{E}}_{
        \substack{
        \textcolor{lightred}{\hat{x}^{\tau}}\sim b_\Delta(\cdot|\textcolor{lightblue}{\hat{x}})\\
        \textcolor{lightblue}{\hat{x}\sim d_{x_{t+2}}^{\pi}}\\
        \textcolor{lightblue}{x_{t+2}\sim \mathcal{P}_{\Delta}(\cdot|x_{t+1}, a_{t+1})}\\
        \textcolor{lightblue}{a_{t+1} = {\mathop{\arg\max}_{a_{t+1}}}Q(x_{t+1}, a_{t+1})}\\
        \textcolor{lightblue}{x_{t+1}}\sim \textcolor{lightblue}{\mathcal{P}_\Delta}(\cdot|\textcolor{lightblue}{x_{t}}, a_{t})
        }}
        \left[\mathcal{W}\left(
            \textcolor{lightred}{\pi^{\tau}(\cdot|\hat{x}^{\tau})} \middle|\middle| 
            \textcolor{lightblue}{\pi(\cdot|\hat{x})}
        \right)\right]\\
        &< \infty
    \end{aligned}
    $$
    The last two steps are derived via applying Lemma \ref{appendix_delayed_q_value_difference_bound} and Eq.~\eqref{bounded_policy_w_d}, respectively. Then, we can get the property of convergence by applying original VI~\cite{rlai}, and the convergence is related to the \textcolor{lightred}{$Q^{\tau}$}.

    For the \textcolor{lightred}{$Q^{\tau}$}, we know that it converges to \textcolor{lightred}{$Q^{\tau}_{(*)}$} as it is updated by the original VI rule, and it satisfies that
    $$
        \mathop{\mathbb{E}}_{\textcolor{lightred}{x^{\tau}_{t}}\sim b_\Delta(\cdot|\textcolor{lightblue}{x_{t}})}\left[\textcolor{lightred}{Q^{\tau}_{(*)}(x^{\tau}_{t}}, \textcolor{lightblue}{a_t})\right]
        = \textcolor{lightred}{\mathcal{R}_\tau}(\textcolor{lightred}{x^\tau_t}, \textcolor{lightblue}{a_t}) + \gamma \mathop{\mathbb{E}}_{
        \textcolor{lightred}{x^{\tau}_{t+1}}\sim b_\Delta(\cdot|\textcolor{lightblue}{x_{t+1}})\atop
        \textcolor{lightblue}{x_{t+1}\sim \mathcal{P}_{\Delta}(\cdot|x_{t}, a_t)}
        }\left[\max_{{a_{t+1}}}\textcolor{lightred}{Q^{\tau}_{(*)}(x^{\tau}_{t+1}}, a_{t+1})\right]
    $$
    
    Without loss of generality, we can assume that the update of \textcolor{lightblue}{$Q_{(k)}$} is based on the \textcolor{lightred}{$Q^{\tau}_{(*)}$}, then the converged fixed point of \textcolor{lightblue}{$Q_{(k)} (k \rightarrow \infty)$}, denoted as \textcolor{lightblue}{$Q_{(\approx)}$}, satisfies the equation as followed:
    $$
        \textcolor{lightblue}{Q_{(\approx)}(x_t, a_t)} = \textcolor{lightblue}{\mathcal{R}_\Delta}(\textcolor{lightblue}{x_t}, \textcolor{lightblue}{a_t}) + \gamma \mathop{\mathbb{E}}_{
        \textcolor{lightred}{x^{\tau}_{t+1}}\sim b_\Delta(\cdot|\textcolor{lightblue}{x_{t+1}})\atop
        \textcolor{lightblue}{x_{t+1}\sim \mathcal{P}_{\Delta}(\cdot|x_{t}, a_t)}
        }\left[\textcolor{lightred}{Q^{\tau}_{(*)}(x^{\tau}_{t+1}}, \textcolor{lightblue}{{\mathop{\arg\max}_{a_{t+1}}}Q_{(\approx)}(x_{t+1}, a_{t+1})})\right]
    $$
    then we can get the fixed point
    
    $$
        \begin{aligned}        
            \textcolor{lightblue}{Q_{(\approx)}(x_t, a_t)} &= \textcolor{lightblue}{\mathcal{R}_\Delta}(\textcolor{lightblue}{x_t}, \textcolor{lightblue}{a_t}) + \gamma \mathop{\mathbb{E}}_{
            \textcolor{lightred}{x^{\tau}_{t+1}}\sim b_\Delta(\cdot|\textcolor{lightblue}{x_{t+1}})\atop
            \textcolor{lightblue}{x_{t+1}\sim \mathcal{P}_{\Delta}(\cdot|x_{t}, a_t)}
            }\left[\textcolor{lightred}{Q^{\tau}_{(*)}(x^{\tau}_{t+1}}, {\mathop{{\arg\max}}_{a_{t+1}}}\textcolor{lightblue}{Q_{(\approx)}(x_{t+1}}, a_{t+1}))\right]\\
            &= \textcolor{lightblue}{\mathcal{R}_\Delta}(\textcolor{lightblue}{x_t}, \textcolor{lightblue}{a_t}) + \gamma \mathop{\mathbb{E}}_{
            \textcolor{lightred}{x^{\tau}_{t+1}}\sim b_\Delta(\cdot|\textcolor{lightblue}{x_{t+1}})\atop
            \textcolor{lightblue}{x_{t+1}\sim \mathcal{P}_{\Delta}(\cdot|x_{t}, a_t)}
            }\left[\textcolor{lightred}{Q^{\tau}_{(*)}(x^{\tau}_{t+1}}, {\mathop{{\arg\max}}_{a_{t+1}}}
            \mathop{\mathbb{E}}_{\textcolor{lightred}{x^{\tau}_{t+1}}\sim b_\Delta(\cdot|\textcolor{lightblue}{x_{t+1}})} \left[\textcolor{lightred}{Q^{\tau}_{(*)}(x^{\tau}_{t+1}}, {a_{t+1}})\right]
            )\right]\\
            &=\mathop{\mathbb{E}}_{\textcolor{lightred}{x^{\tau}_{t}}\sim b_\Delta(\cdot|\textcolor{lightblue}{x_{t}})}\left[\textcolor{lightred}{\mathcal{R}_\tau}(\textcolor{lightred}{x^\tau_t}, \textcolor{lightblue}{a_t})\right] + \gamma \mathop{\mathbb{E}}_{
            \textcolor{lightred}{x^{\tau}_{t+1}}\sim b_\Delta(\cdot|\textcolor{lightblue}{x_{t+1}})\atop
            \textcolor{lightblue}{x_{t+1}\sim \mathcal{P}_{\Delta}(\cdot|x_{t}, a_t)}
            }\left[\max_{{a_{t+1}}}\textcolor{lightred}{Q^{\tau}_{(*)}(x^{\tau}_{t+1}}, a_{t+1})\right]\\
            &= \mathop{\mathbb{E}}_{\textcolor{lightred}{x^{\tau}_{t}}\sim b_\Delta(\cdot|\textcolor{lightblue}{x_{t}})}\left[\textcolor{lightred}{Q^{\tau}_{(*)}(x^{\tau}_{t}}, \textcolor{lightblue}{a_t})\right]\\
        \end{aligned}
    $$
    
\end{proof}

\begin{lemma}[Policy Evaluation Convergence Guarantee]
\label{appendix_soft_policy_evaluation}
Consider the soft bellman operator $\mathcal{T}^\pi$ in Eq.~\eqref{aux_pe} and the initial delayed Q-value function \textcolor{lightblue}{$Q_{(0)}$}: $\textcolor{lightblue}{\mathcal{X}}\times\mathcal{A}\rightarrow\mathbb{R}$ with $|\mathcal{A}| < \infty$, and define a sequence \textcolor{lightblue}{$\{Q_{(k)}\}_{k=0}^\infty$} where \textcolor{lightblue}{$Q_{(k+1)}$}$= \mathcal{T}^\pi$\textcolor{lightblue}{$Q_{(k)}$}. Then for any $(\textcolor{lightblue}{x_t}, a_t)$ $\in$ $\textcolor{lightblue}{\mathcal{X}} \times \mathcal{A}$, as $k\rightarrow\infty$, \textcolor{lightblue}{$Q_{(k)}(x_t, \textcolor{black}{a_t})$} will converge to the fixed point 
    $$
    \mathop{\mathbb{E}}_{
        \textcolor{lightred}{x^{\tau}_t}\sim b_\Delta(\cdot|\textcolor{lightblue}{x_t})
        }\left[
            \textcolor{lightred}{Q^{\tau}_{soft}}(\textcolor{lightred}{x^{\tau}_t}, a_t)
        \right]
        - \log\textcolor{lightblue}{\pi}(a_t|\textcolor{lightblue}{x_t})
    $$
\end{lemma}
\begin{proof}
    The update rule of the soft bellman operator $\mathcal{T}^\pi$ is as follows
    $$
        \textcolor{lightblue}{Q(x_t, a_t)} \leftarrow \textcolor{lightblue}{\mathcal{R}_\Delta}(\textcolor{lightblue}{x_t}, a_t) + \gamma \mathop{\mathbb{E}}_{\substack{
        \textcolor{lightred}{x^{\tau}_{t+1}}\sim b_\Delta(\cdot|\textcolor{lightblue}{x_{t+1}})\\
        \textcolor{lightblue}{a_{t+1} \sim \pi(\cdot|x_{t+1})}\\
        \textcolor{lightblue}{x_{t+1}}\sim \textcolor{lightblue}{\mathcal{P}_\Delta}(\cdot|\textcolor{lightblue}{x_{t}}, a_{t})
        }}
        \left[\textcolor{lightred}{Q^{\tau}(x^{\tau}_{t+1}}, \textcolor{lightblue}{a_{t+1}})
        - \log\textcolor{lightblue}{\pi(a_{t+1}|x_{t+1})}\right]
    $$
    Similar to AD-VI, we rewrite the right-hand side as
    $$
        \underbrace{\textcolor{lightblue}{\mathcal{R}_\Delta}(\textcolor{lightblue}{x_t}, a_t) 
        + \gamma\mathop{\mathbb{E}}_{\substack{        \textcolor{lightred}{x^{\tau}_{t+1}}\sim b_\Delta(\cdot|\textcolor{lightblue}{x_{t+1}})\\
        \textcolor{lightblue}{a_{t+1} \sim \pi(\cdot|x_{t+1})}\\
        \textcolor{lightblue}{x_{t+1}}\sim \textcolor{lightblue}{\mathcal{P}_\Delta}(\cdot|\textcolor{lightblue}{x_{t}}, a_{t})
        }}
        \left[\textcolor{lightred}{Q^{\tau}(x^{\tau}_{t+1}}, \textcolor{lightblue}{a_{t+1}})
        - \textcolor{lightblue}{Q(x_{t+1}}, \textcolor{lightblue}{a_{t+1}})
        - \log\textcolor{lightblue}{\pi(a_{t+1}|x_{t+1})}
        \right]}_{r^{\pi}(\textcolor{lightblue}{x_t}, a_t)}
        \\
        + \gamma\mathop{\mathbb{E}}_{
            \textcolor{lightblue}{a_{t+1} \sim \pi(\cdot|x_{t+1})}\atop 
            \textcolor{lightblue}{x_{t+1}}\sim \textcolor{lightblue}{\mathcal{P}_\Delta}(\cdot|\textcolor{lightblue}{x_{t}}, a_{t})
        }
        \left[\textcolor{lightblue}{Q(x_{t+1}}, \textcolor{lightblue}{a_{t+1}})
        \right]
    $$
    Similarly, we prove that $r^{\pi}(\textcolor{lightblue}{x_t}, a_t)$ is bounded
    $$
    \begin{aligned}
        r^{\pi}(\textcolor{lightblue}{x_t}, a_t) &= \textcolor{lightblue}{\mathcal{R}_\Delta}(\textcolor{lightblue}{x_t}, a_t)  
        + \gamma\mathop{\mathbb{E}}_{\substack{        \textcolor{lightred}{x^{\tau}_{t+1}}\sim b_\Delta(\cdot|\textcolor{lightblue}{x_{t+1}})\\
        \textcolor{lightblue}{a_{t+1} \sim \pi(\cdot|x_{t+1})}\\
        \textcolor{lightblue}{x_{t+1}}\sim \textcolor{lightblue}{\mathcal{P}_\Delta}(\cdot|\textcolor{lightblue}{x_{t}}, a_{t})}
        }
        \left[\textcolor{lightred}{Q^{\tau}(x^{\tau}_{t+1}}, \textcolor{lightblue}{a_{t+1}})
        - \textcolor{lightblue}{Q(x_{t+1}}, \textcolor{lightblue}{a_{t+1}})
        - \log\textcolor{lightblue}{\pi(a_{t+1}|x_{t+1})}
        \right]\\
        &= \textcolor{lightblue}{\mathcal{R}_\Delta}(\textcolor{lightblue}{x_t}, a_t)  
        + \gamma\mathop{\mathbb{E}}_{\substack{
        \textcolor{lightred}{x^{\tau}_{t+1}}\sim b_\Delta(\cdot|\textcolor{lightblue}{x_{t+1}})\\
        \textcolor{lightblue}{a_{t+1} \sim \pi(\cdot|x_{t+1})}\\
        \textcolor{lightblue}{x_{t+1}}\sim \textcolor{lightblue}{\mathcal{P}_\Delta}(\cdot|\textcolor{lightblue}{x_{t}}, a_{t})}
        }
        \left[\textcolor{lightred}{Q^{\tau}(x^{\tau}_{t+1}}, \textcolor{lightblue}{a_{t+1}})
        - \textcolor{lightblue}{Q(x_{t+1}}, \textcolor{lightblue}{a_{t+1}})
        \right]
        - \gamma \log\textcolor{lightblue}{\mathcal{H}(\pi(\cdot|x_{t+1}))}
        \\
        &
        \leq \textcolor{lightblue}{\mathcal{R}_\Delta}(\textcolor{lightblue}{x_t}, a_t)  
        + \frac{\gamma L_Q}{1-\gamma} \mathop{\mathbb{E}}_{\substack{
        \textcolor{lightred}{\hat{x}^{\tau}}\sim b_\Delta(\cdot|\textcolor{lightblue}{\hat{x}})\\
        \textcolor{lightblue}{\hat{x}\sim d_{x_{t+2}}^{\pi}}\\
        \textcolor{lightblue}{x_{t+2}\sim \mathcal{P}_{\Delta}(\cdot|x_{t+1}, a_{t+1})}\\
        \textcolor{lightblue}{a_{t+1} \sim \pi(\cdot|x_{t+1})}\\
        \textcolor{lightblue}{x_{t+1}}\sim \textcolor{lightblue}{\mathcal{P}_\Delta}(\cdot|\textcolor{lightblue}{x_{t}}, a_{t})
        }}
        \left[\mathcal{W}(
            \textcolor{lightred}{\pi^{\tau}(\cdot|\hat{x}^{\tau})} || 
            \textcolor{lightblue}{\pi(\cdot|\hat{x})}
        )\right]
        - \gamma \log\textcolor{lightblue}{\mathcal{H}(\pi(\cdot|x_{t+1}))}
        \\
        &< \infty
    \end{aligned}
    $$
    The last step can be derived due to Eq.~\eqref{bounded_policy_w_d} and Eq.~\eqref{bounded_entropy}. Then, we can get the result from the original policy evaluation~\cite{rlai}.
    Similarly, without loss of generality, we can assume that \textcolor{lightred}{$Q^{\tau}$} has converged to \textcolor{lightred}{$Q^{\tau}_{(soft)}$}, and can get the fixed point easily:
    $$
    \mathop{\mathbb{E}}_{
        \textcolor{lightred}{x^{\tau}_{t}}\sim b_\Delta(\cdot|\textcolor{lightblue}{x_{t}})
        }\left[
            \textcolor{lightred}{Q^{\tau}_{(soft)}}(\textcolor{lightred}{x^{\tau}_{t}}, a_t)
        \right]
        - \log\textcolor{lightblue}{\pi(a_{t}|x_{t})}
    $$
\end{proof}

\begin{lemma}[Soft Policy Improvement Guarantee]
    \label{appendix_soft_policy_improvement}
    Consider the policy update rule in Eq.~\eqref{aux_pi}, and let \textcolor{lightblue}{$\pi_{old}, \pi_{new}$} be the old policy and new policy improved from the old one respectively. Then for any $(\textcolor{lightblue}{x_t}, a_t)$ $\in$ $\textcolor{lightblue}{\mathcal{X}} \times \mathcal{A}$ with $|\mathcal{A}| < \infty$, we have
    $
        \textcolor{lightblue}{Q_{old}(x_t}, a_t)
        \leq
        \textcolor{lightblue}{Q_{new}(x_t}, a_t)
    $.
\end{lemma}
\begin{proof}
    As 
    $$
        \textcolor{lightblue}{\pi_{new}} = \mathop{\arg\min}_{\textcolor{lightblue}{\pi'} \in \textcolor{lightblue}{\Pi}}\mathop{\mathbb{E}}_{
            \textcolor{lightred}{x^{\tau}_{t}} \sim b_\Delta(\cdot|\textcolor{lightblue}{x_{t}})\atop
            \textcolor{lightblue}{a_t \sim \pi'(\cdot|x_{t})}
        }
        \left[
            \log\textcolor{lightblue}{\pi'(a_t|x_t)}
            - \textcolor{lightred}{Q^{\tau}(x^{\tau}_{t}}, \textcolor{lightblue}{a_{t}})
        \right]
    $$
    So
    $$
        \mathop{\mathbb{E}}_{
            \textcolor{lightred}{x^{\tau}_{t}} \sim b_\Delta(\cdot|\textcolor{lightblue}{x_{t}})\atop
            \textcolor{lightblue}{a_t \sim \pi_{new}(\cdot|x_{t})}
        }
        \left[
            \log\textcolor{lightblue}{\pi_{new}(a_t|x_t)}
            - \textcolor{lightred}{Q^{\tau}(x^{\tau}_{t}}, \textcolor{lightblue}{a_{t}})
        \right]
        \leq
        \mathop{\mathbb{E}}_{
            \textcolor{lightred}{x^{\tau}_{t}} \sim b_\Delta(\cdot|\textcolor{lightblue}{x_{t}})\atop
            \textcolor{lightblue}{a_t \sim \pi_{old}(\cdot|x_{t})}
        }
        \left[
            \log\textcolor{lightblue}{\pi_{old}(a_t|x_t)}
            - \textcolor{lightred}{Q^{\tau}(x^{\tau}_{t}}, \textcolor{lightblue}{a_{t}})
        \right]
    $$
    Then, we have
    $$
    \begin{aligned}
        \textcolor{lightblue}{Q_{old}}(x_t, a_t)
        &=r_t + \gamma \mathop{\mathbb{E}}_{\substack{
        \textcolor{lightred}{x^{\tau}_{t+1}} \sim b_\Delta(\cdot|\textcolor{lightblue}{x_{t+1}})\\
        \textcolor{lightblue}{a_{t+1} \sim \pi_{old}(\cdot|x_{t+1})}\\
        \textcolor{lightblue}{x_{t+1} \sim \mathcal{P}_{\Delta}(\cdot|x_{t}, a_{t})}
        }}
        \left[
        \textcolor{lightred}{Q^{\tau}(x^{\tau}_{t+1}}, \textcolor{lightblue}{a_{t+1}})
         - \log\textcolor{lightblue}{\pi_{old}(a_{t+1}|x_{t+1})}
        \right]\\
        &\leq r_t + \gamma \mathop{\mathbb{E}}_{\substack{
        \textcolor{lightred}{x^{\tau}_{t+1}} \sim b_\Delta(\cdot|\textcolor{lightblue}{x_{t+1}})\\
        \textcolor{lightblue}{a_{t+1} \sim \pi_{new}(\cdot|x_{t+1})}\\
        \textcolor{lightblue}{x_{t+1} \sim \mathcal{P}_{\Delta}(\cdot|x_{t}, a_{t})}
        }}
        \left[
        \textcolor{lightred}{Q^{\tau}(x^{\tau}_{t+1}}, \textcolor{lightblue}{a_{t+1}})
         - \log\textcolor{lightblue}{\pi_{new}(a_{t+1}|x_{t+1})}
        \right]\\
        &=\textcolor{lightblue}{Q_{new}}(x_t, a_t)
    \end{aligned}
    $$
\end{proof}

\begin{theorem}[Soft Policy Iteration Convergence Guarantee]
    \label{appendix_spi_convergence}
    Applying policy evaluation in Eq.~\eqref{aux_pe} and policy improvement in Eq.~\eqref{aux_pi} repeatedly to any given policy \textcolor{lightblue}{$\pi\in\Pi$}, it converges to \textcolor{lightblue}{$\pi_{(\approx)}$} such that \textcolor{lightblue}{$Q^{\pi}(x_t$}$, a_t) \leq$ \textcolor{lightblue}{$Q^{\pi}_{(\approx)}(x_t$}$, a_t)$ for any \textcolor{lightblue}{$\pi\in\Pi$}, $(\textcolor{lightblue}{x_t}, a_t)$ $\in$ $\textcolor{lightblue}{\mathcal{X}} \times \mathcal{A}$ with $|\mathcal{A}| < \infty$.
\end{theorem}
\begin{proof}
    Based on Lemma \ref{appendix_soft_policy_evaluation} and Lemma \ref{appendix_soft_policy_improvement}, we have
    $$
    \begin{aligned}
        \textcolor{lightblue}{Q^{\pi}}(x_t, a_t)
        &=r_t + \gamma \mathop{\mathbb{E}}_{\substack{
        \textcolor{lightred}{x^{\tau}_{t+1}} \sim b_\Delta(\cdot|\textcolor{lightblue}{x_{t+1}})\\
        \textcolor{lightblue}{a_{t+1} \sim \pi(\cdot|x_{t+1})}\\
        \textcolor{lightblue}{x_{t+1} \sim \mathcal{P}_{\Delta}(\cdot|x_{t}, a_{t})}
        }}
        \left[
        \textcolor{lightred}{Q^{\tau}_{(soft)}(x^{\tau}_{t+1}}, \textcolor{lightblue}{a_{t+1}})
         - \log\textcolor{lightblue}{\pi(a_{t+1}|x_{t+1})}
        \right]\\
        &\leq r_t + \gamma \mathop{\mathbb{E}}_{\substack{
        \textcolor{lightred}{x^{\tau}_{t+1}} \sim b_\Delta(\cdot|\textcolor{lightblue}{x_{t+1}})\\
        \textcolor{lightblue}{a_{t+1} \sim \pi_{(\approx)}(\cdot|x_{t+1})}\\
        \textcolor{lightblue}{x_{t+1} \sim \mathcal{P}_{\Delta}(\cdot|x_{t}, a_{t})}
        }}
        \left[
        \textcolor{lightred}{Q^{\tau}_{(soft)}(x^{\tau}_{t+1}}, \textcolor{lightblue}{a_{t+1}})
         - \log\textcolor{lightblue}{\pi_{(\approx)}(a_{t+1}|x_{t+1})}
        \right]\\
        &=\textcolor{lightblue}{Q^{\pi}_{(\approx)}}(x_t, a_t)
    \end{aligned}
    $$
\end{proof}


\clearpage
\section{Stochastic MDP Case: Selection of Auxiliary Delays \textcolor{lightred}{$\Delta^\tau$}}
\label{appendix_sto}
We give the following Theorem \ref{sto_performance_diff} to exemplify that the naive selection (e.g. \textcolor{lightred}{$\Delta^\tau$} $=0$) might cause a larger performance gap and potential approximation error.

First of all, we introduce a stochastic MDP $\langle \mathcal{S}, \mathcal{A}, \mathcal{P}, \mathcal{R}, \gamma, \rho \rangle$~\cite{dida} where
\begin{itemize}
    \item $\mathcal{S} = \mathbb{R}$
    \item $\mathcal{A} = \mathbb{R}$
    \item $\mathcal{P}(s'|s, a) = \mathcal{N}(s + \frac{a}{L_\pi}, \sigma^2)$ which means that $s' = s + \frac{a}{L_\pi} + \epsilon$ where $\epsilon \sim \mathcal{N}(0, \sigma^2)$
    \item $\mathcal{R}(s, a) = -L_Q L_\pi |s + \frac{a}{L_\pi}|$
\end{itemize}
\begin{lemma}[value function upper bound~\cite{dida}] 
\label{vfub_sto_mdp}
Let $L_\pi, L_Q>0$, in the MDP defined above with a $L_\pi$-LC optimal policy and $L_Q$-LC optimal value function, given $\Delta > 0$, for any $\Delta$-delayed policy $\pi$ and any augmented state $x\in \mathcal{X}$, it's value function $V^\pi$ has following upper bound
\begin{equation}
\begin{aligned}
    \label{sto_mdp_v}
    V^\pi(x) 
    &\leq 
    - \frac{L_Q L_\pi}{1-\gamma} \frac{\sqrt{2\Delta}}{\sqrt{\pi}}\sigma \\
    &= V^{\pi}_{(*)}(x) \\
\end{aligned}
\end{equation}
where $V^{\pi}_{(*)}(x)$ is the value function of the optimal $\Delta$-delayed policy $\pi_{(*)}$.
\end{lemma}

Then, we can extend Lemma \ref{vfub_sto_mdp} in the context of our AD-RL.
Given any delay \textcolor{lightblue}{$\Delta$} and auxiliary delays \textcolor{lightred}{$\Delta^\tau$} ($<$\textcolor{lightblue}{$\Delta$}), the corresponding optimal value functions are \textcolor{lightblue}{$V_{(*)}$} and \textcolor{lightred}{$V^\tau_{(*)}$}, respectively.
For \textcolor{lightblue}{$V_{(*)}$} and \textcolor{lightred}{$V^\tau_{(*)}$}, we can derive the performance difference in the following Theorem.

\begin{theorem}
\label{sto_performance_diff}
    For optimal policies \textcolor{lightblue}{$\pi_{(*)}$} and \textcolor{lightred}{$\pi^\tau_{(*)}$}, for any \textcolor{lightblue}{$x_t$} $\in$ \textcolor{lightblue}{$\mathcal{X}$}, their corresponding performance difference $I(\textcolor{lightblue}{x_t})$ is
    $$
    I(\textcolor{lightblue}{x_t}) = 
    \frac{L_Q L_\pi}{1-\gamma} \sigma \frac{\sqrt{2}}{\sqrt{\pi}}
    \left[
    \sqrt{\textcolor{lightblue}{\Delta}}
    -
    \sqrt{\textcolor{lightred}{\Delta^\tau}}
    \right]
    $$
\end{theorem}
\begin{proof}
    By applying Lemma \ref{vfub_sto_mdp}, we can get that
    $$
    \begin{aligned}
        \textcolor{lightblue}{V_{(*)}(x_t)} 
        &= \mathop{\mathbb{E}}_{
                \textcolor{lightblue}{a_t}\sim \textcolor{lightblue}{\pi_{(*)}}(\cdot|\textcolor{lightblue}{x_t})
            } [\textcolor{lightblue}{Q_{(*)}(x_t, a_t)}]\\
        &= - \frac{L_Q L_\pi}{1-\gamma} \frac{\sqrt{2\textcolor{lightblue}{\Delta}}}{\sqrt{\pi}}\sigma\\
    \end{aligned}
    $$
    and
    $$
    \begin{aligned}
    \mathop{\mathbb{E}}_{\textcolor{lightred}{x^\tau_t} \sim b_\Delta(\cdot|\textcolor{lightblue}{x_t})}\left[
        \textcolor{lightred}{V^\tau_{(*)}(x^\tau_t)} 
    \right]
        &= \mathop{\mathbb{E}}_{
                \textcolor{lightred}{a_t \sim \pi^\tau_{(*)}(\cdot|x^\tau_t)}\atop
                \textcolor{lightred}{x^\tau_t} \sim b_\Delta(\cdot|\textcolor{lightblue}{x_t})
            } \left[\textcolor{lightred}{Q^\tau_{(*)}(x^\tau_t, a_t)}\right]\\
        &= - \frac{L_Q L_\pi}{1-\gamma} \frac{\sqrt{2\textcolor{lightred}{\Delta^\tau}}}{\sqrt{\pi}}\sigma\\
    \end{aligned}
    $$
    then based on Lemma \ref{lemma_general_delayed_performance_diff}, we can derive that
    $$
    \begin{aligned}
    I(\textcolor{lightblue}{x_t}) &=
            \mathop{\mathbb{E}}_{
                \textcolor{lightred}{x^{\tau}_t}\sim b_\Delta(\cdot|\textcolor{lightblue}{x_t})
            }\left[
                \textcolor{lightred}{V^{\tau}(x^{\tau}_t)}
            \right] 
            - \textcolor{lightblue}{V(x_t)}\\
    &=
    \frac{L_Q L_\pi}{1-\gamma} \sigma \frac{\sqrt{2}}{\sqrt{\pi}}
    \left[
    \sqrt{\textcolor{lightblue}{\Delta}}
    -
    \sqrt{\textcolor{lightred}{\Delta^\tau}}
    \right]\\
    \end{aligned}
    $$
\end{proof}

From Theorem \ref{sto_performance_diff}, we can observe that when the difference between delays (\textcolor{lightred}{$\Delta^\tau$} and \textcolor{lightblue}{$\Delta$}) increases, the performance difference becomes larger.
In other words, if selecting a extremely short \textcolor{lightred}{$\Delta^\tau$} (e.g., \textcolor{lightred}{$\Delta^\tau$} $=0$) for a long \textcolor{lightblue}{$\Delta$}, the value difference between \textcolor{lightred}{$Q^\tau_{(*)}$} and \textcolor{lightblue}{$Q_{(*)}$} might be enlarged, and introducing potential estimation bias.
Based on the analysis presented above, it actually exists a trade-off between sample efficiency (shorter \textcolor{lightred}{$\Delta^\tau$} is better) and approximation error (\textcolor{lightred}{$\Delta^\tau$} is closer to \textcolor{lightblue}{$\Delta$}).

\clearpage
\section{MuJoCo: Additional Experimental Results}
\label{appendix_mujoco_results}
\begin{table*}[h]
\centering
\caption{Results of MuJoCo tasks with 5, 25 and 50 delays for 1M global time-steps. Each method was evaluated with 10 trials and is shown with the standard deviation denoted by $\pm$. The best performance is in blue.}
\scalebox{0.9}{
\begin{tabular}{c|cc|c|ccccc}
\hline
Task                                 & Delay-free SAC              & Random                     & Delays & A-SAC             & DC/AC              & DIDA               & BPQL               & AD-SAC (ours)  \\ \hline
                                     &                             &                            & 5      & $0.18_{\pm 0.01}$ & $0.25_{\pm 0.05}$  & $0.89_{\pm 0.03}$  & \textcolor{blue}{$0.96_{\pm 0.03}$}& $0.72_{\pm 0.25}$  \\
                                     &                             &                            & 25     & $0.07_{\pm 0.07}$ & $0.19_{\pm 0.02}$  & $0.29_{\pm 0.07}$  & $0.57_{\pm 0.11}$  & \textcolor{blue}{$0.66_{\pm 0.04}$} \\
\multirow{-3}{*}{Ant-v4}             & \multirow{-3}{*}{5053.30}   & \multirow{-3}{*}{-123.88}  & 50     & $0.02_{\pm 0.04}$ & $0.19_{\pm 0.02}$  & $0.19_{\pm 0.05}$  & $0.38_{\pm 0.07}$  & \textcolor{blue}{$0.48_{\pm 0.06}$} \\ \hline
                                     &                             &                            & 5      & $0.35_{\pm 0.15}$ & $0.40_{\pm 0.23}$  & $0.90_{\pm 0.01}$  & $1.00_{\pm 0.06}$  & \textcolor{blue}{$1.07_{\pm 0.06}$} \\
                                     &                             &                            & 25     & $0.04_{\pm 0.01}$ & $0.16_{\pm 0.07}$  & $0.12_{\pm 0.03}$  & \textcolor{blue}{$0.87_{\pm 0.04}$}& $0.71_{\pm 0.12}$ \\
\multirow{-3}{*}{HalfCheetah-v4}     & \multirow{-3}{*}{5322.74}   & \multirow{-3}{*}{-304.29}  & 50     & $0.12_{\pm 0.17}$ & $0.12_{\pm 0.13}$  & $0.15_{\pm 0.03}$  & $0.73_{\pm 0.17}$  & \textcolor{blue}{$0.74_{\pm 0.10}$} \\ \hline
                                     &                             &                            & 5      & $1.02_{\pm 0.28}$ & $1.16_{\pm 0.25}$  & $0.40_{\pm 0.40}$  & \textcolor{blue}{$1.25_{\pm 0.09}$}& $1.07_{\pm 0.30}$ \\
                                     &                             &                            & 25     & $0.13_{\pm 0.04}$ & $0.19_{\pm 0.04}$  & $0.27_{\pm 0.08}$  & \textcolor{blue}{$1.21_{\pm 0.18}$}& $0.86_{\pm 0.25}$ \\
\multirow{-3}{*}{Hopper-v4}          & \multirow{-3}{*}{2455.18}   & \multirow{-3}{*}{12.93}    & 50     & $0.04_{\pm 0.01}$ & $0.04_{\pm 0.01}$  & $0.09_{\pm 0.01}$  & $0.71_{\pm 0.13}$  & \textcolor{blue}{$0.72_{\pm 0.03}$} \\ \hline
                                     &                             &                            & 5      & $0.13_{\pm 0.02}$ & $0.59_{\pm 0.17}$  & $0.08_{\pm 0.04}$  & $0.96_{\pm 0.05}$  & \textcolor{blue}{$0.98_{\pm 0.07}$} \\
                                     &                             &                            & 25     & $0.05_{\pm 0.01}$ & $0.04_{\pm 0.01}$  & $0.07_{\pm 0.00}$  & $0.12_{\pm 0.01}$  & \textcolor{blue}{$0.25_{\pm 0.16}$} \\
\multirow{-3}{*}{Humanoid-v4}        & \multirow{-3}{*}{5269.34}   & \multirow{-3}{*}{135.47}   & 50     & $0.04_{\pm 0.01}$ & $0.03_{\pm 0.01}$  & $0.07_{\pm 0.00}$  & $0.08_{\pm 0.01}$  & \textcolor{blue}{$0.10_{\pm 0.01}$} \\ \hline
                                     &                             &                            & 5      & $1.02_{\pm 0.08}$ & $1.16_{\pm 0.12}$  & $1.00_{\pm 0.00}$  & $1.13_{\pm 0.07}$  & \textcolor{blue}{$1.22_{\pm 0.03}$} \\
                                     &                             &                            & 25     & $0.97_{\pm 0.09}$ & $1.03_{\pm 0.03}$  & $0.97_{\pm 0.02}$  & $1.09_{\pm 0.05}$  & \textcolor{blue}{$1.15_{\pm 0.08}$} \\
\multirow{-3}{*}{HumanoidStandup-v4} & \multirow{-3}{*}{129827.70} & \multirow{-3}{*}{35743.26} & 50     & $0.90_{\pm 0.02}$ & $1.02_{\pm 0.07}$  & $0.89_{\pm 0.06}$  & $1.06_{\pm 0.04}$  & \textcolor{blue}{$1.12_{\pm 0.02}$} \\ \hline
                                     &                             &                            & 5      & $1.11_{\pm 0.02}$ & $1.29_{\pm 0.05}$  & $1.01_{\pm 0.01}$  & $1.06_{\pm 0.08}$  & \textcolor{blue}{$1.36_{\pm 0.01}$} \\
                                     &                             &                            & 25     & $0.49_{\pm 0.32}$ & $1.12_{\pm 0.02}$  & $1.04_{\pm 0.01}$  & $1.07_{\pm 0.06}$  & \textcolor{blue}{$1.29_{\pm 0.03}$} \\
\multirow{-3}{*}{Pusher-v4}          & \multirow{-3}{*}{-56.03}    & \multirow{-3}{*}{-148.51}  & 50     & $0.00_{\pm 0.05}$ & $1.13_{\pm 0.01}$  & $1.04_{\pm 0.02}$  & $1.09_{\pm 0.05}$  & \textcolor{blue}{$1.23_{\pm 0.02}$} \\ \hline
                                     &                             &                            & 5      & $0.97_{\pm 0.01}$ & $1.02_{\pm 0.00}$  & $1.03_{\pm 0.00}$  & $1.00_{\pm 0.01}$  & \textcolor{blue}{$1.03_{\pm 0.01}$} \\
                                     &                             &                            & 25     & $0.96_{\pm 0.02}$ & \textcolor{blue}{$1.00_{\pm 0.00}$}  & $0.98_{\pm 0.01}$  & $0.87_{\pm 0.05}$  & $0.98_{\pm 0.02}$ \\
\multirow{-3}{*}{Reacher-v4}         & \multirow{-3}{*}{-5.89}     & \multirow{-3}{*}{-44.11}   & 50     & $0.86_{\pm 0.02}$ & $0.89_{\pm 0.01}$  & \textcolor{blue}{$0.93_{\pm 0.02}$}  & $0.90_{\pm 0.02}$  & $0.91_{\pm 0.03}$ \\ \hline
                                     &                             &                            & 5      & $0.88_{\pm 0.09}$ & $1.11_{\pm 0.30}$  & $1.05_{\pm 0.01}$  & $0.97_{\pm 0.02}$  & \textcolor{blue}{$1.82_{\pm 0.78}$} \\
                                     &                             &                            & 25     & $0.72_{\pm 0.02}$ & $0.78_{\pm 0.12}$  & $0.93_{\pm 0.09}$  & $1.36_{\pm 0.56}$  & \textcolor{blue}{$2.52_{\pm 0.40}$} \\
\multirow{-3}{*}{Swimmer-v4}         & \multirow{-3}{*}{49.41}     & \multirow{-3}{*}{-3.18}    & 50     & $0.69_{\pm 0.04}$ & $0.68_{\pm 0.06}$  & $0.87_{\pm 0.03}$  & $2.23_{\pm 0.55}$  & \textcolor{blue}{$2.71_{\pm 0.14}$} \\ \hline
                                     &                             &                            & 5      & $0.76_{\pm 0.21}$ & $0.85_{\pm 0.12}$  & $0.61_{\pm 0.07}$  & \textcolor{blue}{$1.20_{\pm 0.11}$}& $1.12_{\pm 0.09}$ \\
                                     &                             &                            & 25     & $0.12_{\pm 0.02}$ & $0.26_{\pm 0.08}$  & $0.10_{\pm 0.02}$  & $0.59_{\pm 0.30}$  & \textcolor{blue}{$0.72_{\pm 0.11}$} \\
\multirow{-3}{*}{Walker2d-v4}        & \multirow{-3}{*}{3999.54}   & \multirow{-3}{*}{1.48}     & 50     & $0.11_{\pm 0.02}$ & $0.11_{\pm 0.02}$  & $0.08_{\pm 0.01}$  & \textcolor{blue}{$0.23_{\pm 0.10}$}& $0.00_{\pm 0.11}$ \\ \hline
\end{tabular}
}
\end{table*}

\begin{figure*}[h]
\begin{center}
\centerline{
    \subfigure[Ant-v4]{\includegraphics[width=0.33\linewidth]{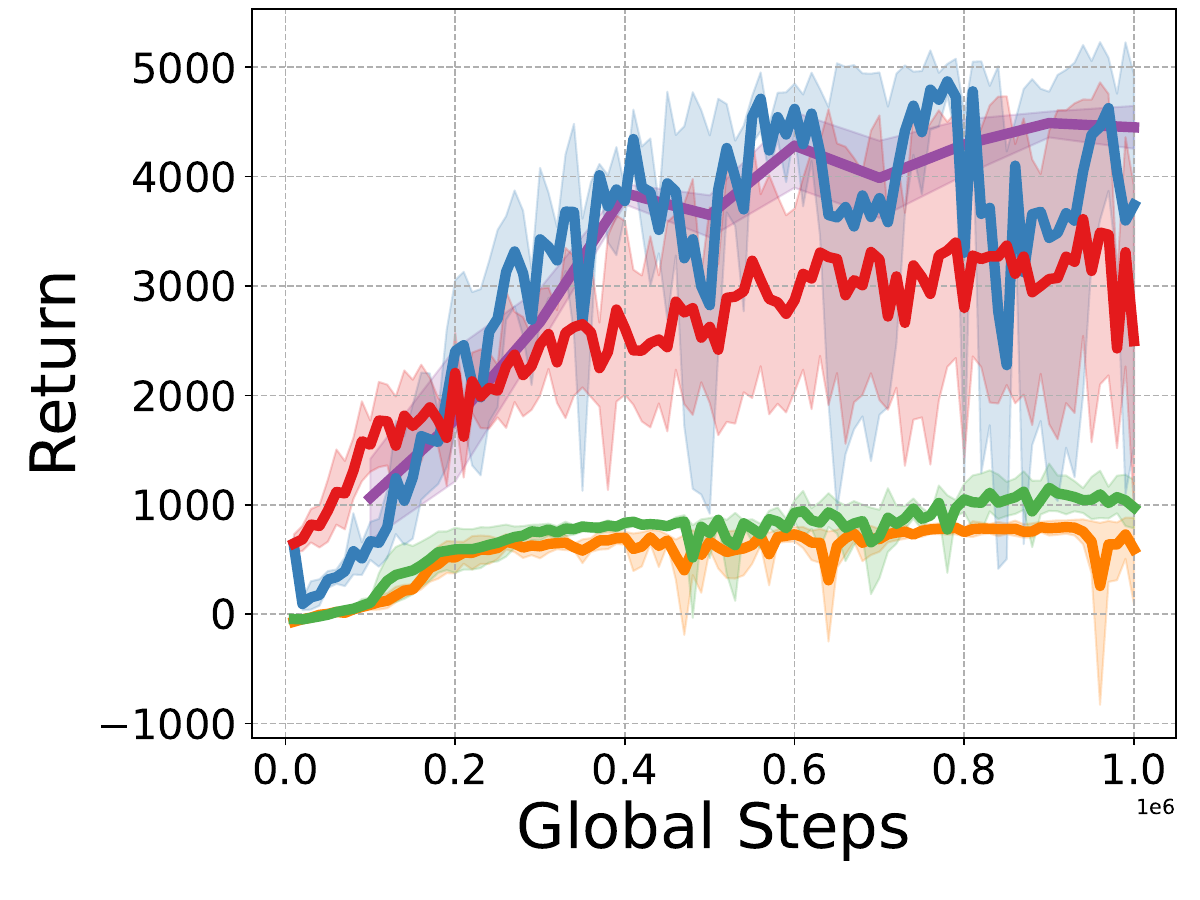}}
    \subfigure[HalfCheetah-v4]{\includegraphics[width=0.33\linewidth]{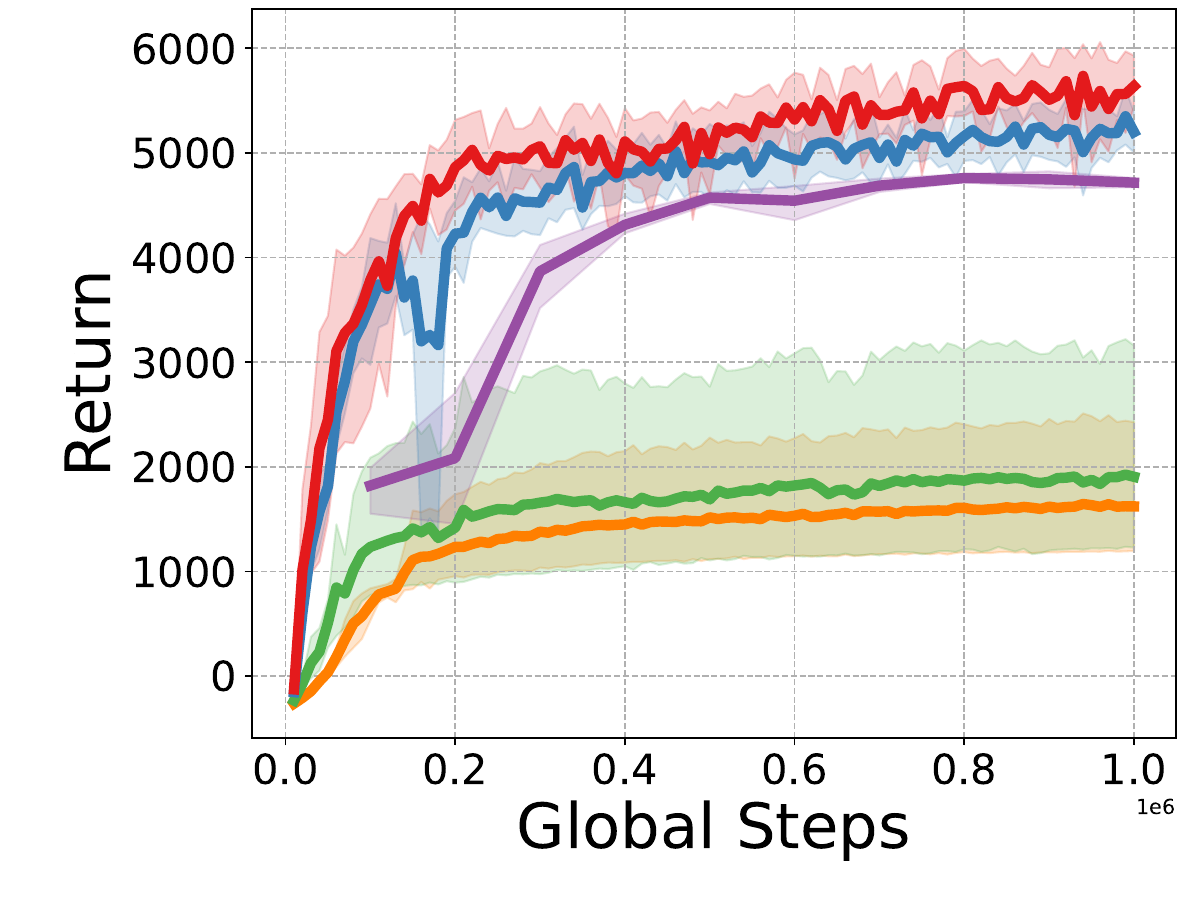}}
    \subfigure[Hopper-v4]{\includegraphics[width=0.33\linewidth]{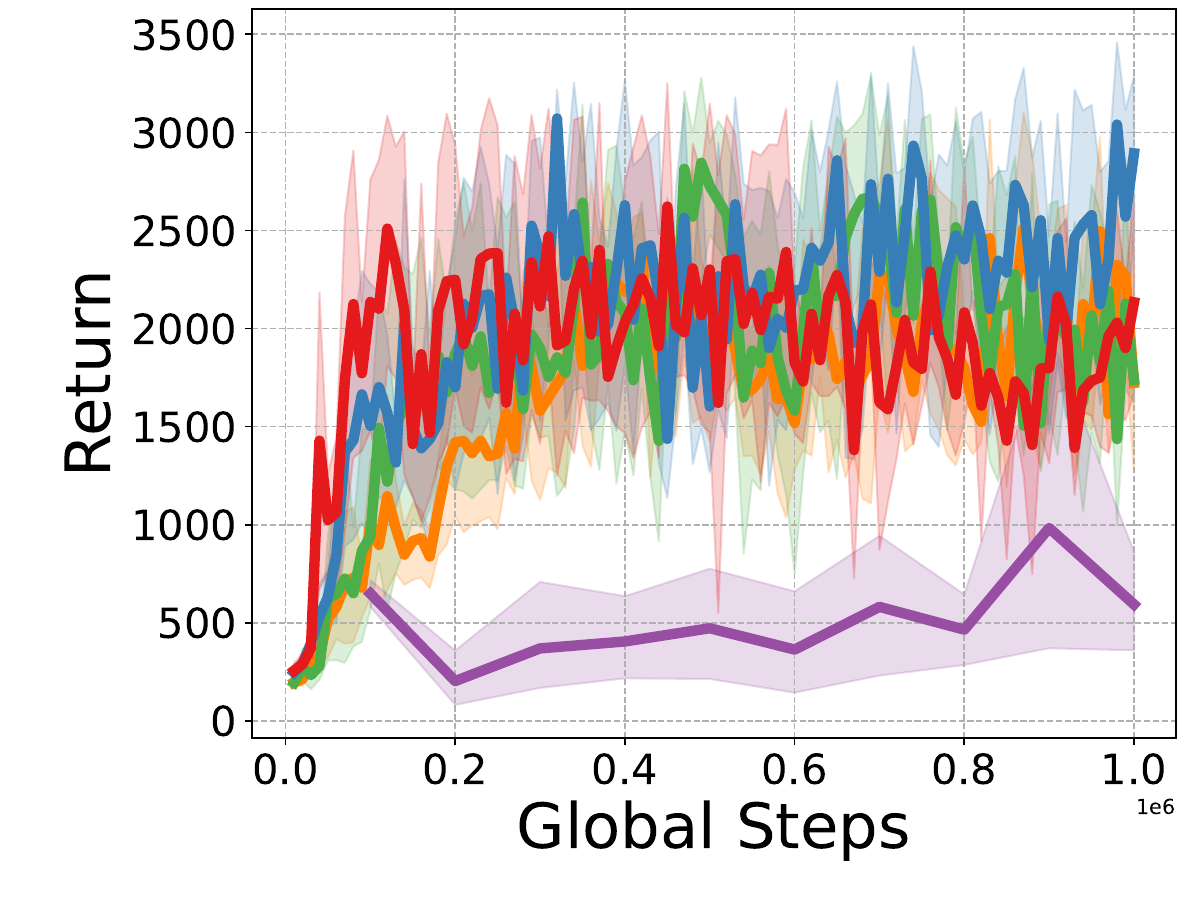}}
}
\centerline{
    \subfigure[Humanoid-v4]{\includegraphics[width=0.33\linewidth]{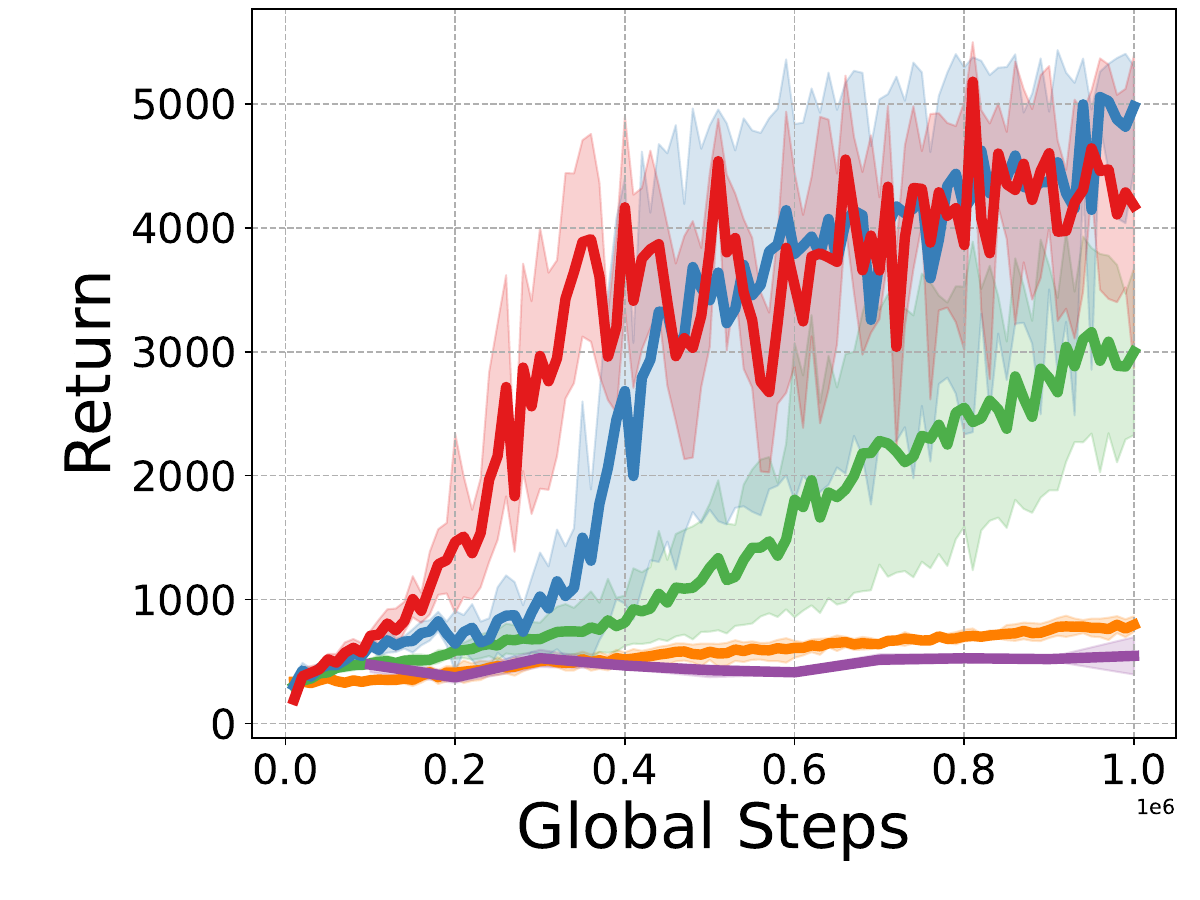}}
    \subfigure[HumanoidStandup-v4]{\includegraphics[width=0.33\linewidth]{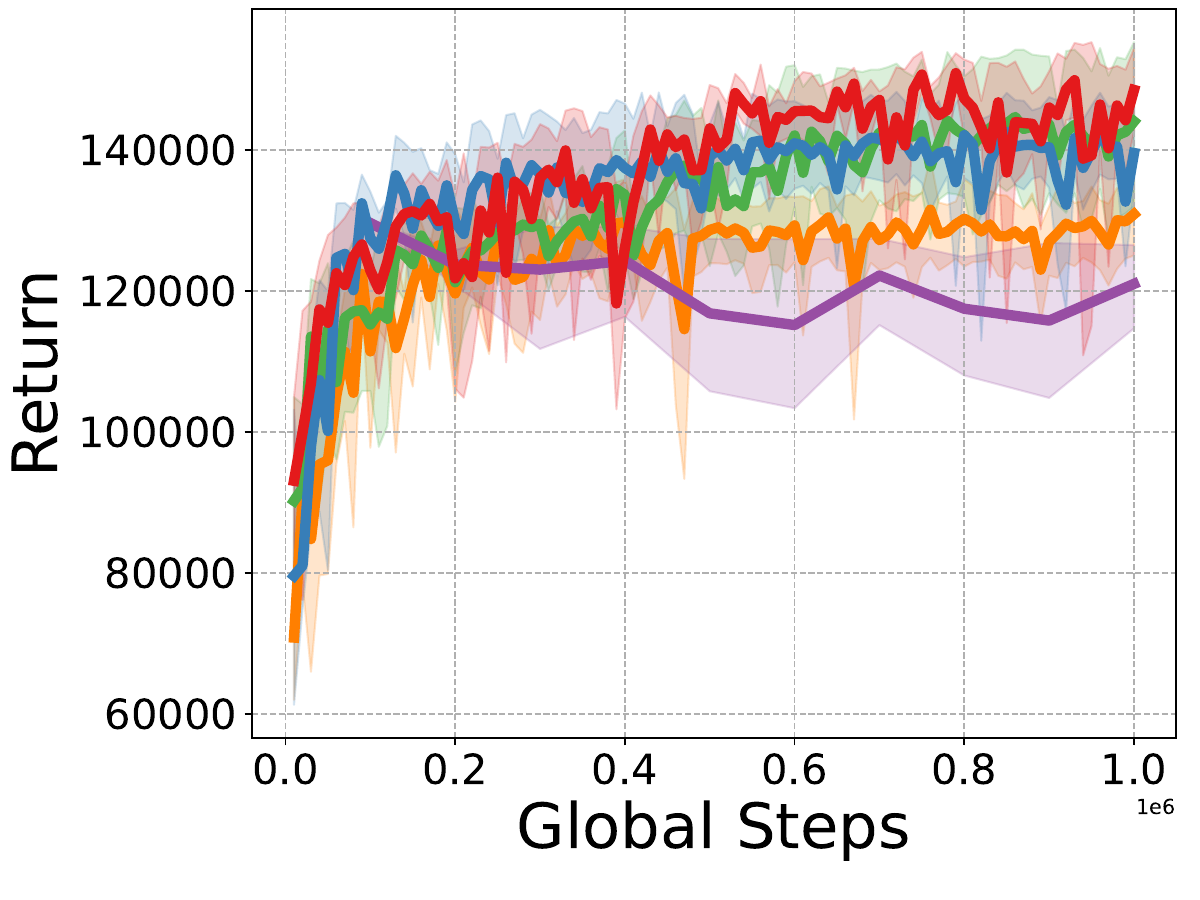}}
    \subfigure[Pusher-v4]{\includegraphics[width=0.33\linewidth]{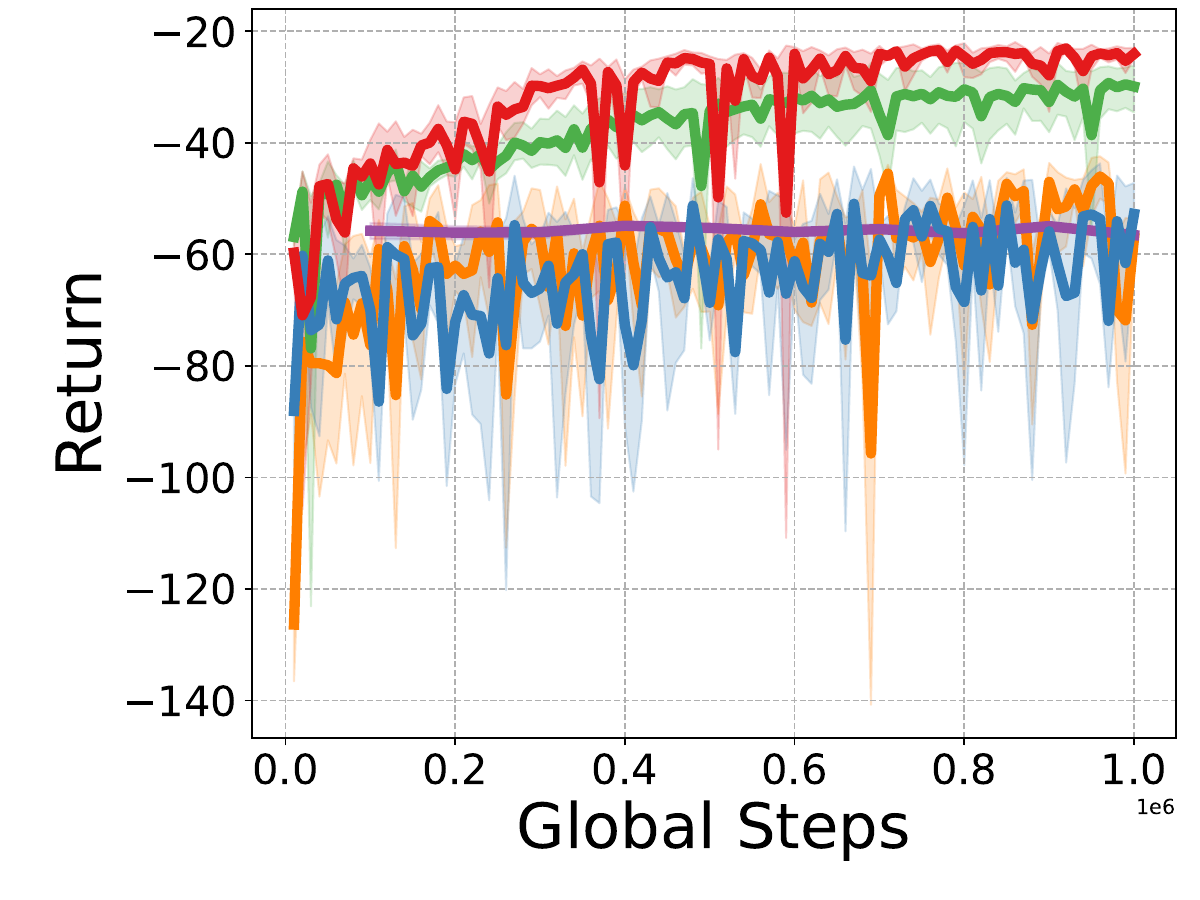}}
}
\centerline{
    \subfigure[Reacher-v4]{\includegraphics[width=0.33\linewidth]{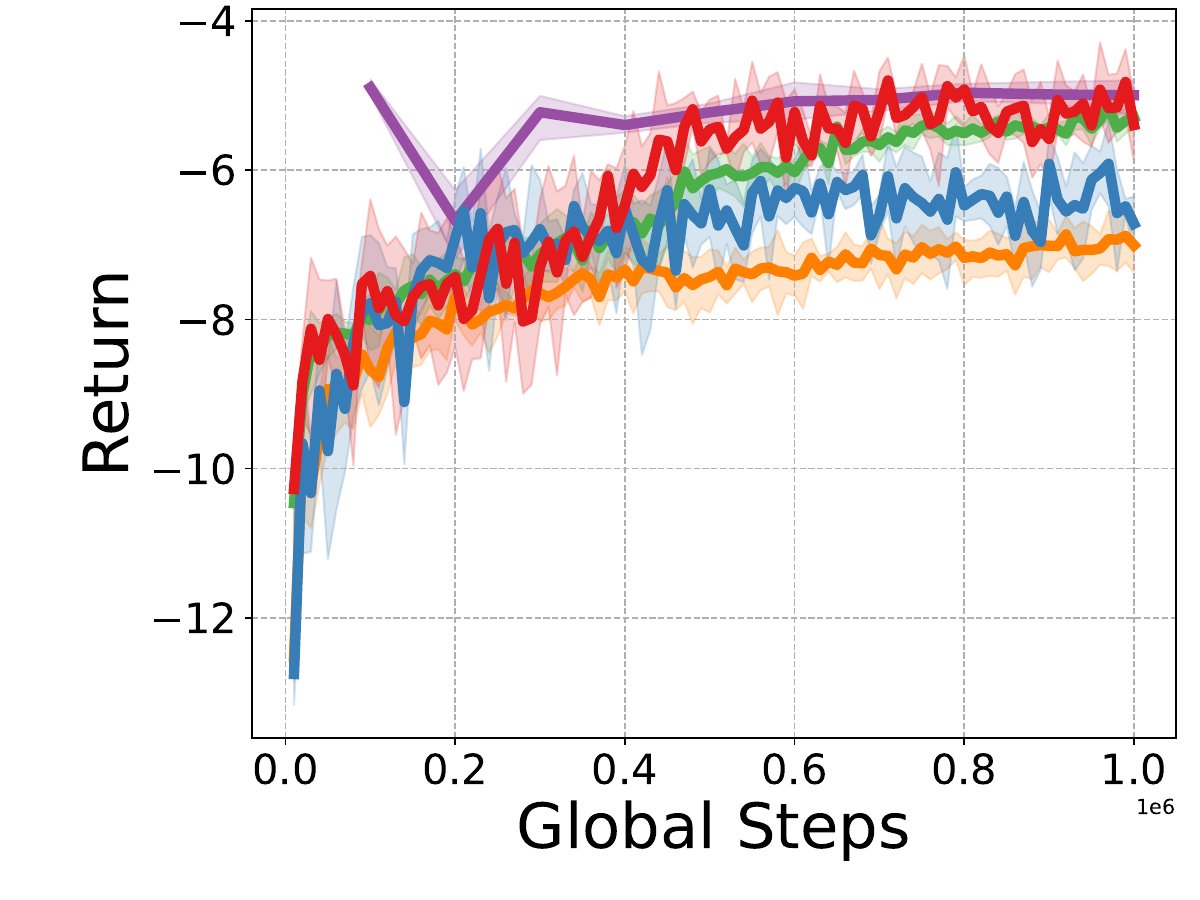}}
    \subfigure[Swimmer-v4]{\includegraphics[width=0.33\linewidth]{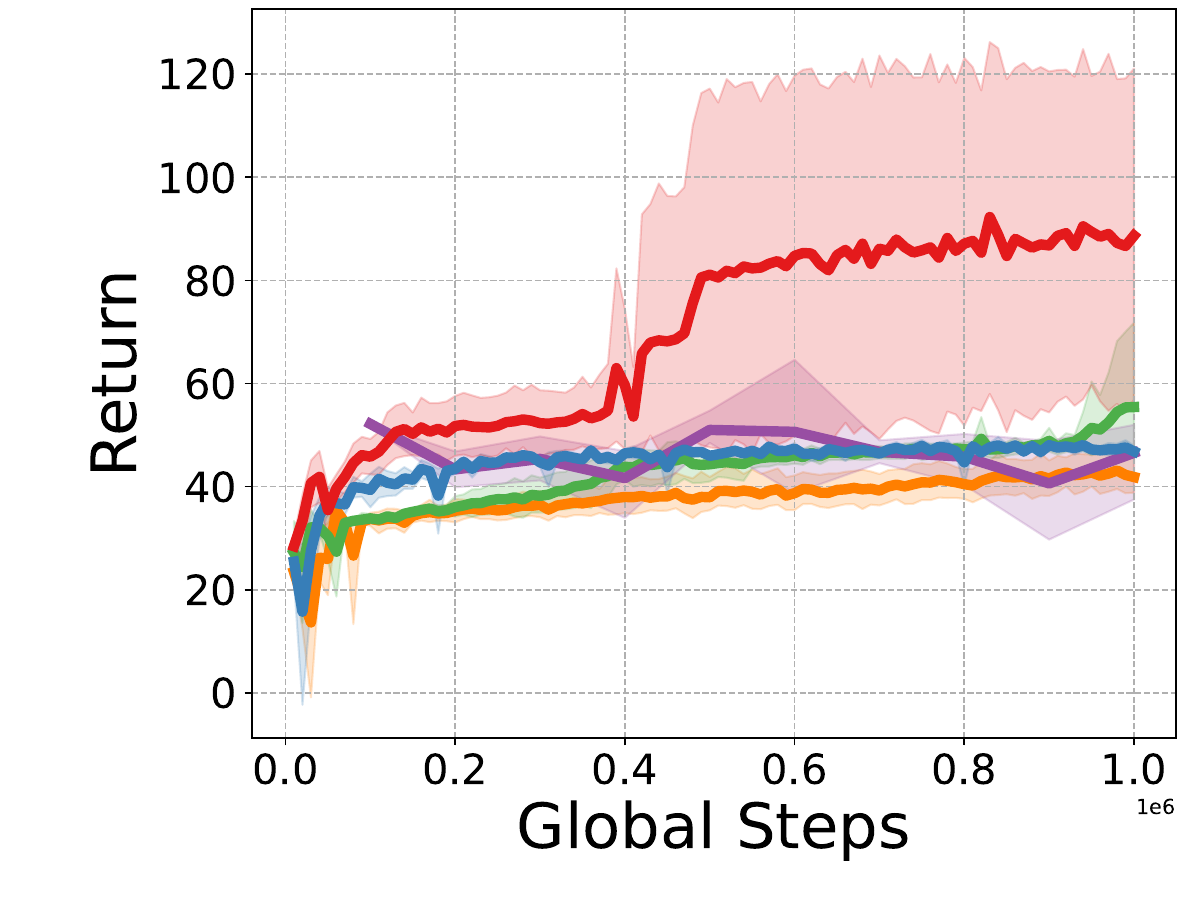}}
    \subfigure[Walker2d-v4]{\includegraphics[width=0.33\linewidth]{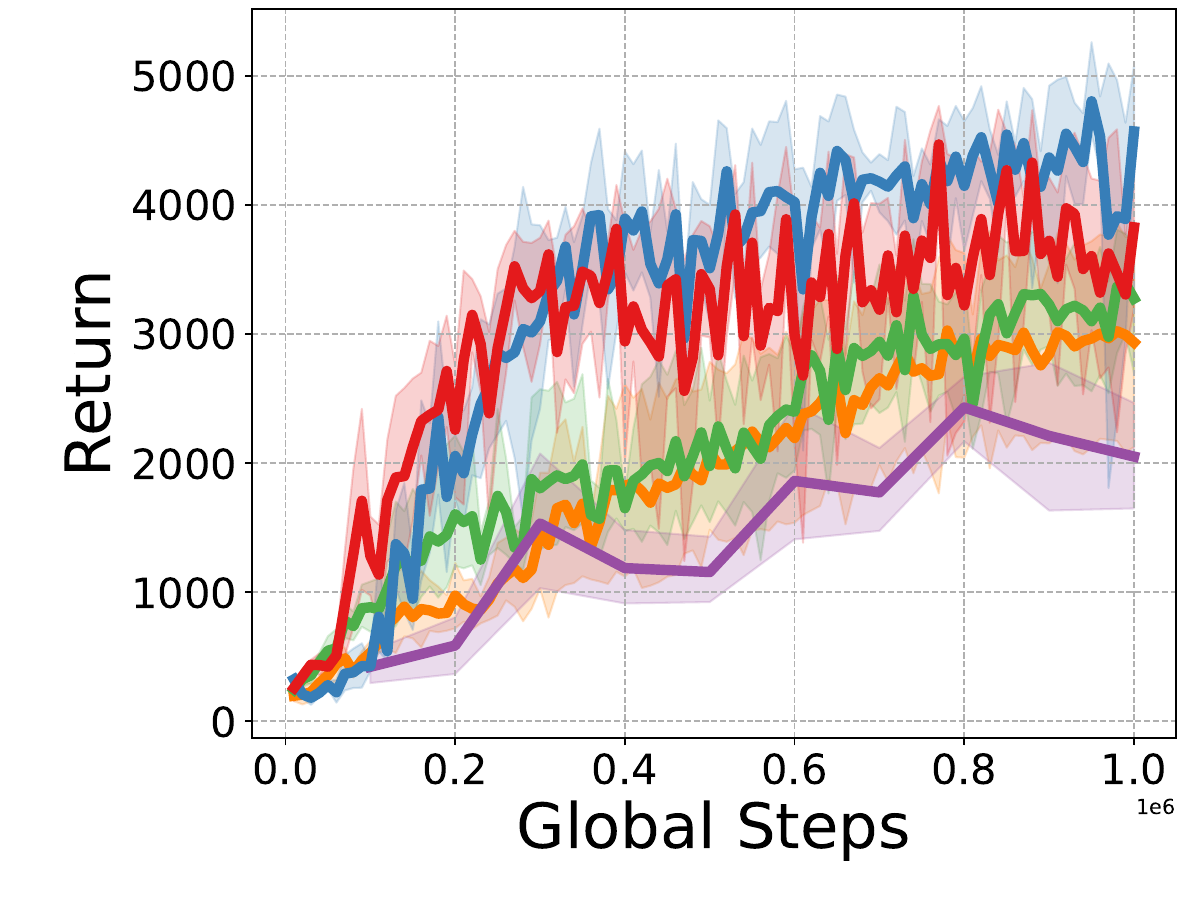}}
}
\centerline{
    \includegraphics[width=0.9\linewidth]{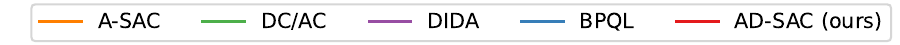}
}
\caption{Results of MuJoCo tasks with 5 delays for learning curves. The shaded areas represented the standard deviation (10 seeds).}
\label{appendix_mujoco_5_delay_step}
\end{center}
\vskip -0.3in
\end{figure*}

\begin{figure*}[h]
\begin{center}
\centerline{
    \subfigure[Ant-v4]{\includegraphics[width=0.33\linewidth]{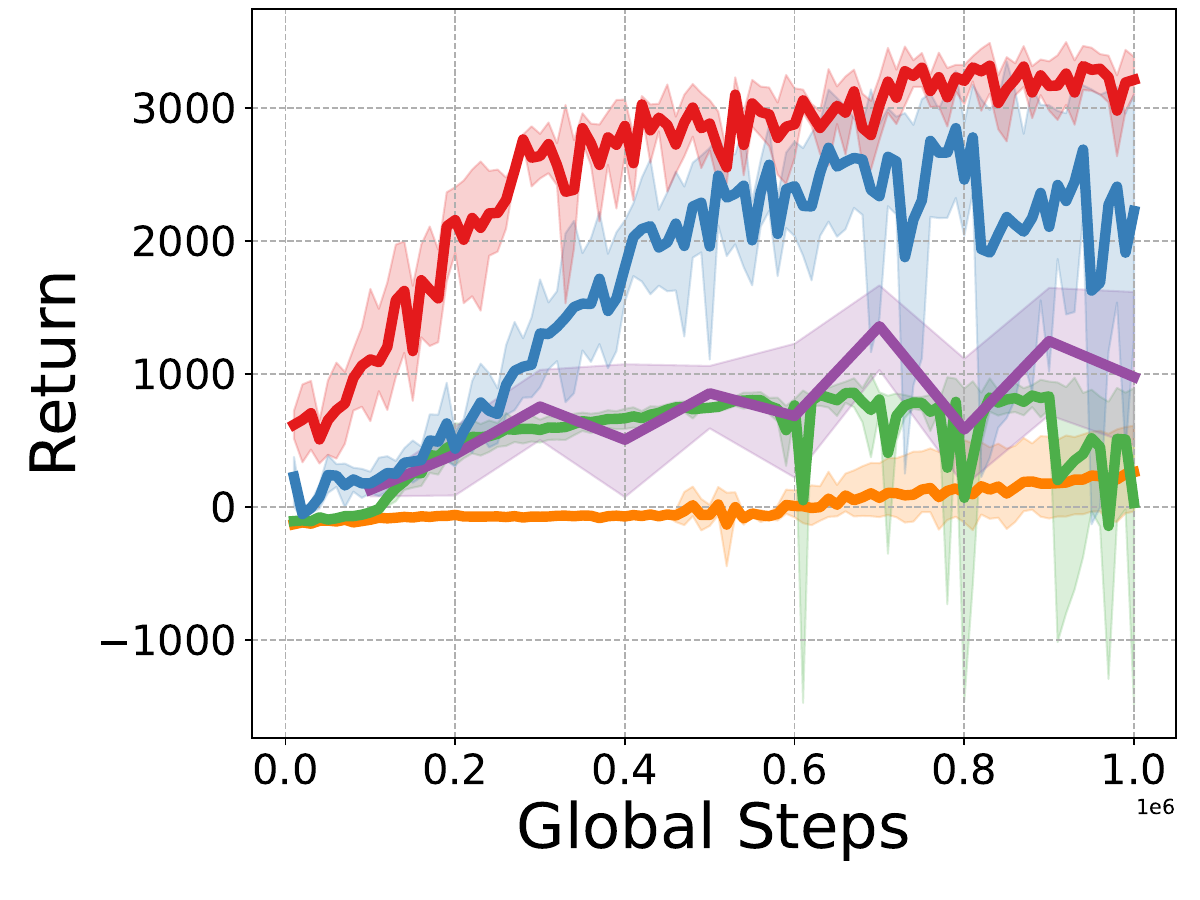}}
    \subfigure[HalfCheetah-v4]{\includegraphics[width=0.33\linewidth]{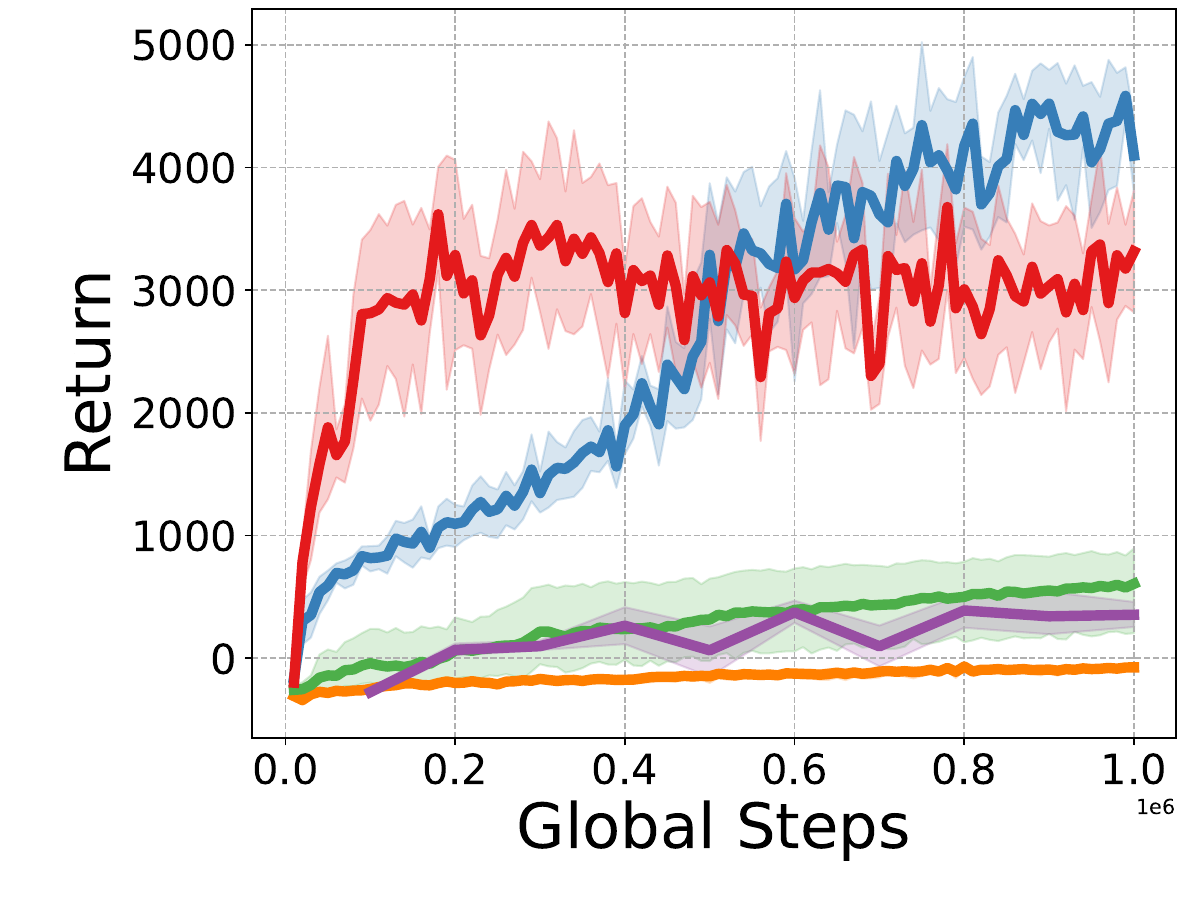}}
    \subfigure[Hopper-v4]{\includegraphics[width=0.33\linewidth]{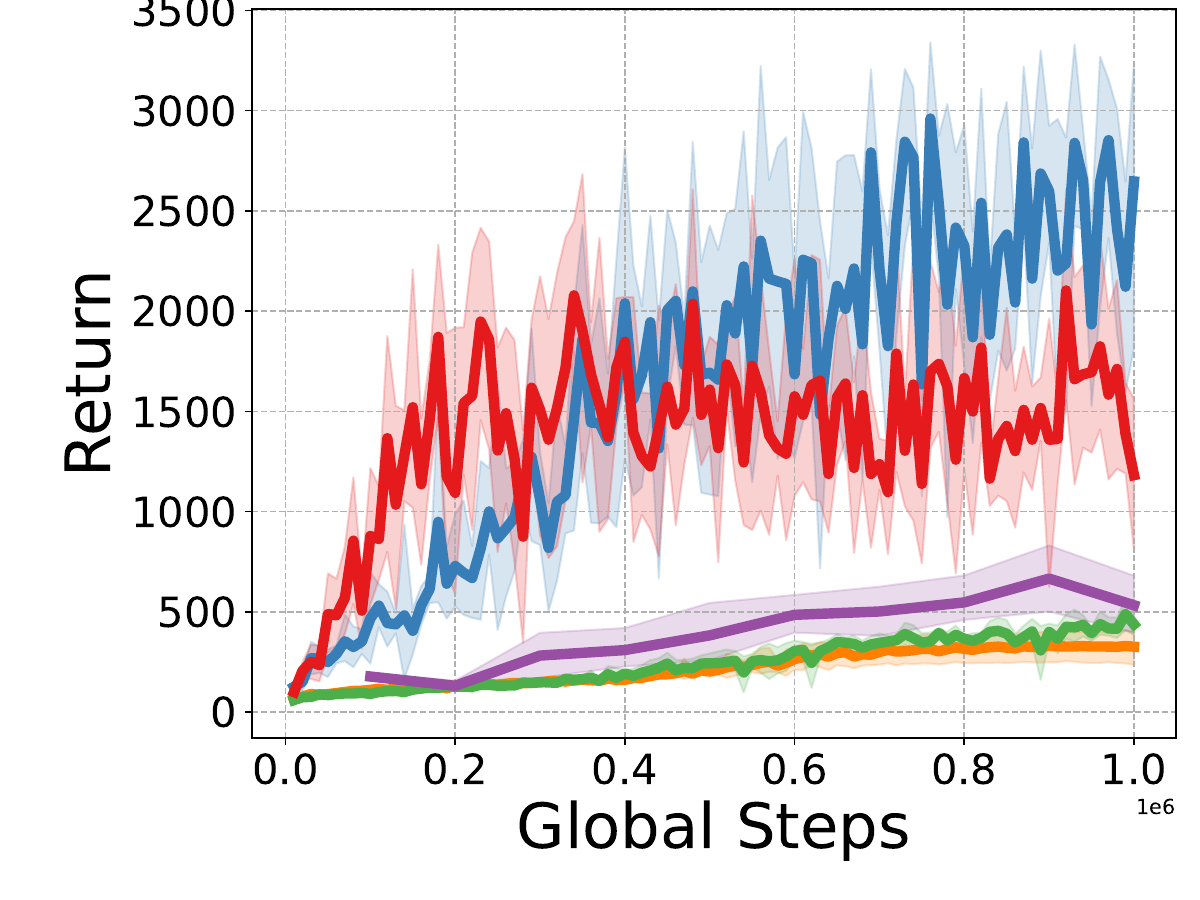}}
}
\centerline{
    \subfigure[Humanoid-v4]{\includegraphics[width=0.33\linewidth]{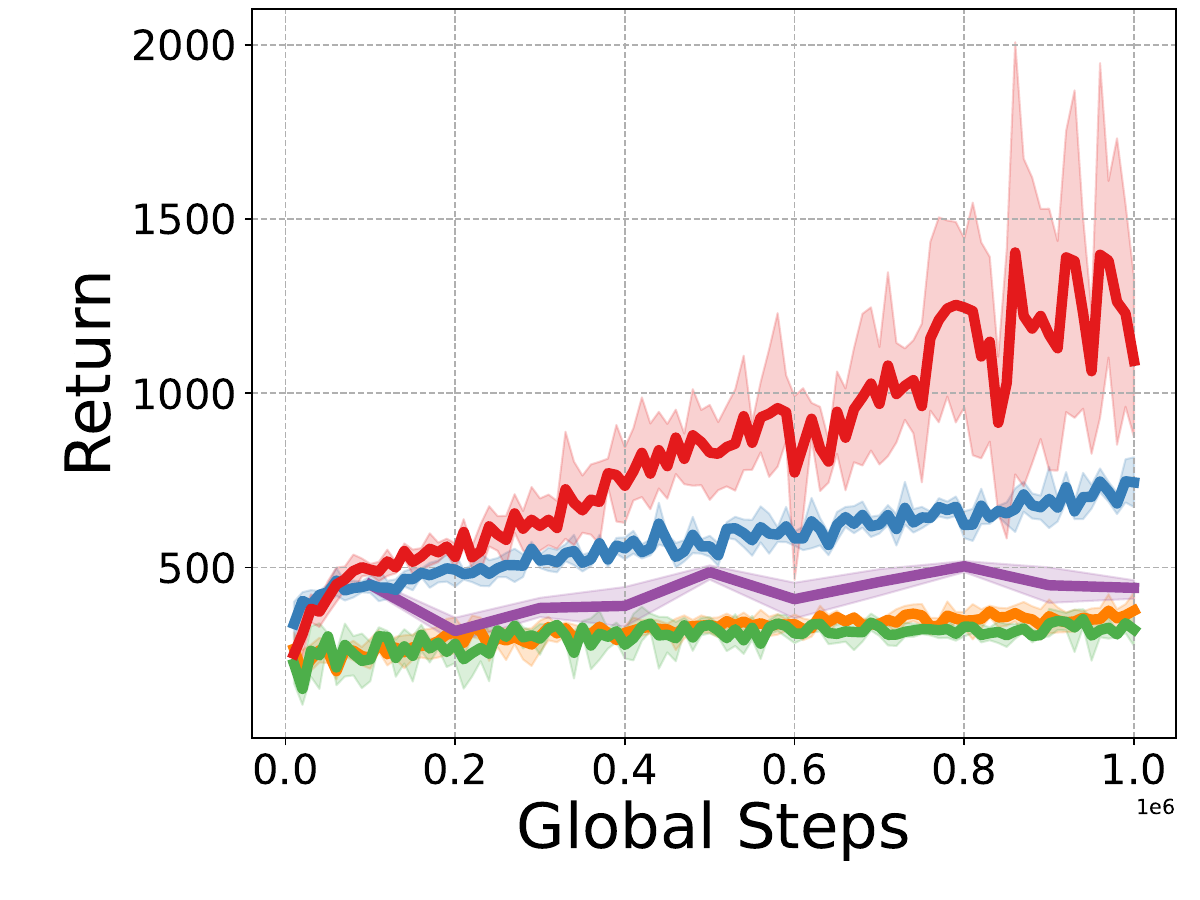}}
    \subfigure[HumanoidStandup-v4]{\includegraphics[width=0.33\linewidth]{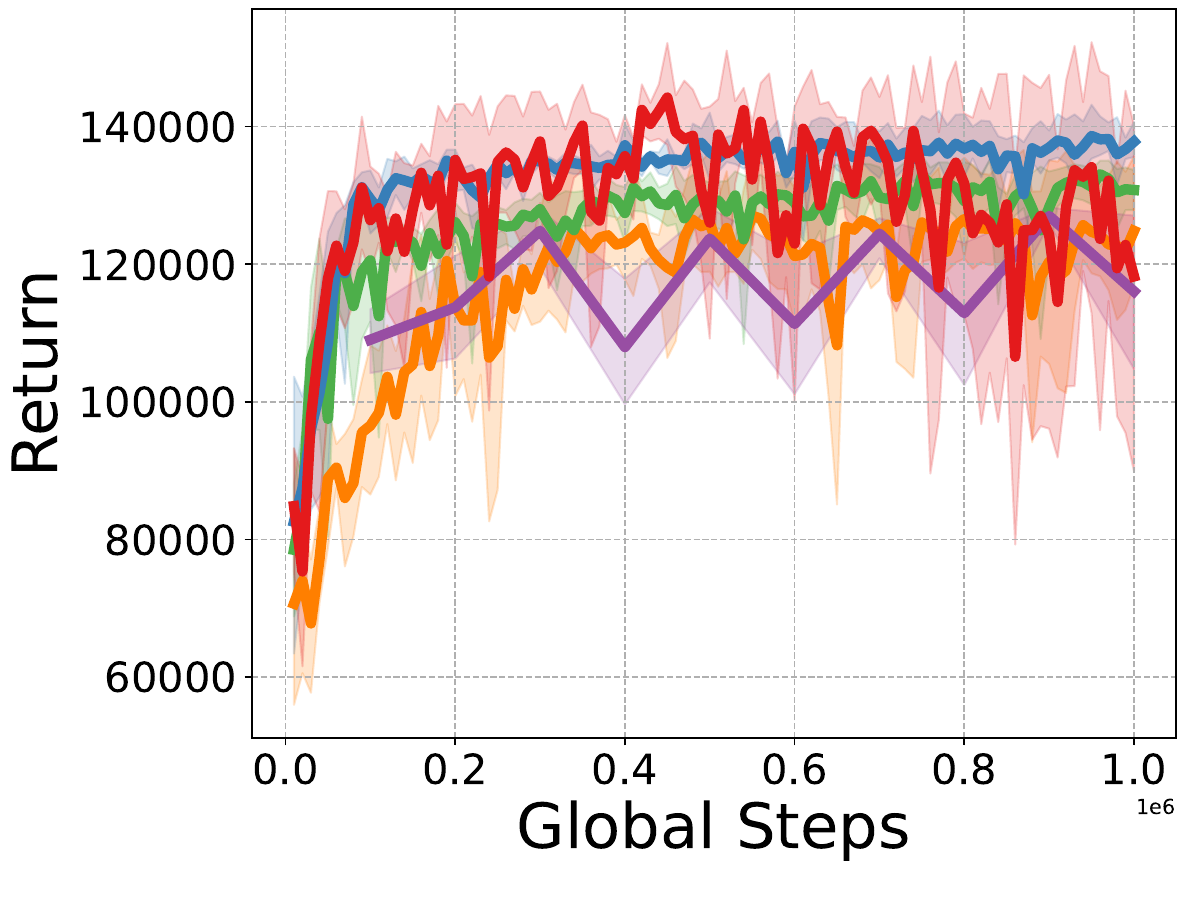}}
    \subfigure[Pusher-v4]{\includegraphics[width=0.33\linewidth]{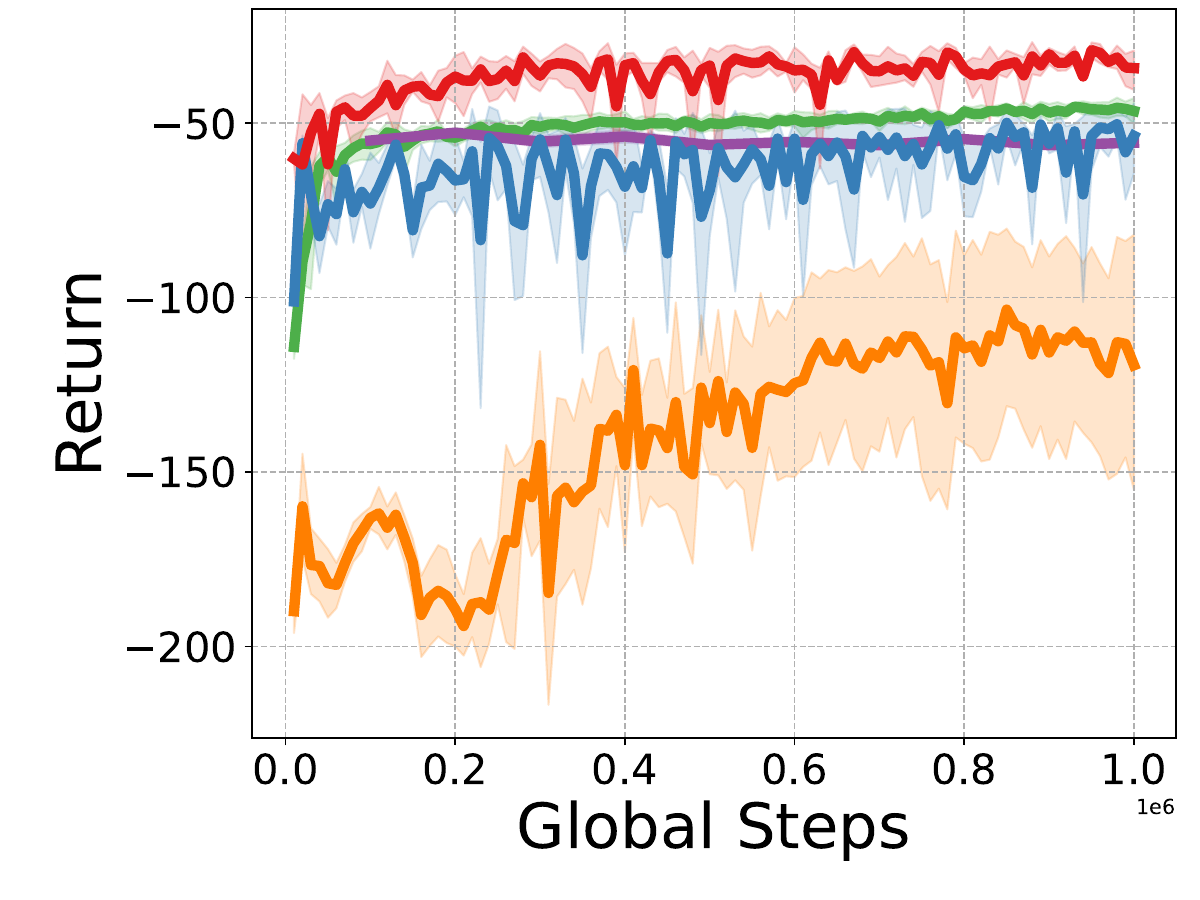}}
}
\centerline{
    \subfigure[Reacher-v4]{\includegraphics[width=0.33\linewidth]{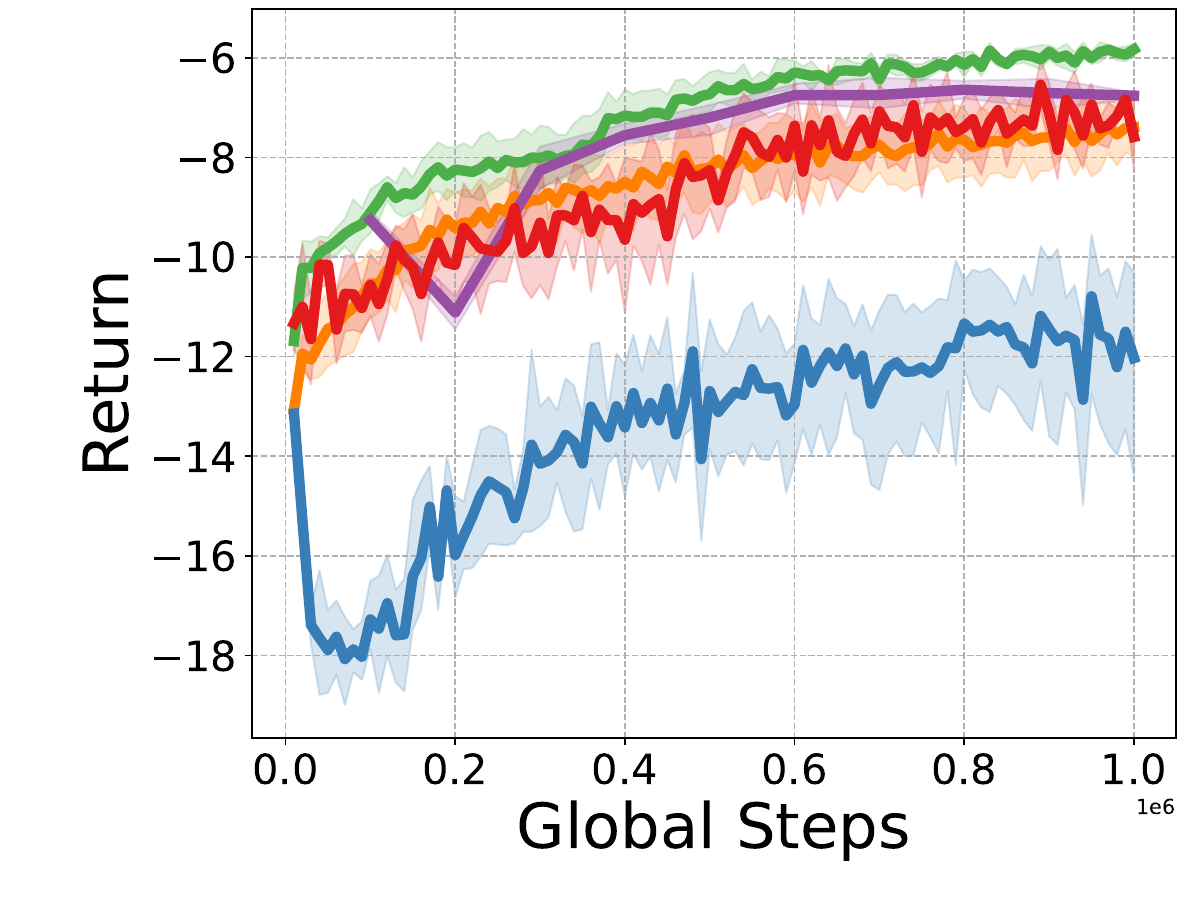}}
    \subfigure[Swimmer-v4]{\includegraphics[width=0.33\linewidth]{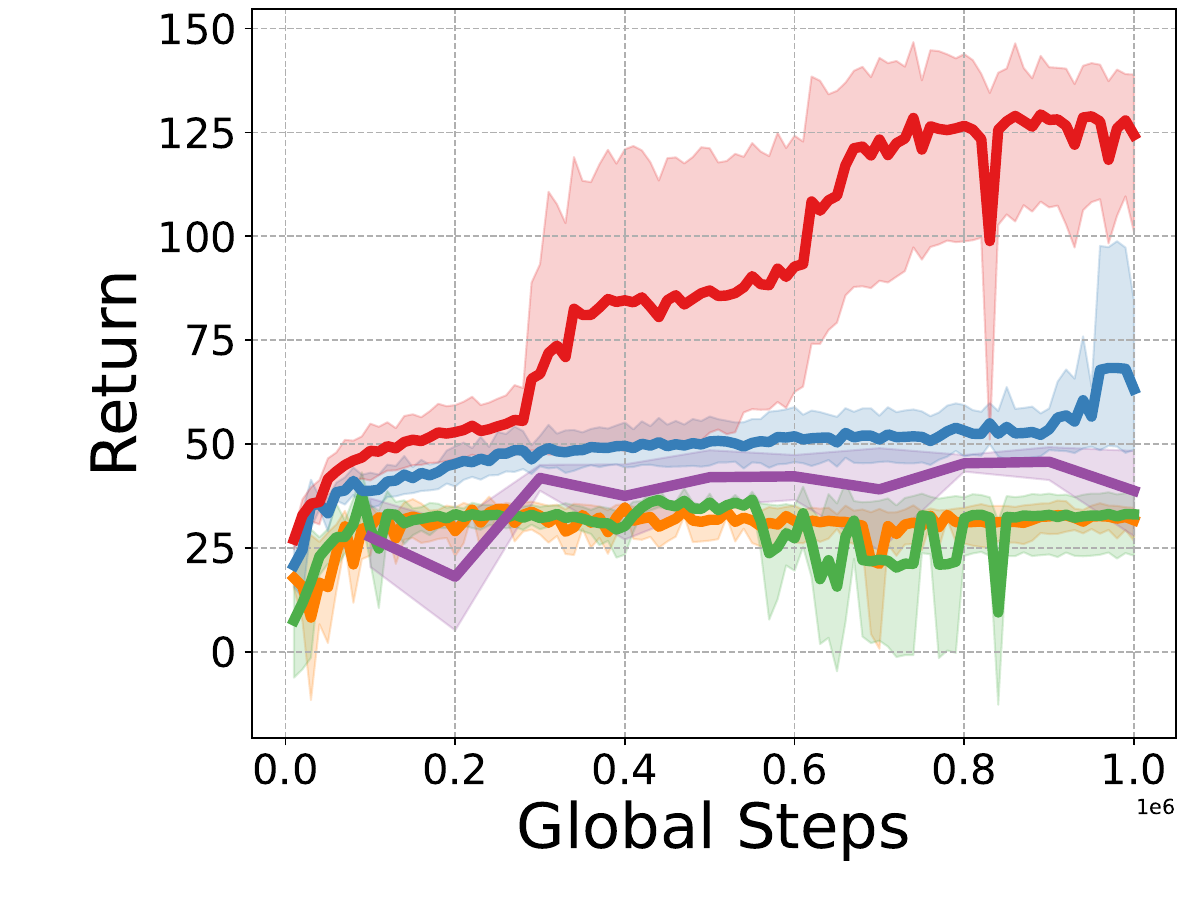}}
    \subfigure[Walker2d-v4]{\includegraphics[width=0.33\linewidth]{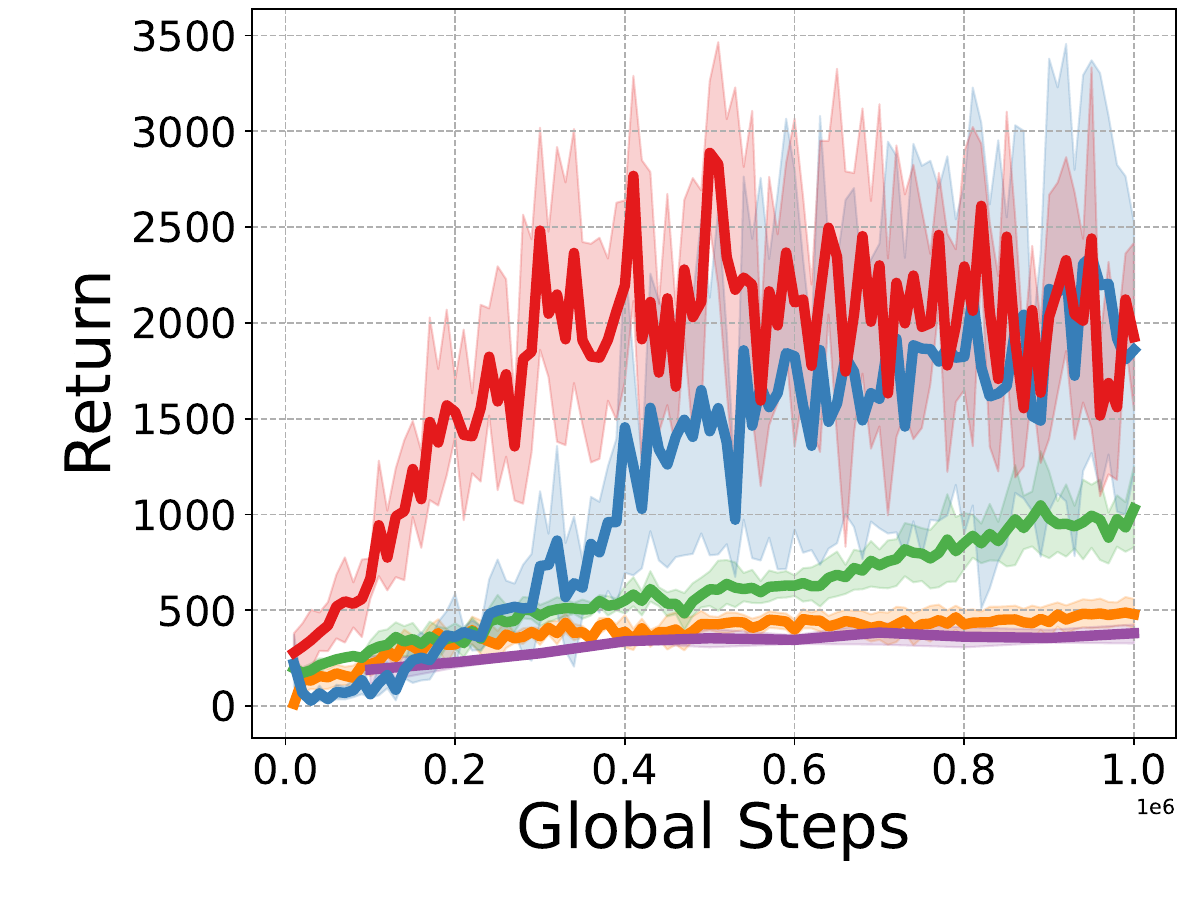}}
}
\centerline{
    \includegraphics[width=0.9\linewidth]{figs/mujoco/legend.pdf}
}
\caption{Results of MuJoCo tasks with 25 delays for learning curves. The shaded areas represented the standard deviation (10 seeds).}
\label{appendix_mujoco_25_delay_step}
\end{center}
\vskip -0.3in
\end{figure*}

\begin{figure*}[h]
\begin{center}
\centerline{
    \subfigure[Ant-v4]{\includegraphics[width=0.33\linewidth]{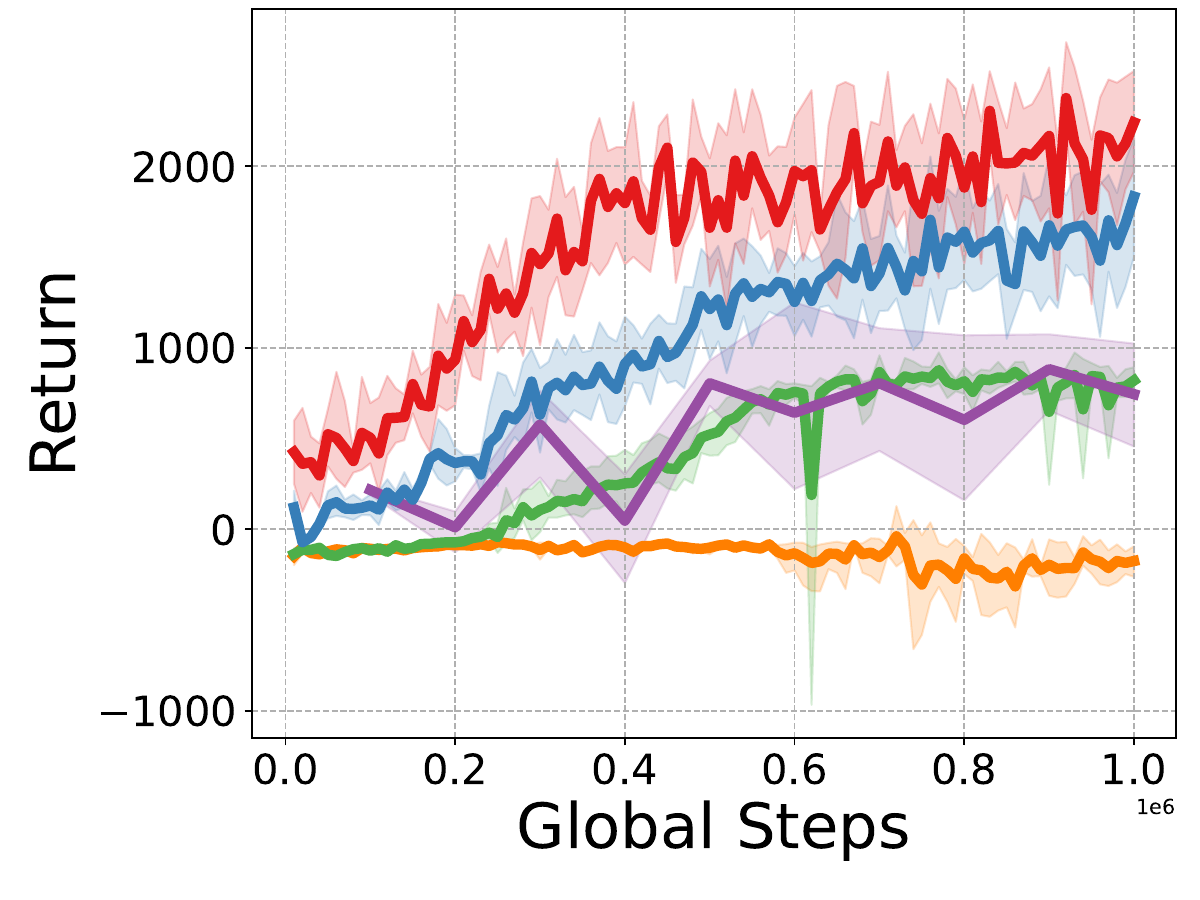}}
    \subfigure[HalfCheetah-v4]{\includegraphics[width=0.33\linewidth]{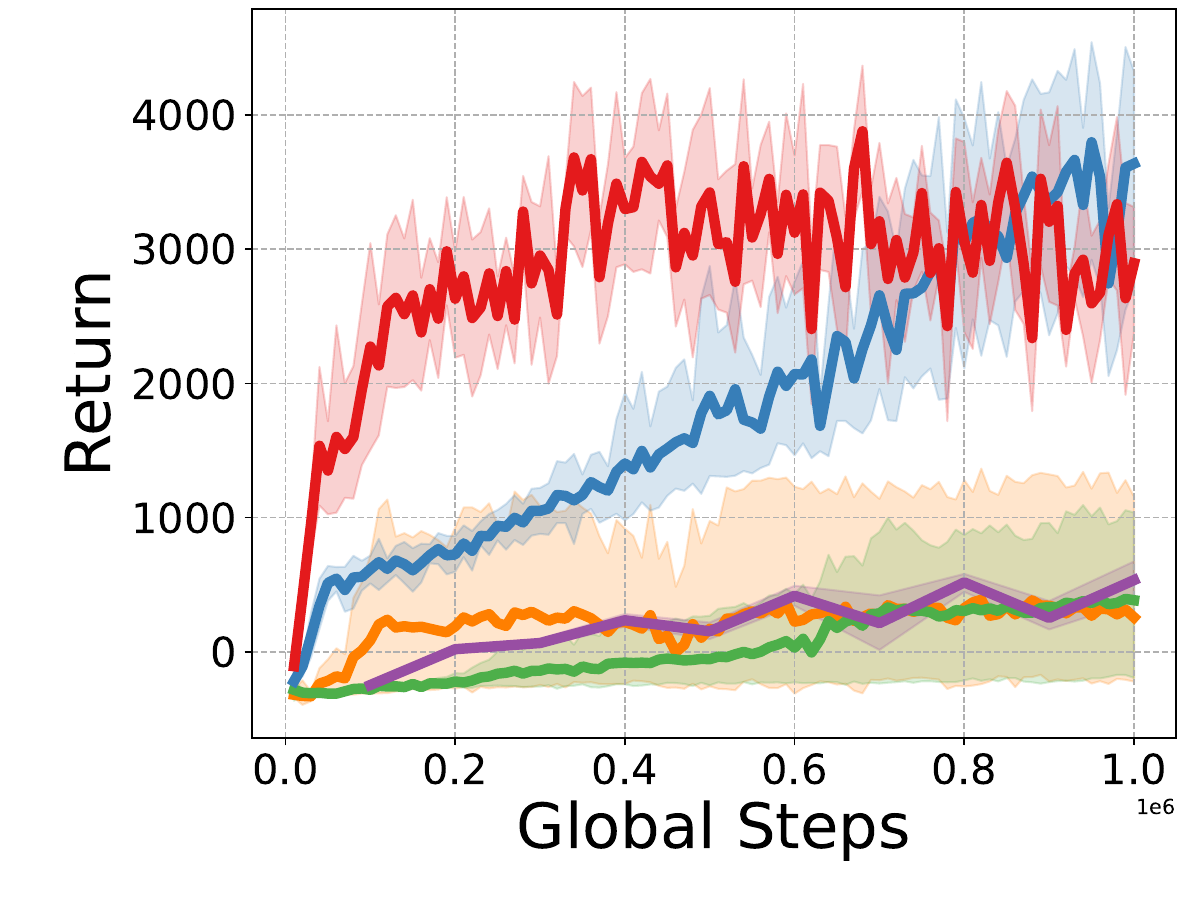}}
    \subfigure[Hopper-v4]{\includegraphics[width=0.33\linewidth]{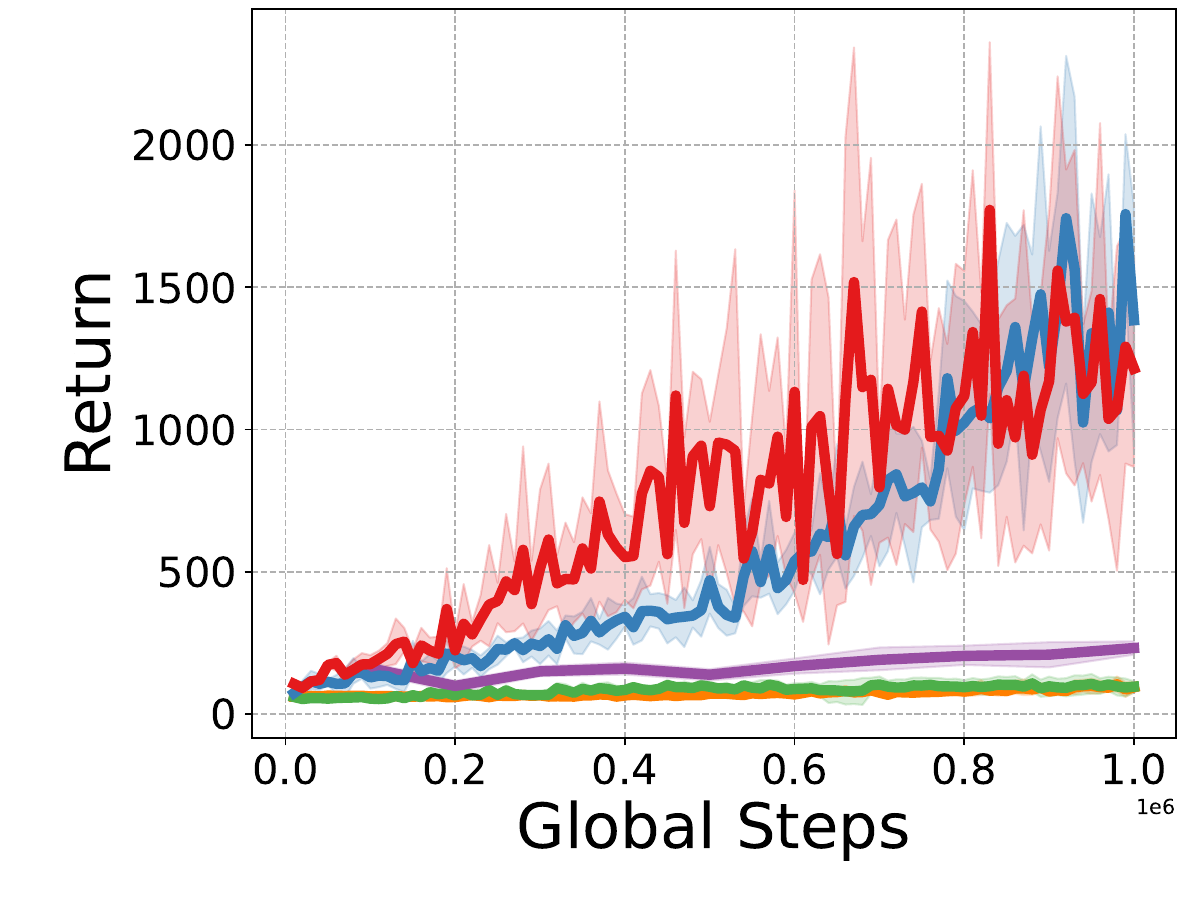}}
}
\centerline{
    \subfigure[Humanoid-v4]{\includegraphics[width=0.33\linewidth]{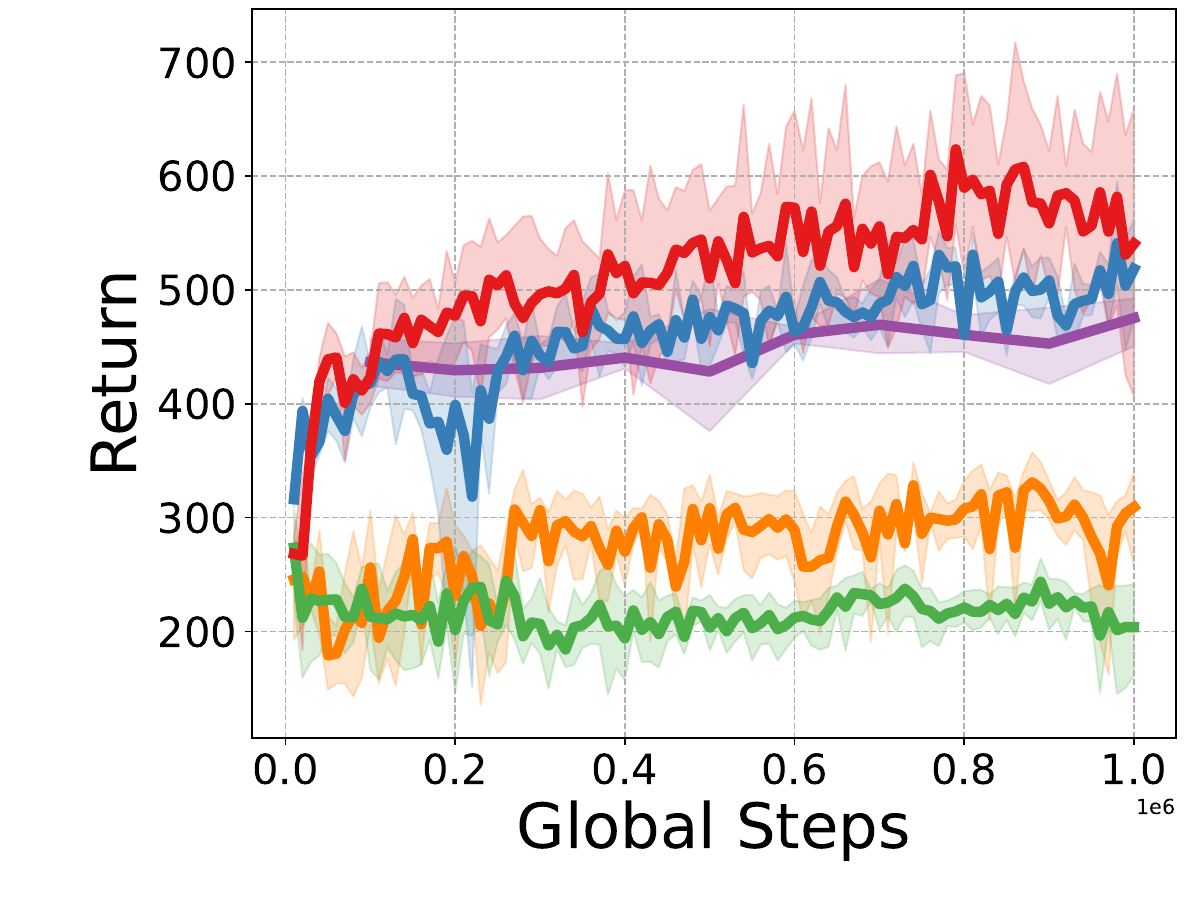}}
    \subfigure[HumanoidStandup-v4]{\includegraphics[width=0.33\linewidth]{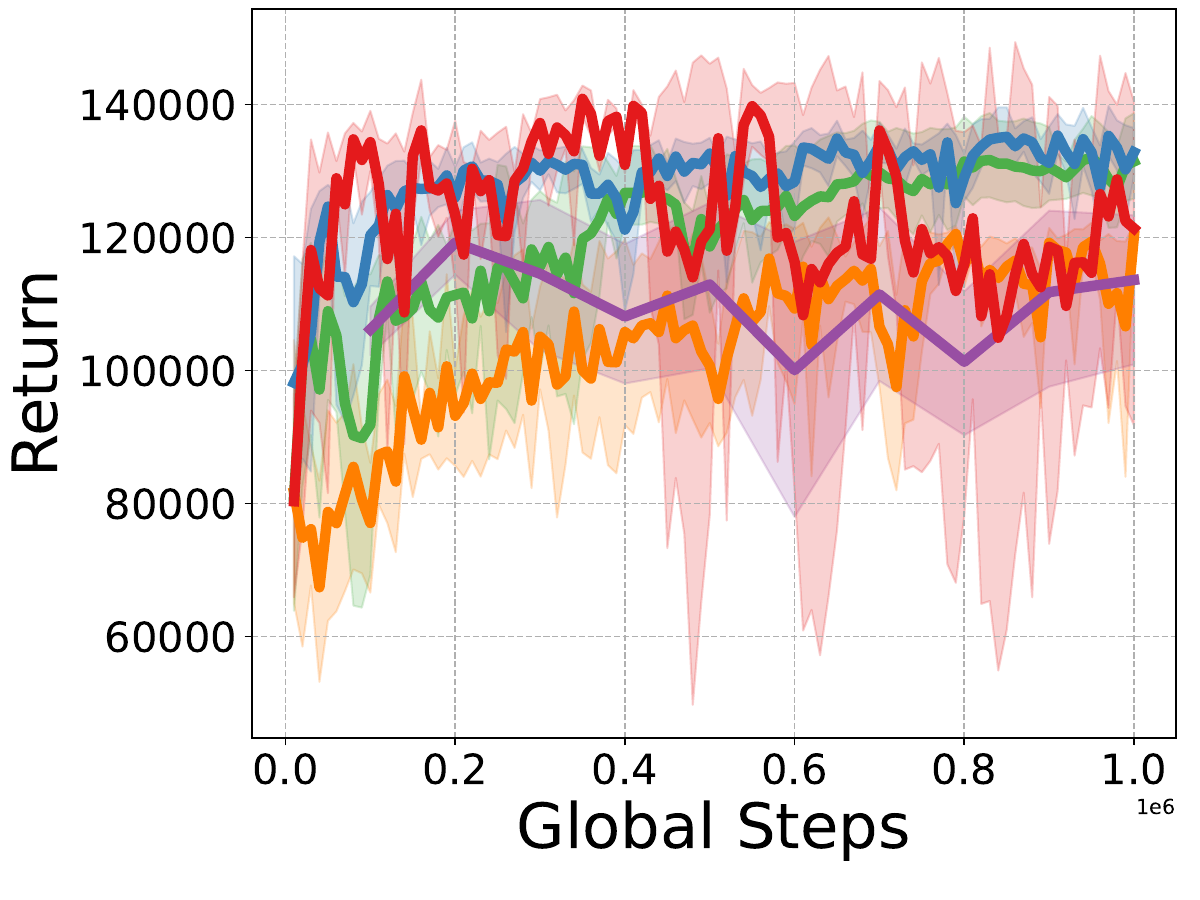}}
    \subfigure[Pusher-v4]{\includegraphics[width=0.33\linewidth]{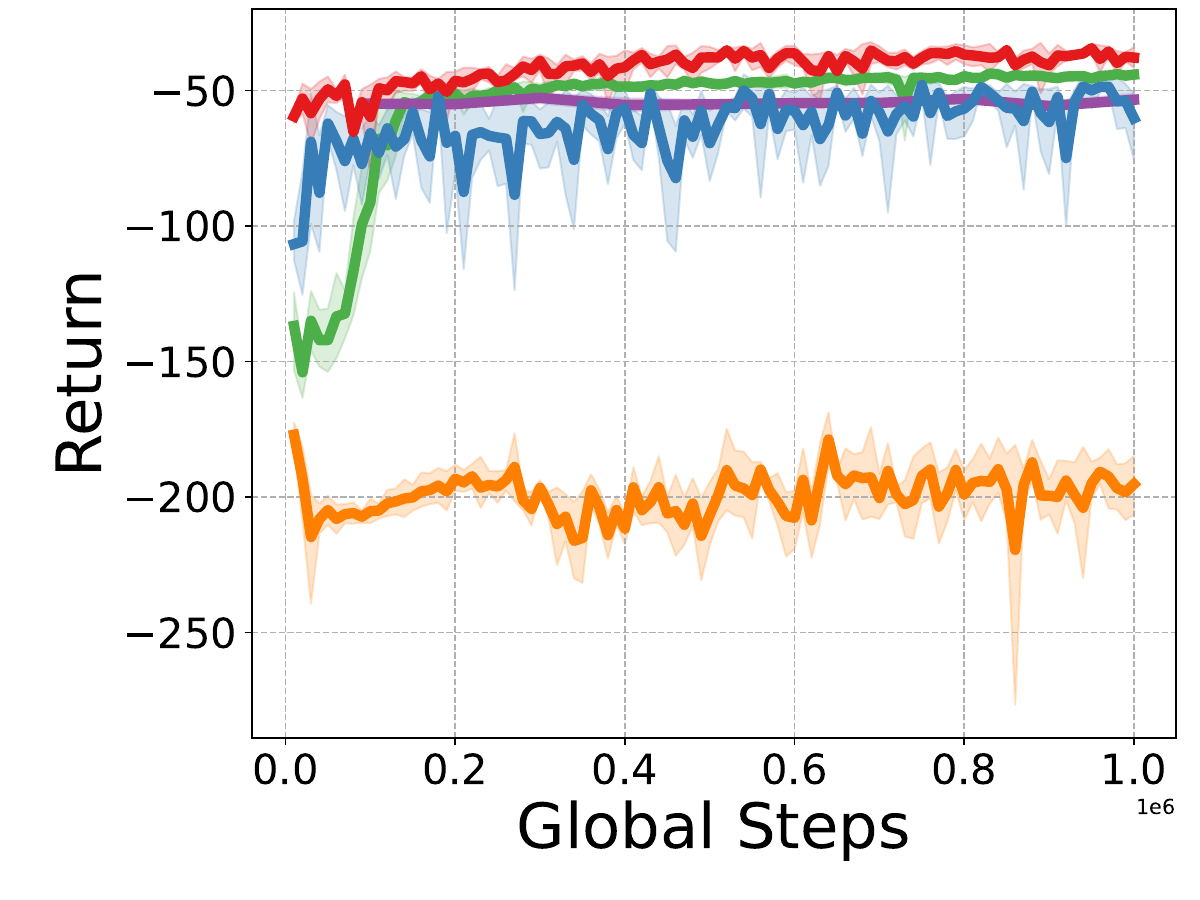}}
}
\centerline{
    \subfigure[Reacher-v4]{\includegraphics[width=0.33\linewidth]{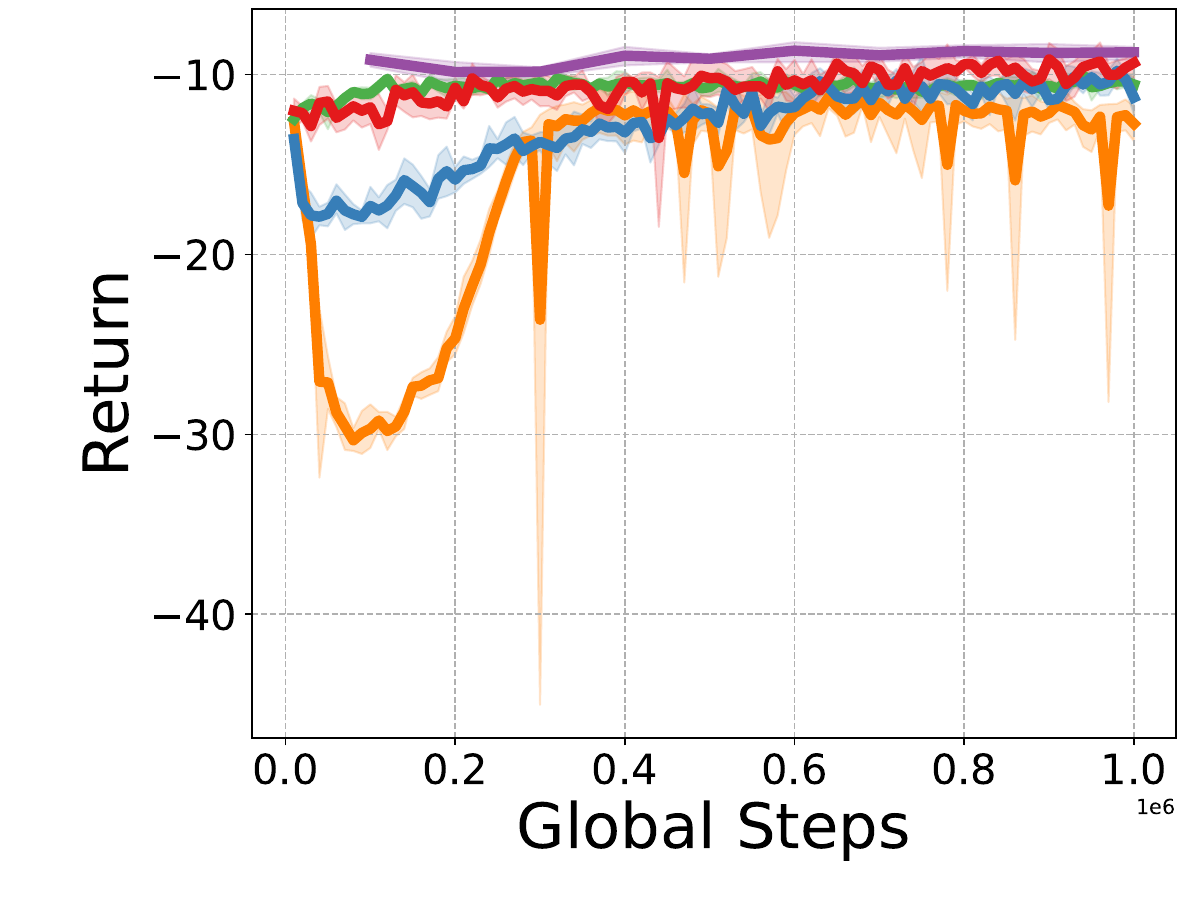}}
    \subfigure[Swimmer-v4]{\includegraphics[width=0.33\linewidth]{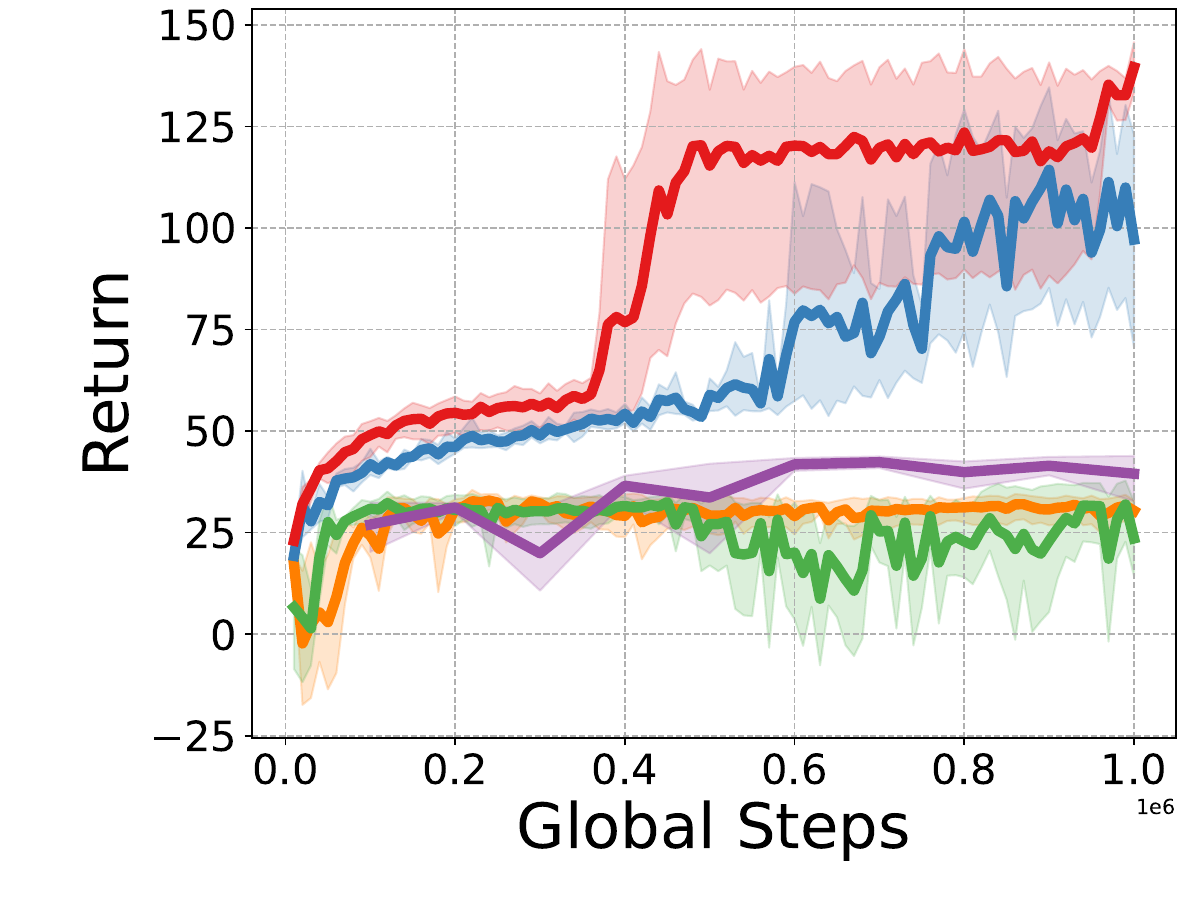}}
    \subfigure[Walker2d-v4]{\includegraphics[width=0.33\linewidth]{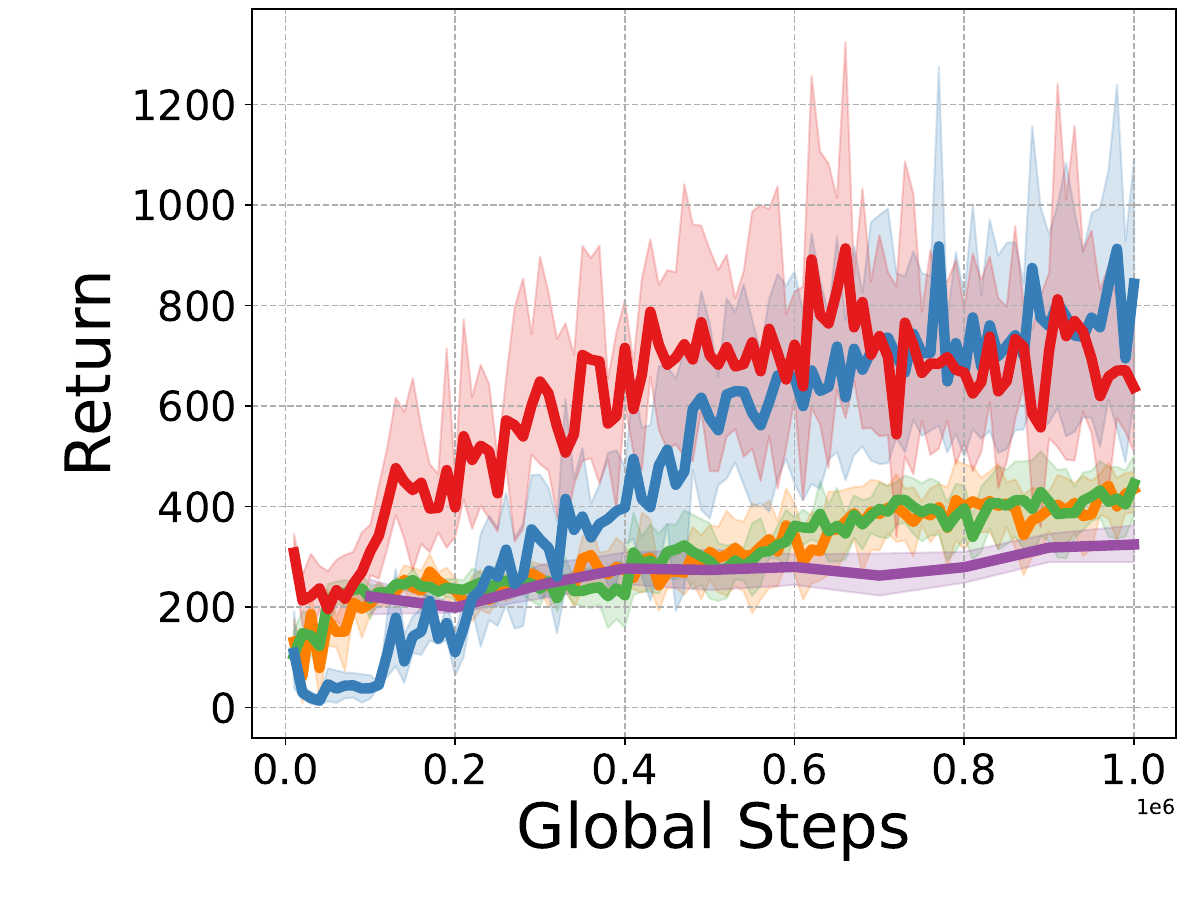}}
}
\centerline{
    \includegraphics[width=0.9\linewidth]{figs/mujoco/legend.pdf}
}
\caption{Results of MuJoCo tasks with 50 delays for learning curves. The shaded areas represented the standard deviation (10 seeds).}
\label{appendix_mujoco_50_delay_step}
\end{center}
\vskip -0.3in
\end{figure*}

\end{document}